%% file: main.tex
\documentclass[acmsmall,screen]{acmart}

\setcopyright{rightsretained}
\acmPrice{}
\acmDOI{10.1145/3498704}
\acmYear{2022}
\copyrightyear{2022}
\acmSubmissionID{popl22main-p329-p}
\acmJournal{PACMPL}
\acmVolume{6}
\acmNumber{POPL}
\acmArticle{43}
\acmMonth{1}

\usepackage{amsmath}
\usepackage{amsfonts}

\input{headers/header}

\input{headers/header-lqa}

\begin{document}

	\title{PRIMA: General and Precise Neural Network Certification via Scalable Convex Hull Approximations}

	\author{Mark Niklas Müller}
	\affiliation{
		\department{Department of Computer Science}              %
		\institution{ETH Zurich}            %
		\city{Zurich}
		\country{Switzerland}                    %
	}
	\authornote{Equal contribution}
	\email{mark.mueller@inf.ethz.ch}

	\author{Gleb Makarchuk}
	\affiliation{
		\department{Department of Computer Science}              %
		\institution{ETH Zurich}            %
		\city{Zurich}
		\country{Switzerland}                    %
	}
	\authornotemark[1]
	\email{gleb.makarchuk@gmail.com}
	
	\author{Gagandeep Singh}
	\affiliation{
		\institution{UIUC and VMware Research}           %
		\country{United States}                   %
	}
	\email{ggnds@illinois.edu}
	\email{gasingh@vmware.com}          %
	
	\author{Markus Püschel}
	\affiliation{
		\department{Department of Computer Science}              %
		\institution{ETH Zurich}            %
		\country{Switzerland}                    %
	}
	\email{pueschel@inf.ethz.ch}    
	\author{Martin Vechev}
	\affiliation{
		\department{Department of Computer Science}              %
		\institution{ETH Zurich}            %
		\city{Zurich}
		\country{Switzerland}                    %
	}
	\email{martin.vechev@inf.ethz.ch}              %

	\input{abstract}



	\maketitle

	\input{introduction}

	\input{background}

	\input{overview}

	\input{technical}

	\input{experimental}

	\input{related}

	\input{conclusion}

	\message{^^JLASTBODYPAGE \thepage^^J}

	\clearpage
	\bibliography{references}

	\message{^^JLASTREFERENCESPAGE \thepage^^J}

	\setbool{includeappendix}{false}
	
	\ifbool{includeappendix}{%
		\clearpage
		\appendix

\input{appendix}
	}{}
	
	\message{^^JLASTPAGE \thepage^^J}
	
\end{document}

%% file: headers/header.tex
\usepackage{bm}
\usepackage{amsmath,amsthm,amsfonts}
\usepackage{xspace}
\usepackage{accents}
\usepackage{adjustbox}%
\usepackage{threeparttable,booktabs}
\usepackage{multicol}
\usepackage[ruled,vlined,noend]{algorithm2e}
\usepackage{xfrac}
\usepackage[english]{babel}
\usepackage{adjustbox}
\usepackage{pgfplots}
\usepackage{pgfplotstable}
\usepackage{paralist}
\usepackage{subcaption}
\usepackage{thmtools}
\usepackage{thm-restate}
\usepackage{placeins}
\usepackage{url}
\usepackage{natbib}
\usepackage{makecell}
\usepackage{wrapfig}
\usepackage{multirow}
\usepackage{stackengine}

\usepackage{tikz}
\usetikzlibrary{calc,decorations,decorations.pathmorphing}
\usetikzlibrary{positioning,fit,arrows}
\usetikzlibrary{decorations.markings}
\usetikzlibrary{shapes,shapes.geometric}
\usetikzlibrary{shadows,patterns,snakes}
\usetikzlibrary{backgrounds,decorations.pathreplacing,automata}
\usetikzlibrary{intersections}

\newcommand{\R}{\mathbb{R}}

\newcommand{\Th}{\textsuperscript{th}\xspace}
\newcommand{\Nd}{\textsuperscript{nd}\xspace}
\newcommand{\St}{\textsuperscript{st}\xspace}
\newcommand{\Rd}{\textsuperscript{rd}\xspace}
\renewcommand{\S}{\mathbb{S}}
\newcommand{\eps}{\epsilon}

\newcommand{\bb}[1]{\mathbb{#1}}
\newcommand{\bc}[1]{\mathcal{#1}}

\renewcommand{\paragraph}[1]{\textit{#1.}}

\declaretheorem[name=Theorem,numberwithin=section]{lemma}

\DeclareMathOperator*{\conv}{conv} %
\DeclareMathOperator*{\argmax}{arg\,max}

\newcommand{\vrep}{$\mathcal{V}$-representation\xspace}
\newcommand{\vreps}{$\mathcal{V}$-representations\xspace}
\newcommand{\hrep}{$\mathcal{H}$-representation\xspace}
\newcommand{\hreps}{$\mathcal{H}$-representations\xspace}
\newcommand{\PDDMl}{Partial Double Description Method\xspace}
\newcommand{\PDDM}{PDDM\xspace}
\newcommand{\PDD}{PDD\xspace}
\newcommand{\PDDl}{Partial Double Description\xspace}

\newcommand{\SBLM}{SBLM\xspace}
\newcommand{\SBLl}{Split-Bound-Lift\xspace}
\newcommand{\SBLMl}{\SBLl Method\xspace}

\newcommand{\FastLin}{\textsc{Fast-Lin}\xspace}
\newcommand{\Neurify}{\textsc{Neurify}\xspace}
\newcommand{\tool}{\textsc{Prima}\xspace}
\newcommand{\DeepPoly}{\textsc{DeepPoly}\xspace}
\newcommand{\GPUpoly}{\textsc{GPUPoly}\xspace}
\newcommand{\DeepZono}{\textsc{DeepZono}\xspace}
\newcommand{\RefinePoly}{\textsc{RefinePoly}\xspace}
\newcommand{\RefineZono}{\textsc{RefineZono}\xspace}
\newcommand{\kPoly}{\textsc{kPoly}\xspace}
\newcommand{\optcv}{\textsc{OptC2V}\xspace}
\newcommand{\fastcv}{\textsc{FastC2V}\xspace}
\newcommand{\CNNCert}{\textsc{CNN-Cert}\xspace}
\newcommand{\Crown}{\textsc{Crown}\xspace}
\newcommand{\BCrown}{\mbox{\textsc{$\beta$-Crown}}\xspace}
\newcommand{\SDPFO}{\textsc{SDP-FO}\xspace}

\newcommand{\convsmall}{\texttt{ConvSmall}\xspace}

\newcommand{\CNNBadv}{\texttt{CNN-B-Adv}\xspace}
\newcommand{\CNNAmix}{\texttt{CNN-A-Mix}\xspace}
\newcommand{\convbig}{\texttt{ConvBig}\xspace}
\newcommand{\resnet}{\texttt{ResNet}\xspace}
\newcommand{\NVIDIA}{\texttt{DAVE}\xspace}
\newcommand{\CNNA}{\texttt{CNN-A}\xspace}
\newcommand{\CNNB}{\texttt{CNN-B}\xspace}

\newcommand{\cifar}{CIFAR10\xspace}
\newcommand{\mnist}{MNIST\xspace}

\makeatletter
\def\els@aparagraph[#1]#2{\elsparagraph[#1]{#2}}
\def\els@bparagraph#1{\elsparagraph*{#1}}
\makeatother

\def\Figref#1{Figure~\ref{#1}}

\def\Tableref#1{Table~\ref{#1}}

\def\Secref#1{Section~\ref{#1}}

\def\Algref#1{Algorithm~\ref{#1}}

%% file: headers/header-lqa.tex
\usepackage[utf8]{inputenc} %
\usepackage{lipsum} %

\usepackage[T1]{fontenc}

\usepackage{etoolbox}
\newbool{includeappendix}
\setbool{includeappendix}{true}

\input{headers/lqa/overfull}

\input{headers/lqa/comments}
\input{headers/lqa/acronyms}

\input{headers/lqa/listing}

\input{headers/lqa/references}

%% file: headers/lqa/overfull.tex
\ifdefined\isoverfull
\else
\fi

%% file: headers/lqa/acronyms.tex
\usepackage{acro} %

\DeclareAcronym{cli} {
    short = CLI,
    long = Command Line Interface,
    class = abbrev
}

%% file: headers/lqa/listing.tex
\usepackage{listings}

\usepackage{textcomp}

\usepackage{xcolor}

\usepackage[scaled=0.8]{beramono}

\definecolor{ckeyword}{HTML}{7F0055}
\definecolor{ccomment}{HTML}{3F7F5F}
\definecolor{cstring}{HTML}{2A0099}

\lstdefinestyle{numbers}{
	numbers=left,
	framexleftmargin=20pt,
	numberstyle=\tiny,
	firstnumber=auto,
	numbersep=1em,
	xleftmargin=2em
}

\lstdefinestyle{layout}{
	frame=none,
	captionpos=b,
}

\lstdefinestyle{comment-style}{
	morecomment=[l]//,
	morecomment=[s]{/*}{*/},
	commentstyle={\color{ccomment}\itshape},
}

\lstdefinestyle{string-style}{
	morestring=[b]",%
	morestring=[b]',%
	stringstyle={\color{cstring}},
	showstringspaces=false,%
}

\lstdefinestyle{keyword-style}{
	keywordstyle={\ttfamily\bfseries},
	morekeywords={
		function,
		constructor,
		int,
		bool,
		return,
		returns,
		uint
	},
	morekeywords = [2]{},
	keywordstyle = [2]{\text},
	sensitive=true,
}

\lstdefinestyle{input-encoding}{
	inputencoding=utf8,
	extendedchars=true,
	literate=
	{ℝ}{$\reals$}1%
	{→}{$\rightarrow$}1%
	{α}{$\alpha$}1%
	{β}{$\beta$}1%
	{λ}{$\lambda$}1%
	{θ}{$\theta$}1%
	{ϕ}{$\phi$}1%
}

\lstdefinestyle{escaping}{
	moredelim={**[is][\color{blue}]{\%}{\%}},
	escapechar=|,
	mathescape=true
}

\lstdefinestyle{default-style}{
	basicstyle=\fontencoding{T1}\ttfamily\footnotesize,
	style=numbers,
	style=layout,
	style=comment-style,
	style=string-style,
	style=keyword-style,
	style=input-encoding,
	style=escaping,
	tabsize=2,
	upquote=true
}

\lstdefinelanguage{BASIC}{
	language=C++,
	style=default-style
}[keywords,comments,strings]%

%% file: headers/lqa/references.tex
\usepackage[capitalize]{cleveref}

\crefformat{section}{\S#2#1#3}

\crefrangeformat{section}{\S#3#1#4\crefrangeconjunction\S#5#2#6}

\crefmultiformat{section}{\S#2#1#3}{\crefpairconjunction\S#2#1#3}{\crefmiddleconjunction\S#2#1#3}{\creflastconjunction\S#2#1#3}

\newcommand{\crefrangeconjunction}{--}

\crefname{listing}{Lst.}{listings}
\crefname{line}{Lin.}{Lin.}
\crefname{appendix}{App.}{App.}

\newcommand{\app}[1]{%
	\ifbool{includeappendix}{\cref{#1}}{the appendix}%
}
\newcommand{\App}[1]{%
	\ifbool{includeappendix}{\cref{#1}}{The appendix}%
}

%% file: abstract.tex
\begin{abstract} Formal verification of neural networks is critical for their safe adoption in real-world applications. However, designing a precise and scalable verifier which can handle different activation functions, realistic network architectures and relevant specifications remains an open and difficult challenge.

In this paper, we take a major step forward in addressing this challenge and present a new verification framework, called \tool. \tool is both (i) general: it handles any non-linear activation function, and (ii) precise: it computes precise convex abstractions involving \emph{multiple} neurons via novel convex hull approximation algorithms that leverage concepts from computational geometry. The algorithms have polynomial complexity, yield fewer constraints, and minimize precision loss.

We evaluate the effectiveness of \tool on a variety of challenging tasks from prior work. Our results show that \tool is significantly more precise than the state-of-the-art, verifying robustness to input perturbations for up to 20\%, 30\%, and 34\% more images than existing work on ReLU-, Sigmoid-, and Tanh-based networks, respectively. Further, \tool enables, for the first time, the precise verification of a realistic neural network for autonomous driving within a few minutes.
\end{abstract}

%% file: introduction.tex
\section{Introduction}\label{sec:intro}

The growing adoption of neural networks (NNs) in many safety critical domains highlights the importance of providing formal, deterministic guarantees about their safety and robustness when deployed in the real world \cite{szegedy:13}. While the last few years have seen significant progress in formal verification of NNs, existing deterministic methods (see \citet{Caterina21Review} for a survey) still either do not scale to or are too imprecise when handling realistic networks.

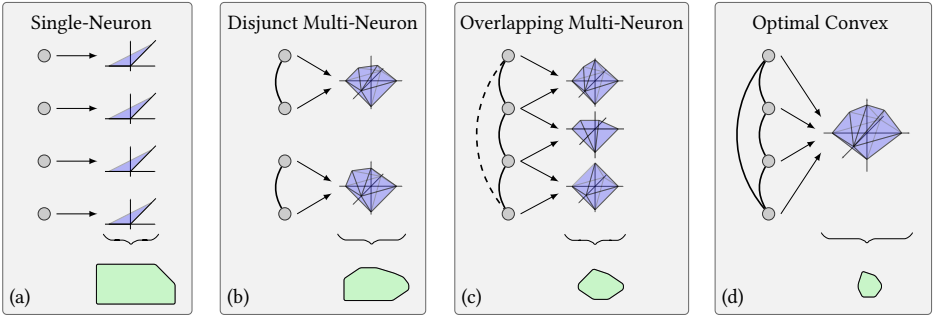
\begin{figure*}[t]
	\centering
	\vspace{-2mm}
	\scalebox{0.77}{\input{figures/concept_multi}}
	\vspace{-2mm}
	\caption{Illustration of the tightness of different abstraction strategies, for a layer of four neurons (grey dots). Strong interdependencies between neurons that can be captured directly or indirectly are shown as solid or dashed lines, respectively. Individual single-neuron, multi-neuron or optimal convex abstractions are illustrated in blue and the resulting overall layer-wise abstraction in green.}%
	\label{fig:concept}
	\vspace{-4mm}
\end{figure*}

\begin{wrapfigure}[9]{r}{0.48 \textwidth}
	\centering
	\vspace{-4mm}
	\scalebox{0.92}{\input{figures/ReLU}}
	\vspace{-4mm}
	\caption{Convex single-neuron approximation (blue) of a ReLU (black) with bounded inputs $x\in [{l_x},{u_x}]$.}
	\label{fig:conv_barrier_ReLU}
\end{wrapfigure}
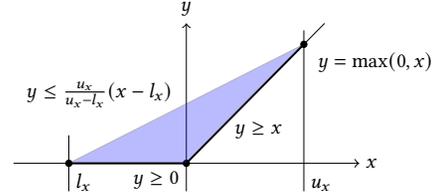
\paragraph{Key challenge: handling non-linearities} Neural networks interleave affine and non-linear activation layers (e.g., ReLU, Sigmoid), leading to highly non-linear behaviours. Because affine layers can be captured exactly using linear constraints, the key challenge in neural network verification rests in designing methods that can handle the effect of these non-linear activations in a precise and scalable manner.

Exact verification, e.g., \citep{tjeng2017evaluating, bunel2020branch, ehlers2017formal, katz2017reluplex, anderson2020strong, Pailoor19, wang2018neurify, singh2019boosting, wang2021beta}, has, in the worst-case, exponential complexity in the (large) number of non-linear activations due to a combinatorial blow-up of case distinctions (e.g., for ReLUs) and complex shapes for general activations (e.g., for Sigmoids). Therefore exact verifiers typically only handle piecewise linear activations and do not scale to larger networks.

To overcome this limitation, state-of-the-art verifiers, e.g., \citep{singh2019abstract, zhang2018crown, weng2018fastlin, Lirpa:20, singh2019beyond, tjandraatmadja2020convex}, often sacrifice completeness for scalability and leverage abstract interpretation \cite{Cousot96} to over-approximate the effect of each activation layer with convex polyhedra. Naturally, the scalability and precision of these incomplete methods are tied to the particular polyhedral fragment they utilize.

Below, we contrast different state-of-the-art abstraction approaches with our work by comparing the strong inter-neuron dependencies they can capture directly or indirectly, illustrated as solid or dashed lines, respectively, in \Figref{fig:concept} for a layer of four neurons. Individual abstractions are visualized in blue and the resulting layer-wise shape in green.

\paragraph{Optimal convex approximation}
Assume a layer of $n$ neurons, each applying the scalar, univariate, non-linear activation function $f \colon \R \rightarrow \R$ and the most precise polyhedral abstraction $\bc{P}$ of the layer's inputs $\bm{x}$. The most precise convex abstraction of the layer output is then given by the convex hull of all input-output vector pairs $\conv(\{(\bm{x},\bm{f}(\bm{x})) \, | \, \bm{x} \in \bc{P} \subseteq \R^n\})$, illustrated in \Figref{fig:concept}~(d), where all interactions are fully captured. Computing this $2n$-dimensional convex hull, however, is intractable due to the exponential cost $\bc{O}(n_v\log(n_v)+n_v^{n})$ \cite{chazelle1993optimal} in the number of neurons $n$, where the number of vertices $n_v = \bc{O}(n_c^{n})$ of the input polytope $\bc{P}$ is at worst also exponential in $n$ \cite{seidel1995upper} ($n_c$ is the number of constraints of the input polytope $\bc{P}$).

\paragraph{Single-Neuron approximation}
Most incomplete verifiers are fundamentally based on single-neuron convex abstractions, i.e., activations are approximated separately.
The tightest single-neuron abstractions maintain upper and lower bounds $l_x, u_x$ for each input $x$ and compute convex hulls of all input-output tuples: $\conv(\{({x},{f}({x})) \, | \, {x} \in [{l_x},{u_x}] \subseteq \R\})$, as illustrated in \Figref{fig:conv_barrier_ReLU} for a ReLU.
The union of the obtained constraints is the final abstraction of the layer. Geometrically, it is the Cartesian product of the convex hulls for each ReLU. This abstraction is significantly larger in volume (exponential in $n$) than the optimal convex hull discussed earlier, the key reason being that the interdependencies between neurons in the same layer are ignored, as illustrated in \Figref{fig:concept}~(a). Thus, the approximation error can grow exponentially with each layer, accumulating significant imprecision.

\paragraph{Multi-Neuron approximation}
To mitigate this limitation for ReLU networks, recent works \cite{singh2019beyond, tjandraatmadja2020convex, depalma2021scaling} introduced multi-neuron abstractions as a first compromise between the optimal but intractable layer-wise and the imprecise but scalable neuron-wise abstraction. \citet{singh2019beyond} partition the neurons of an activation layer into small sets of size $n_s\leq 5$, form groups of $k \leq 3$ neurons for each partition%
, approximate the group's input with octahedra \cite{clariso2007octahedron}, and then compute exact convex hulls \emph{jointly} approximating the \emph{output} of $k$ ReLUs for this input. These exact convex hull computations are computationally expensive and yield complex constraints, limiting the approach to only a few, mostly disjoint neuron groups, and restricting the number of captured dependencies, see \Figref{fig:concept}~(b). \citet{tjandraatmadja2020convex} and \citet{depalma2021scaling} merge the activation layer with the preceding affine layer and compute a convex approximation over the resulting multivariate activation layer for a hyperbox approximation of its input. This coarse input abstraction effectively restricts their approach to interactions over a single affine layer at a time. While both approaches currently yield state-of-the-art precision, they are limited to ReLU activations and lack scalability as they require small instances of the NP-hard convex hull problem to be solved exactly or large instances to be solved partially. They also do not address the problem of capturing enough neuron-interdependencies within a layer to come as close as possible to the optimal convex abstraction.

\paragraph{This work: precise multi-neuron approximations}
In this work, we present the first general verification framework for networks with arbitrary, bounded, multivariate activation functions called \tool (PRecIse Multi-neuron Abstraction). \tool builds on the group-wise approximations from \citet{singh2019beyond} and leverages the key insight that most interdependencies between neurons can be captured by considering a large number of relatively small, overlapping neuron-groups. While not achieving the tightness of the optimal convex approximation, \tool yields much tighter layer-wise approximations than previous methods, as shown in \Figref{fig:concept}~(c).

The key technical contributions of our work are: (i) \PDDM (\PDDMl) -- a general, precise, and fast convex hull approximation method for polytopes that enables the consideration of many neuron groups, and (ii) \SBLM (\SBLMl) -- a novel decomposition approach that builds upon the \PDDM to quickly compute multi-neuron constraints. While we combine these methods with abstraction refinement approaches in \tool, we note that they are also of general interest (beyond neural networks) and can be used independently of each other. 

\tool can be applied to any network with bounded, multivariate activation functions and arbitrary specifications expressible as polyhedra such as individual fairness \cite{ruoss2020learning}; global safety properties \cite{katz2017reluplex}; and acoustic  \cite{ryou2020fast}, geometric \cite{balunovic2019certifying}, spatial \cite{ruoss2020efficient}, and $\ell_p$-norm bounded perturbations \cite{gehr2018ai2}. Our experimental evaluation shows that \tool achieves state-of-the-art precision on the majority of our ReLU-based classifiers while remaining competitive on the rest. For Sigmoid- and Tanh-based networks, \tool significantly outperforms prior work on all benchmarks. Further, \tool enables, for the first time, precise and scalable verification of a realistic architecture for autonomous driving containing $>100$k neurons in a regression setting. Finally, while \tool is incomplete, it can be used for boosting the scalability of state-of-the-art complete verifiers \citep{singh2019boosting, wang2021beta} for ReLU-based networks that benefit from more precise convex abstractions.

\paragraph{Main contributions} Our key contributions are:
\begin{enumerate}
	\item \PDDM, a precise method for approximating the convex hull of polytopes, with worst-case polynomial time- and space-complexity and exactness guarantees in low dimensions.

	\item \SBLMl, a technique which efficiently computes joint constraints over groups of non-linear functions, by decomposing the underlying convex hull problem into lower-dimensional spaces.

	\item \tool, a novel verifier combining these approaches with a sparse neuron grouping technique and abstraction refinement, to obtain the first multi-neuron verifier for arbitrary, bounded, multivariate non-linear activations (e.g., ReLU, Sigmoid, Tanh, and MaxPool).
	
	\item An evaluation of \tool on a range of activations and network architectures (e.g., fully connected, convolutional, and residual). We show that \tool is significantly more precise than state-of-the-art, with gains of up to 20\%, 30\%, and 34\% for \mbox{ReLU-,} Sigmoid-, and Tanh-based networks, while being effective in a regression setting, scaling to large networks, and enabling verification in real-world settings such as autonomous driving.
\end{enumerate}
We release our code as part of the open-source framework ERAN at \url{https://github.com/eth-sri/eran}.

%% file: figures/concept_multi.tex
\begin{tikzpicture}
	\tikzset{>=latex}

	\def\x0{-3.5}
	\def\xx0{-0.0}
	\def\yy0{2.5}
	\def\dx{4.3}
	\def\ddx{1.5}
	\def\ddy{0.9}

	\node (single) [draw=black!60, fill=black!05, rectangle,
	minimum width=1.2cm, minimum height=2cm, align=center, scale=1.0, rounded corners=2pt,
	anchor=north] at (\x0 +0.4, 0) {
		\begin{tikzpicture}[scale=1.0]
		\node at (\xx0, \yy0 + 0.8) {};
		\node (A)[draw=black!80, fill=black!20, circle, minimum size=6pt, inner sep=0pt, anchor=center] at (\xx0, \yy0) {};	
		\node (B)[draw=black!80, fill=black!20, circle, minimum size=6pt, inner sep=0pt, anchor=center] at (\xx0, \yy0-\ddy) {};	
		\node (C)[draw=black!80, fill=black!20, circle, minimum size=6pt, inner sep=0pt, anchor=center] at (\xx0, \yy0-2*\ddy) {};	
		\node (D)[draw=black!80, fill=black!20, circle, minimum size=6pt, inner sep=0pt, anchor=center] at (\xx0, \yy0-3*\ddy) {};
				
		\node (single_A) [anchor=center, scale=0.22] at (\xx0 + \ddx, \yy0-0*\ddy){
			\begin{tikzpicture}
			\draw[->] (-2.0, 0) -- (2.0, 0);
			\draw[->] (0, -0.4) -- (0, 2.0);
			
			\def\a{-1.7}
			\def\b{1.7}
			\coordinate (a) at ({\a},{0});
			\coordinate (b) at ({\b},{\b});
			\coordinate (c) at ({0},{0});			
		
			\draw[fill=blue!90,opacity=0.3] (a) -- (c) -- (b) -- cycle;
			\draw[black,thick] (a) -- (c) -- (b);
			\draw[-] (b) -- ($(b)+(0.3,0.3)$);
			\end{tikzpicture}
		};
		\node (single_B) [anchor=center, scale=0.22] at (\xx0 + \ddx, \yy0-1*\ddy){
			\begin{tikzpicture}
			\draw[->] (-2.0, 0) -- (2.0, 0);
			\draw[->] (0, -0.4) -- (0, 2.0);
			
			\def\a{-1.7}
			\def\b{1.7}
			\coordinate (a) at ({\a},{0});
			\coordinate (b) at ({\b},{\b});
			\coordinate (c) at ({0},{0});
			
			\draw[fill=blue!90,opacity=0.3] (a) -- (c) -- (b) -- cycle;
			\draw[black,thick] (a) -- (c) -- (b);
			\draw[-] (b) -- ($(b)+(0.3,0.3)$);
			\end{tikzpicture}
		};
		\node (single_C) [anchor=center, scale=0.22] at (\xx0 + \ddx, \yy0-2*\ddy){
			\begin{tikzpicture}
			\draw[->] (-2.0, 0) -- (2.0, 0);
			\draw[->] (0, -0.4) -- (0, 2.0);
			
			\def\a{-1.7}
			\def\b{1.7}
			\coordinate (a) at ({\a},{0});
			\coordinate (b) at ({\b},{\b});
			\coordinate (c) at ({0},{0});
			
			\draw[fill=blue!90,opacity=0.3] (a) -- (c) -- (b) -- cycle;
			\draw[black,thick] (a) -- (c) -- (b);
			\draw[-] (b) -- ($(b)+(0.3,0.3)$);
			\end{tikzpicture}
		};
		\node (single_D) [anchor=center, scale=0.22] at (\xx0 + \ddx, \yy0-3*\ddy){
			\begin{tikzpicture}
			\draw[->] (-2.0, 0) -- (2.0, 0);
			\draw[->] (0, -0.4) -- (0, 2.0);
			
			\def\a{-1.7}
			\def\b{1.7}
			\coordinate (a) at ({\a},{0});
			\coordinate (b) at ({\b},{\b});
			\coordinate (c) at ({0},{0});
			
			\draw[fill=blue!90,opacity=0.3] (a) -- (c) -- (b) -- cycle;
			\draw[black,thick] (a) -- (c) -- (b);
			\draw[-] (b) -- ($(b)+(0.3,0.3)$);
			\end{tikzpicture}
		};
		\draw[->]  (A.east)+(0.1,0) -- ($(single_A.west)+(-0.1,0)$);
		\draw[->]  (B.east)+(0.1,0) -- ($(single_B.west)+(-0.1,0)$);
		\draw[->]  (C.east)+(0.1,0) -- ($(single_C.west)+(-0.1,0)$);	
		\draw[->]  (D.east)+(0.1,0) -- ($(single_D.west)+(-0.1,0)$);
		
		\draw [decorate,decoration={brace,mirror,amplitude=6pt}] ($(single_D.south west)+(-0.,-0.05)$) -- ($(single_D.south east)+(0.,-0.05)$);
		
		\node (single_J) [anchor=center, scale=0.40] at (\xx0 + \ddx +0.1, \yy0-4*\ddy-0.3){
			\begin{tikzpicture}
			
			\def\a{-1.7}
			\def\b{1.7}
			\coordinate (a) at ({\a},{0});
			\coordinate (b) at ({\b},{0});
			\coordinate (c) at ({\b},{\b/2});
			\coordinate (d) at ({\b/2},{\b});
			\coordinate (e) at ({\a},{\b});
						
			\draw[black,fill=green!60,opacity=0.3] (a) -- (b) -- (c) -- (d) -- (e) -- cycle;
			\draw[black,thick] (a) -- (b) -- (c) -- (d) -- (e) -- cycle;
			\end{tikzpicture}
		};
		
		\end{tikzpicture}
	};
	\node [anchor=north] at (\x0+0.3, -0.1) {Single-Neuron};
	\node [anchor=south west] at ($(single.south west)+(0.0, 0.0)$) {(a)};

  \node (few) [draw=black!60, fill=black!05, rectangle,
  minimum width=2.0cm, minimum height=2cm, align=left, scale=1.0, rounded corners=2pt,
  anchor=north] at (\x0 + \dx, 0) {
  	\begin{tikzpicture}[scale=1.0]
  	\node at (\xx0, \yy0 + 0.8) {};
  	\node (A)[draw=black!80, fill=black!20, circle, minimum size=6pt, inner sep=0pt, anchor=center] at (\xx0, \yy0) {};	
  	\node (B)[draw=black!80, fill=black!20, circle, minimum size=6pt, inner sep=0pt, anchor=center] at (\xx0, \yy0-\ddy) {};	
  	\node (C)[draw=black!80, fill=black!20, circle, minimum size=6pt, inner sep=0pt, anchor=center] at (\xx0, \yy0-2*\ddy) {};	
  	\node (D)[draw=black!80, fill=black!20, circle, minimum size=6pt, inner sep=0pt, anchor=center] at (\xx0, \yy0-3*\ddy) {};
  	\draw[-, thick]  (A) to [out=-120, in=120] (B);
	\draw[-, thick]  (C) to [out=-120, in=120] (D);
  	
  	\node (few_A) [anchor=center, scale=0.22] at (\xx0 + \ddx, \yy0-0.5*\ddy){
  		\begin{tikzpicture}
  		\coordinate (O) at (0, 0, 0);
		\coordinate (tc) at (2, 0, 0);
		\coordinate (bc) at (-2, 0, 0);
		\coordinate (mn) at (0, -2, 0);
		\coordinate (me) at (0, 0, 2);
		\coordinate (ms) at (0.5, 1.2, 0);
		\coordinate (ts) at (-1, 1, 0);
		\coordinate (mw) at (0, 0, -2);
		
		\coordinate (xp) at (2.5, 0, 0);
		\coordinate (xn) at (-2.5, 0, 0);
		\coordinate (yp) at (0, 2.0, 0);
		\coordinate (yn) at (0, -2.2, 0);
		\coordinate (zp) at (0, 0, 3.8);
		\coordinate (zn) at (0, 0, -2.5);
		
		\draw[->,opacity=0.7] (xn) -- (xp) ;
		\draw[->,opacity=0.7] (yn) -- (yp) ;
		\draw[->,opacity=0.7] (zn) -- (zp) ;
		
		\draw[fill=blue!90, opacity=0.15] (tc) -- (mn) -- (me) -- cycle;
		\draw[fill=blue!90, opacity=0.15] (tc) -- (mn) -- (mw) -- cycle;
		\draw[fill=blue!90, opacity=0.15] (tc) -- (ms) -- (me) -- cycle;
		\draw[fill=blue!90, opacity=0.15] (tc) -- (ms) -- (mw) -- cycle;
		\draw[fill=blue!90, opacity=0.15] (bc) -- (ts) -- (me) -- cycle;
		\draw[fill=blue!90, opacity=0.15] (bc) -- (ts) -- (mw) -- cycle;
		\draw[fill=blue!90, opacity=0.15] (bc) -- (mn) -- (me) -- cycle;
		\draw[fill=blue!90, opacity=0.15] (bc) -- (mn) -- (mw) -- cycle;
		\draw[fill=blue!90, opacity=0.15] (ts) -- (ms) -- (me) -- cycle;
		\draw[fill=blue!90, opacity=0.15] (ts) -- (ms) -- (mw) -- cycle; 
		
		\draw[opacity=0.5] (mn) -- (tc); 
		\draw[opacity=0.5] (mn) -- (bc); 
		\draw[opacity=0.5] (mn) -- (me); 
		\draw[opacity=0.5] (me) -- (ms); 
		\draw[opacity=0.5] (me) -- (ts);
		\draw[opacity=0.5] (me) -- (tc);
		\draw[opacity=0.5] (me) -- (bc);
		\draw[opacity=0.5] (tc) -- (ms);
		\draw[opacity=0.5] (bc) -- (ts);
		\draw[opacity=0.5] (ts) -- (ms);   
		
  		\end{tikzpicture}
  	};
  	\node (few_B) [anchor=center, scale=0.22] at (\xx0 + \ddx, \yy0-2.5*\ddy){
  		\begin{tikzpicture}
  		\coordinate (O) at (0, 0, 0);
		\coordinate (tc) at (2, 0, 0);
		\coordinate (bc) at (-2, 0, 0);
		\coordinate (mn) at (0, -2, 0);
		\coordinate (me) at (0, 0, 2);
		\coordinate (ms) at (0, 1.4, 0);
		\coordinate (ts) at (-1.5, 1.2, 0);
		\coordinate (mw) at (0, 0, -2);
		
		\coordinate (xp) at (2.5, 0, 0);
		\coordinate (xn) at (-2.5, 0, 0);
		\coordinate (yp) at (0, 2.0, 0);
		\coordinate (yn) at (0, -2.2, 0);
		\coordinate (zp) at (0, 0, 3.8);
		\coordinate (zn) at (0, 0, -2.5);
		
		\draw[->,opacity=0.7] (xn) -- (xp) ;
		\draw[->,opacity=0.7] (yn) -- (yp) ;
		\draw[->,opacity=0.7] (zn) -- (zp) ;
		
		\draw[fill=blue!90, opacity=0.15] (tc) -- (mn) -- (me) -- cycle;
		\draw[fill=blue!90, opacity=0.15] (tc) -- (mn) -- (mw) -- cycle;
		\draw[fill=blue!90, opacity=0.15] (tc) -- (ms) -- (me) -- cycle;
		\draw[fill=blue!90, opacity=0.15] (tc) -- (ms) -- (mw) -- cycle;
		\draw[fill=blue!90, opacity=0.15] (bc) -- (ts) -- (me) -- cycle;
		\draw[fill=blue!90, opacity=0.15] (bc) -- (ts) -- (mw) -- cycle;
		\draw[fill=blue!90, opacity=0.15] (bc) -- (mn) -- (me) -- cycle;
		\draw[fill=blue!90, opacity=0.15] (bc) -- (mn) -- (mw) -- cycle;
		\draw[fill=blue!90, opacity=0.15] (ts) -- (ms) -- (me) -- cycle;
		\draw[fill=blue!90, opacity=0.15] (ts) -- (ms) -- (mw) -- cycle; 
		
		\draw[opacity=0.5] (mn) -- (tc); 
		\draw[opacity=0.5] (mn) -- (bc); 
		\draw[opacity=0.5] (mn) -- (me); 
		\draw[opacity=0.5] (me) -- (ms); 
		\draw[opacity=0.5] (me) -- (ts);
		\draw[opacity=0.5] (me) -- (tc);
		\draw[opacity=0.5] (me) -- (bc);
		\draw[opacity=0.5] (tc) -- (ms);
		\draw[opacity=0.5] (bc) -- (ts);
		\draw[opacity=0.5] (ts) -- (ms);

  		\end{tikzpicture}
  	};
  	\draw[->]  (A.east)+(0.1,0) -- ($(few_A.west)+(-0.1,0.1)$);
  	\draw[->]  (B.east)+(0.1,0) -- ($(few_A.west)+(-0.1,-0.1)$);
  	\draw[->]  (C.east)+(0.1,0) -- ($(few_B.west)+(-0.1,0.1)$);	
  	\draw[->]  (D.east)+(0.1,0) -- ($(few_B.west)+(-0.1,-0.1)$);
  	
  	\draw [decorate,decoration={brace,mirror,amplitude=6pt}] ($(few_B.south west)+(-0.,-0.3)$) -- ($(few_B.south east)+(0.,-0.3)$);
  	
  	\node (single_J) [anchor=center, scale=0.40] at (\xx0 + \ddx +0.1, \yy0-4*\ddy-0.3){
  		\begin{tikzpicture}
  		
  		\def\a{-1.4}
  		\def\b{1.4}
			\coordinate (a) at ({\a},{0});
			\coordinate (b) at ({\b/2},{0});
			\coordinate (c) at ({\b},{\b/3});
			\coordinate (d) at ({\b},{\b/2});
			\coordinate (e) at ({\b/1.2},{\b/1.4});
			\coordinate (f) at ({\b/1.8},{\b/1.2});
			\coordinate (g) at ({0},{\b});
			\coordinate (h) at ({\a/2},{\b});
			\coordinate (i) at ({\a},{\b/1.5});
			
			\draw[black,fill=green!60,opacity=0.3] (a) -- (b) -- (c) -- (d) -- (e) -- (f) -- (g) -- (h) -- (i) -- cycle;
			\draw[black, thick] (a) -- (b) -- (c) -- (d) -- (e) -- (f) -- (g) -- (h) -- (i) -- cycle;
  		\end{tikzpicture}
  	};
  	
  \end{tikzpicture}
};
\node [anchor=north] at (\x0 + \dx, -0.1) {Disjunct Multi-Neuron};	
\node [anchor=south west] at ($(few.south west)+(0.0, 0.0)$) {(b)};	

\node (div) [draw=black!60, fill=black!05, rectangle,
minimum width=2.8cm, minimum height=2cm, scale=1.0, rounded corners=2pt,
anchor=north] at (\x0 + 2* \dx, 0) {
	\def\xx0{-1.5}
	\begin{tikzpicture}[scale=1.0]
	\node at (\xx0+1.6, \yy0 + 0.8) {};
	\node (A)[draw=black!80, fill=black!20, circle, minimum size=6pt, inner sep=0pt, anchor=center] at (\xx0, \yy0) {};	
	\node (B)[draw=black!80, fill=black!20, circle, minimum size=6pt, inner sep=0pt, anchor=center] at (\xx0, \yy0-\ddy) {};	
	\node (C)[draw=black!80, fill=black!20, circle, minimum size=6pt, inner sep=0pt, anchor=center] at (\xx0, \yy0-2*\ddy) {};	
	\node (D)[draw=black!80, fill=black!20, circle, minimum size=6pt, inner sep=0pt, anchor=center] at (\xx0, \yy0-3*\ddy) {};
  	\draw[-, thick]  (A) to [out=-120, in=120] (B);
	\draw[-, dashed, thick]  (A) to [out=-130, in=130] (D);
	\draw[-, thick]  (B) to [out=-120, in=120] (C);
	\draw[-, thick]  (C) to [out=-120, in=120] (D);
	
	\node (div_A) [anchor=center, scale=0.2] at (\xx0 + \ddx, \yy0-0.5*\ddy){
		\begin{tikzpicture}
		\coordinate (O) at (0, 0, 0);
		\coordinate (tc) at (2, 0, 0);
		\coordinate (bc) at (-2, 0, 0);
		\coordinate (mn) at (0, -2, 0);
		\coordinate (me) at (0, 0, 2.3);
		\coordinate (ms) at (0, 1.8, 0);
		\coordinate (ts) at (-1, 1.2, 0);
		\coordinate (mw) at (0, 0, -2);
		
		\coordinate (xp) at (2.5, 0, 0);
		\coordinate (xn) at (-2.5, 0, 0);
		\coordinate (yp) at (0, 2.0, 0);
		\coordinate (yn) at (0, -2.2, 0);
		\coordinate (zp) at (0, 0, 3.8);
		\coordinate (zn) at (0, 0, -2.5);
		
		\draw[->,opacity=0.7] (xn) -- (xp) ;
		\draw[->,opacity=0.7] (yn) -- (yp) ;
		\draw[->,opacity=0.7] (zn) -- (zp) ;
		
		\draw[fill=blue!90, opacity=0.15] (tc) -- (mn) -- (me) -- cycle;
		\draw[fill=blue!90, opacity=0.15] (tc) -- (mn) -- (mw) -- cycle;
		\draw[fill=blue!90, opacity=0.15] (tc) -- (ms) -- (me) -- cycle;
		\draw[fill=blue!90, opacity=0.15] (tc) -- (ms) -- (mw) -- cycle;
		\draw[fill=blue!90, opacity=0.15] (bc) -- (ts) -- (me) -- cycle;
		\draw[fill=blue!90, opacity=0.15] (bc) -- (ts) -- (mw) -- cycle;
		\draw[fill=blue!90, opacity=0.15] (bc) -- (mn) -- (me) -- cycle;
		\draw[fill=blue!90, opacity=0.15] (bc) -- (mn) -- (mw) -- cycle;
		\draw[fill=blue!90, opacity=0.15] (ts) -- (ms) -- (me) -- cycle;
		\draw[fill=blue!90, opacity=0.15] (ts) -- (ms) -- (mw) -- cycle;
		
		\draw[opacity=0.5] (mn) -- (tc); 
		\draw[opacity=0.5] (mn) -- (bc); 
		\draw[opacity=0.5] (mn) -- (me); 
		\draw[opacity=0.5] (me) -- (ms); 
		\draw[opacity=0.5] (me) -- (ts);
		\draw[opacity=0.5] (me) -- (tc);
		\draw[opacity=0.5] (me) -- (bc);
		\draw[opacity=0.5] (tc) -- (ms);
		\draw[opacity=0.5] (bc) -- (ts);
		\draw[opacity=0.5] (ts) -- (ms);

		\end{tikzpicture}
	};
	
	\node (div_B) [anchor=center, scale=0.2] at (\xx0 + \ddx, \yy0-1.5*\ddy){
		\begin{tikzpicture}
		\coordinate (O) at (0, 0, 0);
		\coordinate (tc) at (2, 0, 0);
		\coordinate (bc) at (-2, 0, 0);
		\coordinate (mn) at (0, -2, 0);
		\coordinate (me) at (0, 0, 2);
		\coordinate (mse) at (0.8, 1.2, 1.2);
		\coordinate (mw) at (0, 0, -2);
		\coordinate (bse) at (-0.8, 1.2,1.2);
		
		\coordinate (xp) at (2.5, 0, 0);
		\coordinate (xn) at (-2.5, 0, 0);
		\coordinate (yp) at (0, 1.5, 0);
		\coordinate (yn) at (0, -2.5, 0);
		\coordinate (zp) at (0, 0, 3.5);
		\coordinate (zn) at (0, 0, -2.5);
		\coordinate (zm) at (0, 0, {2*1.2/3.2});

		\draw[-,opacity=0.8] (zn) -- (zm) ;
		\draw[-,opacity=0.8] (xn) -- (xp) ;
		\draw[-,opacity=0.8] (yn) -- (yp) ;
		
		\draw[fill=blue!90, opacity=0.15] (bc) -- (mn) -- (tc) -- cycle;
		\draw[fill=blue!90, opacity=0.15] (bc) -- (bse) -- (mse) -- (tc) -- cycle;
		\draw[fill=blue!90, opacity=0.15] (bc) -- (bse) -- (mn) -- cycle;
		\draw[fill=blue!90, opacity=0.15] (mse) -- (bse) -- (mn) -- cycle;
		\draw[fill=blue!90, opacity=0.15] (tc) -- (mse) -- (mn) -- cycle;
		
		\draw[opacity=0.5] (tc) -- (mn);  
		\draw[opacity=0.5] (tc) -- (mse);
		\draw[opacity=0.5] (bc) -- (mn);  
		\draw[opacity=0.5] (bc) -- (bse);
		\draw[opacity=0.5] (bse) -- (mse);  
		\draw[opacity=0.5] (mn) -- (mse);
		\draw[opacity=0.5] (mn) -- (bse);
		
		\draw[-,opacity=0.8] (zm) -- (zp) ;
		\end{tikzpicture}
	};
	\node (div_C) [anchor=center, scale=0.20] at (\xx0 + \ddx, \yy0-2.5*\ddy){
		\begin{tikzpicture}
		\coordinate (O) at (0, 0, 0);
		\coordinate (tc) at (2, 0, 0);
		\coordinate (bc) at (-2, 0, 0);
		\coordinate (mn) at (0, -2, 0);
		\coordinate (me) at (0, 0, 2);
		\coordinate (ms) at (0, 2, 0);
		\coordinate (ts) at (-1, 1, 0);
		\coordinate (mw) at (0, 0, -2);
		
		\coordinate (xp) at (2.5, 0, 0);
		\coordinate (xn) at (-2.5, 0, 0);
		\coordinate (yp) at (0, 2.0, 0);
		\coordinate (yn) at (0, -2.2, 0);
		\coordinate (zp) at (0, 0, 3.8);
		\coordinate (zn) at (0, 0, -2.5);
		
		\draw[->,opacity=0.7] (xn) -- (xp) ;
		\draw[->,opacity=0.7] (yn) -- (yp) ;
		\draw[->,opacity=0.7] (zn) -- (zp) ;
		
		\draw[fill=blue!90, opacity=0.15] (tc) -- (mn) -- (me) -- cycle;
		\draw[fill=blue!90, opacity=0.15] (tc) -- (mn) -- (mw) -- cycle;
		\draw[fill=blue!90, opacity=0.15] (tc) -- (ms) -- (me) -- cycle;
		\draw[fill=blue!90, opacity=0.15] (tc) -- (ms) -- (mw) -- cycle;
		\draw[fill=blue!90, opacity=0.15] (bc) -- (ms) -- (me) -- cycle;
		\draw[fill=blue!90, opacity=0.15] (bc) -- (ms) -- (mw) -- cycle;
		\draw[fill=blue!90, opacity=0.15] (bc) -- (mn) -- (me) -- cycle;
		\draw[fill=blue!90, opacity=0.15] (bc) -- (mn) -- (mw) -- cycle;

		\draw[opacity=0.5] (mn) -- (tc); 
		\draw[opacity=0.5] (mn) -- (bc); 
		\draw[opacity=0.5] (mn) -- (me); 
		\draw[opacity=0.5] (me) -- (ms); 
		\draw[opacity=0.5] (me) -- (tc);
		\draw[opacity=0.5] (me) -- (bc);
		\draw[opacity=0.5] (tc) -- (ms);

		\end{tikzpicture}
	};
	\draw[->]  (A.east)+(0.1,0) -- ($(div_A.west)+(-0.1,0.1)$);
	\draw[->]  (B.east)+(0.1,0) -- ($(div_A.west)+(-0.1,-0.1)$);
	\draw[->]  (B.east)+(0.1,0) -- ($(div_B.west)+(-0.1,0.1)$);	
	\draw[->]  (C.east)+(0.1,0) -- ($(div_B.west)+(-0.1,-0.1)$);
	\draw[->]  (C.east)+(0.1,0) -- ($(div_C.west)+(-0.1,0.1)$);	
	\draw[->]  (D.east)+(0.1,0) -- ($(div_C.west)+(-0.1,-0.1)$);
	
	\draw [decorate,decoration={brace,mirror,amplitude=6pt}] ($(div_C.south west)+(-0.,-0.35)$) -- ($(div_C.south east)+(0.,-0.35)$);
	
	\node (single_J) [anchor=center, scale=0.40] at (\xx0 + \ddx+0.1, \yy0-4*\ddy-0.3){
		\begin{tikzpicture}
		
		\def\a{-1.3}
		\def\b{1.3}
		\coordinate (a) at ({\a/2},{0});
		\coordinate (b) at ({0},{0});
		\coordinate (c) at ({\b/2},{\b/3});
		\coordinate (d) at ({\b/2},{\b/2});
		\coordinate (e) at ({\b/4},{\b/1.5});
		\coordinate (f) at ({0},{\b/1.2});
		\coordinate (g) at ({\a/2},{\b});
		\coordinate (h) at ({\a},{\b/2});
		\coordinate (i) at ({\a},{\b/3});
		
		\draw[black,fill=green!60,opacity=0.3] (a) -- (b) -- (c) -- (d) -- (e) -- (f) -- (g) -- (h) -- (i) -- cycle;
		\draw[black, thick] (a) -- (b) -- (c) -- (d) -- (e) -- (f) -- (g) -- (h) -- (i) -- cycle;
		\end{tikzpicture}
	};
	
\end{tikzpicture}
};
\node [anchor=north] at (\x0 + 2*\dx, -0.1) {Overlapping Multi-Neuron};	
\node [anchor=south west] at ($(div.south west)+(0.0, 0.0)$) {(c)};	

\node (layer) [draw=black!60, fill=black!05, rectangle,
minimum width=2.2cm, minimum height=2cm, scale=1.0, rounded corners=2pt,
anchor=north] at (\x0 + 3* \dx+0.05, 0) {
	\def\xx0{-1.5}
	\begin{tikzpicture}[scale=1.0]
	\node at (\xx0+1.6, \yy0 + 0.8) {};
	\node (A)[draw=black!80, fill=black!20, circle, minimum size=6pt, inner sep=0pt, anchor=center] at (\xx0, \yy0) {};	
	\node (B)[draw=black!80, fill=black!20, circle, minimum size=6pt, inner sep=0pt, anchor=center] at (\xx0, \yy0-\ddy) {};	
	\node (C)[draw=black!80, fill=black!20, circle, minimum size=6pt, inner sep=0pt, anchor=center] at (\xx0, \yy0-2*\ddy) {};	
	\node (D)[draw=black!80, fill=black!20, circle, minimum size=6pt, inner sep=0pt, anchor=center] at (\xx0, \yy0-3*\ddy) {};
	\draw[-, thick]  (A) to [out=-120, in=120] (B);
	\draw[-, thick]  (A) to [out=-130, in=130] (D);
	\draw[-, thick]  (B) to [out=-120, in=120] (C);
	\draw[-, thick]  (C) to [out=-120, in=120] (D);	
	
	\node (layer_A) [anchor=center, scale=0.3] at (\xx0 + \ddx +0.2, \yy0-1.5*\ddy){
		\begin{tikzpicture}
		\coordinate (O) at (0, 0, 0);
		\coordinate (me) at (2, 0, 0);
		\coordinate (mw) at (-2, 0, 0);
		\coordinate (b) at (0, -1.5, 0);
		\coordinate (ms) at (0, 0, 2.0);
		\coordinate (t) at (0, 1.6, 0);
		\coordinate (tw) at (-1, 1.2, 0);
		\coordinate (te) at (1, 1.3, 0);
		\coordinate (mn) at (0, 0, -2);
		
		\coordinate (xp) at (2.5, 0, 0);
		\coordinate (xn) at (-2.5, 0, 0);
		\coordinate (yp) at (0, 2.0, 0);
		\coordinate (yn) at (0, -2.2, 0);
		\coordinate (zp) at (0, 0, 3.8);
		\coordinate (zn) at (0, 0, -2.5);
		
		\draw[->,opacity=0.7] (xn) -- (xp) ;
		\draw[->,opacity=0.7] (yn) -- (yp) ;
		\draw[->,opacity=0.7] (zn) -- (zp) ;
		
		\draw[fill=blue!90, opacity=0.15] (b) -- (mn) -- (me) -- cycle;
		\draw[fill=blue!90, opacity=0.15] (b) -- (mn) -- (mw) -- cycle;
		\draw[fill=blue!90, opacity=0.15] (b) -- (ms) -- (me) -- cycle;
		\draw[fill=blue!90, opacity=0.15] (b) -- (ms) -- (mw) -- cycle;
		
		\draw[fill=blue!90, opacity=0.15] (tw) -- (mw) -- (mn) -- cycle;
		\draw[fill=blue!90, opacity=0.15] (tw) -- (mw) -- (ms) -- cycle;
		\draw[fill=blue!90, opacity=0.15] (tw) -- (ms) -- (t) -- cycle;
		\draw[fill=blue!90, opacity=0.15] (tw) -- (mn) -- (t) -- cycle;
		
		\draw[fill=blue!90, opacity=0.15] (te) -- (me) -- (mn) -- cycle;
		\draw[fill=blue!90, opacity=0.15] (te) -- (me) -- (ms) -- cycle;
		\draw[fill=blue!90, opacity=0.15] (te) -- (ms) -- (t) -- cycle;
		\draw[fill=blue!90, opacity=0.15] (te) -- (mn) -- (t) -- cycle;
		
		\draw[opacity=0.5] (b) -- (me); 
		\draw[opacity=0.5] (b) -- (ms); 
		\draw[opacity=0.5] (b) -- (mw); 
		\draw[opacity=0.5] (me) -- (ms); 
		\draw[opacity=0.5] (ms) -- (mw);
		\draw[opacity=0.5] (me) -- (te);
		\draw[opacity=0.5] (te) -- (t);
		\draw[opacity=0.5] (tw) -- (t);
		\draw[opacity=0.5] (tw) -- (mw);
		\draw[opacity=0.5] (tw) -- (ms);
		\draw[opacity=0.5] (te) -- (ms);
		\draw[opacity=0.5] (t) -- (ms);

		\end{tikzpicture}
	};
	
	\draw[->]  (A.east)+(0.1,0) -- ($(layer_A.west)+(-0.0,0.3)$);
	\draw[->]  (B.east)+(0.1,0) -- ($(layer_A.west)+(-0.1,0.1)$);
	\draw[->]  (C.east)+(0.1,0) -- ($(layer_A.west)+(-0.1,-0.1)$);	
	\draw[->]  (D.east)+(0.1,0) -- ($(layer_A.west)+(-0.0,-0.3)$);
	
	\draw [decorate,decoration={brace,mirror,amplitude=6pt}] ($(layer_A.south west)+(-0.,-1.0)$) -- ($(layer_A.south east)+(0.,-1.0)$);
	
	\node (single_J) [anchor=center, scale=0.40] at (\xx0 + \ddx +0.25, \yy0-4*\ddy-0.3){
		\begin{tikzpicture}
		
		\def\a{-1.1}
		\def\b{1.1}
		\coordinate (a) at ({\a/2},{0});
		\coordinate (b) at ({0},{0});
		\coordinate (c) at ({\b/4},{\b/3});
		\coordinate (d) at ({\b/4},{\b/2});
		\coordinate (e) at ({\b/8},{\b/1.5});
		\coordinate (f) at ({0},{\b/1.2});
		\coordinate (g) at ({\a/2},{\b});
		\coordinate (h) at ({\a/1.5},{\b/2});
		\coordinate (i) at ({\a/1.5},{\b/3});
		
		\draw[black,fill=green!60,opacity=0.3] (a) -- (b) -- (c) -- (d) -- (e) -- (f) -- (g) -- (h) -- (i) -- cycle;
		\draw[black, thick] (a) -- (b) -- (c) -- (d) -- (e) -- (f) -- (g) -- (h) -- (i) -- cycle;
		\end{tikzpicture}
	};
	
\end{tikzpicture}
};
\node [anchor=north] at (\x0 + 3*\dx, -0.1) {Optimal Convex};	
\node [anchor=south west] at ($(layer.south west)+(0.0, 0.0)$) {(d)};	

\end{tikzpicture}

%% file: figures/ReLU.tex
\begin{tikzpicture}
	\draw[->] (-2.5, 0) -- (2.5, 0) node[right,scale=0.85] {$x$};
	\draw[->] (0, -0.4) -- (0, 2.0) node[above,scale=0.85] {$y$};

	\def\a{-1.7}
	\def\b{1.7}
	\coordinate (a) at ({\a},{0});
	\coordinate (b) at ({\b},{\b});
	\coordinate (c) at ({0},{0});
	
	\node[circle, fill=black, minimum size=3pt,inner sep=0pt, outer sep=0pt] at (a) {};
	\node[circle, fill=black, minimum size=3pt,inner sep=0pt, outer sep=0pt] at (b) {};
	\node[circle, fill=black, minimum size=3pt,inner sep=0pt, outer sep=0pt] at (c) {};
	
	\draw[fill=blue!90,opacity=0.3] (a) -- (c) -- (b) -- cycle;
	\draw[black,thick] (a) -- (c) -- (b);
	\draw[-] (b) -- ($(b)+(0.3,0.3)$);
	\draw[-] ({\a},-0.4) -- ({\a},0.4);
	\draw[-] ({\b},-0.4) -- ({\b},1.9);
	
	\node[anchor=south west,align=center,scale=0.85] at ({\a},-0.50) {$l_x$};
	\node[anchor=south west,align=center,scale=0.85] at ({\b},-0.50) {$u_x$};	
	\node[anchor=south east,align=center,scale=0.85] at ({-0.1},0.70) {$y\leq \frac{u_x}{u_x-l_x} (x-l_x)$};
	\node[anchor=north east,align=center,scale=0.85] at ({-0.0},0.00) {$y\geq 0$};
	\node[anchor=north west,align=center,scale=0.85] at ({0.6},0.70) {$y\geq x$};
	
	\node[anchor=south west,align=center,scale=0.85] at ($(b)+(0.1,-0.5)$) {$y=\max(0,x)$};

\end{tikzpicture}

%% file: background.tex
\section{Background} \label{sec:problem}

In this section, we establish the terminology we use to discuss polyhedra, neural networks (NNs) and their verification.

\paragraph{Notation}
We use lower case Latin or Greek letters $a, b, x, \dots, \lambda, \dots$ for scalars, bold for vectors $\bm{a}$, capitalized bold for matrices $\bm{A}$, and calligraphic $\bc{A}$ or blackboard bold $\mathbb{A}$ for sets. Similarly, we denote scalar functions as $f \colon \R^d \rightarrow \R$ and vector valued functions bold as $\bm{f} \colon \R^{d} \rightarrow \R^k$.

\paragraph{Neural networks}
We focus our discussion on networks $\bm{h}(\bm{x}) \colon \bc{X} \rightarrow \R^{|\bc{Y}|}$ that map input samples (images) $\bm{x} \in \bc{X}$ to numerical scores $\bm{y} \in \R^{|\bc{Y}|}$. For a classification task, the network $\bm{h}$ classifies an input $\bm{x}$ by applying argmax to its output: $c(\bm{x})=\argmax_{j} \bm{h}(\bm{x})_j$.
While our methods can refine the abstraction of activation functions in arbitrary neural architectures \cite{Lirpa:20}, for simplicity, we discuss a feedforward architecture which is an interleaved composition of affine functions $\bm{g}(\bm{x})=\bm{W}\bm{x}+\bm{b}$, such as normalization, linear, convolutional, or average pooling layers, with non-linear activation layers $\bm{f}(\bm{x})$ such as ReLU, Tanh, Sigmoid, or MaxPool: 
\begin{equation*} 
\bm{h}(\bm{x}) = \bm{g}_L \circ \bm{f}_L \circ \bm{g}_{L-1} \circ ... \circ \bm{f}_1 \circ \bm{g}_0(\bm{x}) .
\end{equation*}
\subsection{Neural Network Verification}
\tool is an optimization-based verification approach and supports any safety specification (pre- and post-condition) which can be expressed as a convex polyhedron. Examples of such specifications include but are not limited to individual fairness \cite{ruoss2020learning}, global safety properties \cite{katz2017reluplex}, acoustic  \cite{ryou2020fast}, geometric \cite{balunovic2019certifying}, spatial \cite{ruoss2020efficient}, and $\ell_p$-norm bounded perturbations \cite{gehr2018ai2}.

At its core, \tool is based on accumulating linear constraints encoding the whole network for a given (convex) pre-condition, defining a linear optimization objective representing the property to be verified, and finally using an LP solver to derive a bound on this objective. If this bound satisfies a predetermined threshold (that depends on the property), the property is verified.

While all affine layers (e.g., linear, convolutional, and normalization layers) can be encoded exactly using linear constraints, non-linearities have to be over-approximated via constraints in their input-output space. That is, for an activation layer $\bm{f} \colon \R^{n} \rightarrow \R^d$ and a given set of inputs $\bc{P}_\text{in} \subseteq \R^{n}$, we need to derive \emph{sound} output constraints, that represent a set $\bc{P}_\text{in-out} \subseteq \R^{d+n}$ which includes all possible input-output pairs that can be obtained by applying $\bm{f}$ to the inputs in $\bc{P}_\text{in}$.

We show an over-approximation for a \emph{single} ReLU in \Figref{fig:conv_barrier_ReLU}. In the concrete, the ReLU maps input $x$ to $y = \max(0,x)$. If the bounds $0 > l_x\leq x\leq u_x > 0$ are known, the best convex approximation is given by the blue triangle. In this work we present novel methods to compute tighter shapes by considering \emph{multiple neurons jointly} in a higher dimensional space.

\subsection{Overview of Convex Polyhedra}\label{sec:poly}
We now introduce the necessary background on polyhedra. A polyhedron can be represented as the convex hull of its extremal points, called the vertex- or \vrep, or as the subspace satisfying a set of linear constraints, called the halfspace constraint or \hrep. Simultaneously maintaining both representations of the same polyhedron is called double description.

\paragraph{Vertex representation}
A polyhedron $\bc{P} \subseteq \R^d$ is the closed convex hull of a set of generators called vertices \mbox{$\bc{R} = \{x_i \in \R^d\}$}:
\begin{equation*}
	\bc{P} = \bc{P}(\bc{R}) = \biggl\{\sum_{i} \lambda_i \bm{x}_i \, |\, \bm{x}_i \in \bc{R},  \; \sum_i \lambda_i = 1,\; \lambda_i \in \R^+_0 \biggr\},
\end{equation*}
where $\R^+_0$ are the positive real numbers including $0$. A polyhedral cone $\bc{P} \subseteq \R^d$ is the positive linear span of a set of generators called rays $\bc{R} = \{x_i \in \R^d\}$ and always includes the origin:
\begin{equation*}
	\bc{P} = \bc{P}(\bc{R}) = \biggl\{\sum_{i} \lambda_i \bm{x}_i \, | \, \bm{x}_i \in \bc{R}, \; \lambda_i \in \R^+_0 \biggr\}.
\end{equation*}
\paragraph{Halfspace representation}
Alternatively, a polyhedron can be described as the set $\bc{P} \subseteq \R^d$ satisfying a system of linear inequalities (or constraints) defined by $\bm{A} \in \R ^{m \times d}$ and $\bm{b} \in \R^{m}$:
\begin{equation*}
	\bc{P} = \bc{P}(\bm{A},\bm{b}) \equiv \{\bm{x} \in \R^d \,|\, \bm{A} \bm{x} \geq \bm{b}\}.
	\label{eqn:polytope}
\end{equation*}
Geometrically, $\bc{P}$ is the intersection of $m$ closed affine halfspaces $\bc{H}_i = \{x \in \R^d \mid \bm{a}_i \bm{x} \geq b_i\}$ with $\bm{a}_i \in \R ^{d}$ and $b_i \in \R$.
For a polyhedral cone we have $\bm{b} = \bm{0}$. For convenience, a polyhedron $\bc{P}(\bm{A},\bm{b})$ can be equivalently described in so-called homogenized coordinates $\bm{x}'=[1,\bm{x}]$, where it can be expressed as \mbox{$\bc{P}(\bm{A}')=\{\bm{x'} \in \R^{d+1} ~|~ \bm{A}' \bm{x}' \geq 0\}$} with the new constraint matrix $\bm{A}'=[-\bm{b},\bm{A}]$.

A $k$-face $\bc{F}$ of a $d$-dimensional polyhedron is a $k$-dimensional subset $\bc{F} \subseteq \bc{P}$ satisfying $d-k$ linearly independent constraints\footnote{We call a set of constraints $\bm{a}_i\bm{x}\geq b_i$ linearly independent, if the $\bm{a}_i$ are linearly independent.} with equality. We call a $0$-face a vertex and a ($d-1$)-face a facet \cite{edelsbrunner2012algorithms}.
The rank of a ray or vertex in a $d$-dimensional polyhedron is the number of linearly independent constraints it satisfies with equality. We call a ray of rank $d-1$ and a vertex of rank $d$ extremal. A ray of rank $d-n$ can be represented as the positive combination of $n$ extremal rays and a vertex of rank $d-n$ as the convex combination of $n+1$ extremal points.

\paragraph{Double description}
Polyhedra static analysis \citep{motzkin1953double,fukuda1995double, Singh:17} usually maintains both representations ($\bc{H}$ and $\bc{V}$) in a pair $(\bm{A}',\bc{R})$, called double description. This is useful as computing the convex hull in the \vrep is trivial (union of generator sets), but computing intersections is NP-hard. Conversely, computing intersections in the \hrep is trivial (union of constraints), but computing the convex hull is NP-hard. The transformation from the $\bc{V}$- to the \hrep is called the \emph{convex hull} problem and the reverse the \emph{vertex enumeration} problem. Both are NP-hard in general.

\paragraph{Inclusion}
We define the inclusion of a polytope $\bc{Q}$ in a polytope $\bc{P}$ as: $\bc{Q} \subseteq \bc{P}$ or equivalently, $\forall \bm{x} \in \bc{Q}, \bm{x} \in \bc{P}$. In this setting, we say $\bc{P}$ over-approximates $\bc{Q}$ and $\bc{Q}$ under-approximates $\bc{P}$.

%% file: overview.tex
\section{Overview of \tool} \label{sec:overview}

We now present an overview of \tool, our framework for faster and more precise verification of neural networks with arbitrary, bounded, multivariate, non-linear activations. We provide a complete formal description of its main components \PDDM and \SBLM in Sections~\ref{sec:PDDM} and \ref{sec:FastPoly}, and of \tool in Section~\ref{sec:framework}. In our explanations, we follow the setup outlined in~\Secref{sec:intro}: an activation layer consisting of $n$ neurons representing non-linear activations $f(x)$ (e.g., ReLU, Tanh, Sigmoid).

\paragraph{Computing a convex approximation of a whole layer}
Conceptually, given an $n$-dimensional polytope $\bc{S}$ constraining the input to the activation layer, \tool computes a set of multi-neuron constraints, forming a convex over-approximation of this layer, as follows:
\begin{enumerate}
\item {\em Group decomposition:} Decompose the set of $n$ activations in the layer into overlapping groups (subsets) of size $k$.
\item {\em Octahedral projection:} For each such group $i$, compute an octahedral over-approximation $\bc{P}^i$ of the projection of $\bc{S}$ to the input-space of group $i$.
\item {\em \SBLMl (\SBLM):} Then, for each polytope $\bc{P}^i$, compute a joint convex over-approximation $\bc{K}^i$ of the group output in the \hrep using our novel \SBLM method. This method decomposes the problem into lower dimensions and leverages our novel \PDDMl (\PDDM) with polynomial complexity to compute fast and scalable convex hull approximations. Both \SBLM and \PDDM are also key to making \tool applicable to non-piecewise-linear activations.
\item {\em Combine constraints:} Finally, take the intersection of all output constraints $\bc{K}^i$ (a union of all constraints) to obtain an over-approximation of the entire layer output.
\end{enumerate}

Verification is performed by solving an LP problem which combines the generated multi-neuron constraints with an LP encoding of the whole network (evaluated in~\Secref{sec:experiments}). We now explain the basic workings of each step and illustrate the key concepts on a running example.

\paragraph{Group decomposition}
Computing convex hulls for large sets of activations (e.g., a whole layer) is infeasible. Thus, we consider groups of size $k$, typically $k= 3$ or $4$. The key idea here is to capture dependencies between activation inputs and outputs ignored by neuron-wise approximations and thus achieve tighter approximations. The tightness increases with the number of groups and, importantly, the degree of overlap between them. Considering all possible $\binom{n}{k}$ groups for every layer is too expensive; thus we define the parameters partition size $n_s$ and group overlap $s$ for tuning the cost and precision of our approximations. We first partition the activations of a layer into sets of size $n_s$ (sorting by volume of the single neuron abstraction) and then for every set\footnote{For piecewise-linear activations, typically $n_s$ is chosen large enough such that there is only one set.} 
choose a subset of all $\binom{n_s}{k}$ groups that pairwise overlap by at most $s$, $0\leq s < k$.

\begin{wrapfigure}[9]{r}{0.35 \textwidth}
	\centering
	\vspace{-3.5mm}
	\scalebox{1.0}{\input{figures/overview_input_constraints.tex}}
	\vspace{-2mm}
	\caption{Exact projection of $\bc{S} \in \R^3$ (left) to $k=2$ variables (green) and its octahedral over-approximation $\bc{P}^i$ (blue).
	}
	\label{fig:overview_poly_input}
	\vspace{-3mm}
\end{wrapfigure}
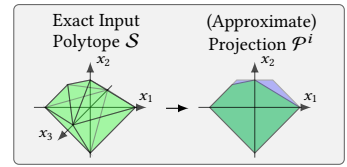
\paragraph{Octahedral projection}
Projecting the layer-wise input polytope $\bc{S}$ onto the input dimensions of every group is generally intractable due to the high dimensionality and large number of constraints. Therefore, we follow the idea of \cite{singh2019beyond} and over-approximate the projection. Empirically we find that multidimensional octahedra \cite{clariso2007octahedron}, yielding $3^k-1$ input constraints per group of $k$ neurons, provide a good trade-off between accuracy and complexity. Such a projection is illustrated in \Figref{fig:overview_poly_input} for a layer of $n=3$ neurons and $k=2$.

\subsection{\SBLMl}

The next and most demanding step takes a $k$-dimensional input polytope for a given $k$-activation group, and computes a $2k$-dimensional convex over-approximation of the output of the corresponding $k$ activations. We introduce a new technique, called \SBLMl, and illustrate its workings in~\Figref{fig:overview} on an example. We assume ReLU activations, group size \mbox{$k=2$}, and an octahedral input polytope ${\bc{P}^i}$ (left panel in \Figref{fig:overview}) described by
\begin{align*}
	{\bc{P}^i}=\{ &x_1 + x_2 \geq -2, \; - x_1 + x_2 \geq -2,   \;	x_1 - x_2 \geq -2, - x_1 - x_2 \geq -2,\; - x_2 \geq -1.2\}.
\end{align*}
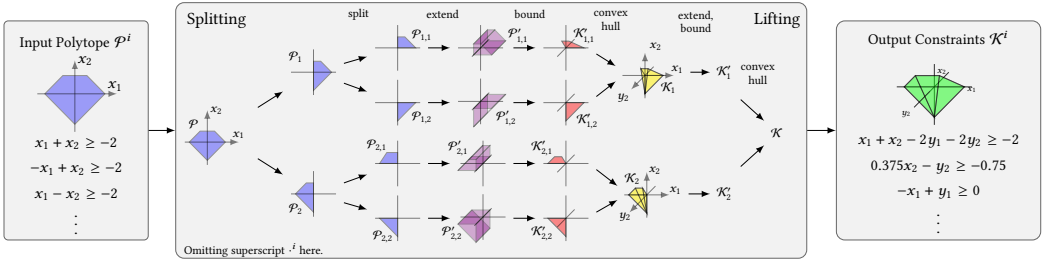
\begin{figure*}[t]
	\centering
	\scalebox{0.8}{\input{figures/overview_figure_new}}
	\vspace{-6mm}
	\caption{Illustration of the \SBLMl for a group of $k=2$ neurons and a ReLU activation.	}
	\label{fig:overview}
	\vspace{-1mm}
\end{figure*}
Our method has three main components explained next.

\paragraph{Split the input polytope}
We first split ${\bc{P}^i}$ into regions, which we call quadrants, for which tight or even exact, linear bounds of the activation functions are available. Choosing the right splits is essential for ensuring tight approximations.
For piecewise-linear activation functions (like ReLU), splitting into their linear regions even yields exact bounds in every quadrant, leading to the tightest approximations.
For our example with ReLU activations, this corresponds to splitting along hyperplanes where the input variables $x_1$ and $x_2$ are $0$. We (randomly) choose the ordering $\{y_1,y_2\}$ of output variables and split $\bc{P}^i$ (in the following we omit the superscript $i$) along the corresponding hyperplanes. That is, we first intersect $\bc{P}$ with the halfspaces $\{\bm{x} \in\R^2 \,|\, x_1 \geq 0 \}$ and $\{\bm{x} \in\R^2 \,|\, x_1 \leq 0 \}$, obtaining $\bc{P}_1$ and $\bc{P}_2$, and then $\bc{P}_1$ and $\bc{P}_2$ with \mbox{$\{\bm{x} \in\R^2 \,|\, x_2 \geq 0 \}$} and $\{\bm{x} \in\R^2 \,|\, x_2 \leq 0 \}$. These intersections generate a tree of polytopes visualized in the first three columns in the central panel of~\Figref{fig:overview} with the quadrants as leafs (third column). For brevity, we only follow the bottom half. There, the two quadrants $\bc{P}_{2,1}$ and $\bc{P}_{2,2}$ are described by
\begin{align*}
	&{\bc{P}_{2,1}}=\{
	x_1  - x_2\geq -2,\;
	-x_1 \geq 0,\;
	-\,x_2 \geq -1.2,\;
	x_2 \geq 0\},\\
	&{\bc{P}_{2,2}}=\{
	x_1 + x_2 \geq -2,\;
	-x_1 \geq 0,\;
	-x_2 \geq 0\}.
\end{align*}

In the second part of the algorithm, we lift these quadrants step-by-step from the space of only their inputs to the space of both their inputs and outputs. We will now describe one step of lifting consisting of extending, bounding and computing a convex hull.

\paragraph{Extend and bound the quadrants} We extend\footnote{Extending a $d$-dimensional polytope by a variable defines it in the $d+1$-dimensional space, where it is (initially) unbounded in the dimension of the added variable.} the quadrants one output variable at a time, which, as we will see later, enables significant gains in speed while reducing the approximation error. In our example, we first trivially extend all quadrants from the $(x_1,x_2)$-space to the $(y_2,x_1,x_2)$-space (fourth column in~\Figref{fig:overview}).
Next, we bound the quadrants in the added dimension using the linear bounds (parametrically defined, see Section~\ref{sec:FastPoly}) corresponding to applying (an approximation of) the activation in the quadrant. Here, $y_2 \leq x_2$ and $y_2 \geq x_2$ for the quadrant $\bc{P}_{2,1}$ (since $x_2\geq0$) and $y_2 \leq 0$ and $y_2 \geq 0$ for the quadrant $\bc{P}_{2,2}$ (since $x_2\leq 0$). Note that in this case the bounds we apply on every quadrant are exact, yielding the two polytopes (fifth column) with 0 volume in their $3d$-space (in general, the bounds need not be exact):
\vspace{-0.5mm}
\begin{align*}
	{\bc{K}_{2,1}'}=\{
	& x_1-x_2 \geq -2,\;
	-x_1 \geq 0,\;
	-x_2 \geq -1.2,\;
	x_2 \geq 0, x_2 -y_2 \geq 0,\;
	-x_2 + y_2\geq 0\},\\
	{\bc{K}_{2,2}'}=\{&
	x_1+ x_2 \geq -2,\;
	-x_1 \geq 0,\;
	- x_2 \geq  0,\;
	-y_2\geq 0,\;
	y_2\geq 0\}.
\end{align*}

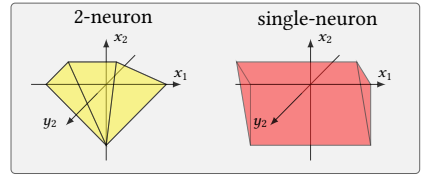
\begin{wrapfigure}[10]{r}{0.41\textwidth}
	\centering
	\vspace{-3.5mm}
	\scalebox{0.8}{\input{figures/precision_cmp}}
	\vspace{-1mm}
	\caption{Comparison of 2-neuron and 1-neuron constraints projected into $y_2$-$x_1$-$x_2$-space for a ReLU activation, given input polytope $\bc{P}^i$.}
	\label{fig:precision_comp}
\end{wrapfigure}
\paragraph{Approximate convex hull}
Next, we compute the convex hull of $\bc{K}_{2,1}'$ and $\bc{K}_{2,2}'$. Instead of using an exact method, we utilize our \PDDM to compute precise over-approximations, leveraging the concept of duality, ideas from computational geometry and our novel \PDD polyhedron representation (explained below and in more detail in \Secref{sec:PDDM}). Note that because the considered quadrants are only extended one variable at a time, the computation takes place in $3d$ despite the group-output being in the $4d$ $(y_1,y_2,x_1,x_2)$-space. This yields two main benefits: (i) precision -- directly computing $2k$-dimensional convex hulls with \PDDM will lose more precision than our decomposed method, because \PDDM is exact for polytopes of dimension up to three and loses precision only slowly for higher dimensions, and (ii) speed -- a lower-dimensional polytope with fewer constraints and generally also fewer vertices significantly reduces the time required for the individual convex hull computations.

Importantly, our approximate method scales quartically as $\bc{O}\{n_a^4 \cdot n_v + n_a^2 \log(n_a^2)\}$ in the number of input constraints $n_a$ and linear in the number of vertices $n_v$ (see Theorem~\ref{lem:complexity}) while optimal exact methods are in $\bc{O}(n_v\log(n_v)+n_v^{\lfloor d/2 \rfloor})$ \cite{chazelle1993optimal}, i.e., exponential in the number of dimensions and superlinear in the number of input vertices. 

Note that for non-piecewise-linear functions (e.g., Tanh or Sigmoid), the number of vertices doubles when extending by a dimension. This makes exact methods intractable and approximate methods not using the decompositional \SBLM approach (that is, extending by all dimensions at the same time) slow (see our evaluation in Section~\ref{sec:experiments}).

We now obtain the convex hull (sixth column) of the two polytopes $\bc{K}_{2,1}'$ and $\bc{K}_{2,2}'$ which is exact in our $3d$ case:
\begin{align*}%
	\bc{K}_{2} = \{&x_1+x_2-2y_2\geq -2,
	-x_1 \geq 0,\;
	0.375  x_2-y_2\geq -0.75, -x_2 +y_2\geq 0,\;
	y_2 \geq 0\}.
\end{align*}

We compute $\bc{K}_{1}$ analogously, thus completing the first step of lifting. The next and in this case final step of lifting starts with extending $\bc{K}_{1}$ and $\bc{K}_{2}$ by $y_1$ into the $(y_1,y_2,x_1,x_2)$-space, where we apply bounds on $y_1$ yielding (in $4d$ and thus not illustrated as figure)
\begin{align*}%
	\mathbf{\bc{K}_{1}'} = \{&-x_1+x_2-2y_2\geq -2,
	x_1 \geq 0,\;
	0.375 x_2-y_2\geq -0.75,\\
	&-x_2 +y_2\geq 0,\;
	y_2 \geq 0,\;
	x_1-y_1 \geq 0,\;
	-x_1+y_1 \geq 0\},\\
	\mathbf{\bc{K}_{2}'} = \{&x_1+x_2-2y_2\geq -2,
	-x_1 \geq 0,\;
	0.375  x_2-y_2\geq -0.75,\\
	&-x_2 +y_2\geq 0,\;
	y_2 \geq 0,\;
	-y_1 \geq 0,\;
	y_1 \geq 0\}.
\end{align*}
Completing the second and final step of lifting by computing their convex hull yields the final tight 2-neuron constraints:
\begin{align*}%
	\mathbf{\bc{K}} = \{& x_1+ x_2 - 2 y_1 -2 y_2 \geq -2,\;
	0.375 x_2 - y_2 \geq -0.75,\\
	&-x_1 + y_1 \geq 0,\;
	-x_2+ y_2 \geq 0,\;
	y_1 \geq 0,\;
	y_2 \geq 0\}.
\end{align*}
Naturally, the region $\bc{K}$ is tighter than the tightest single-neuron approximations (triangle relaxation, discussed earlier). We illustrate this point by comparing their projections into the $(y_2,x_1,x_2)$-space in \Figref{fig:precision_comp}.

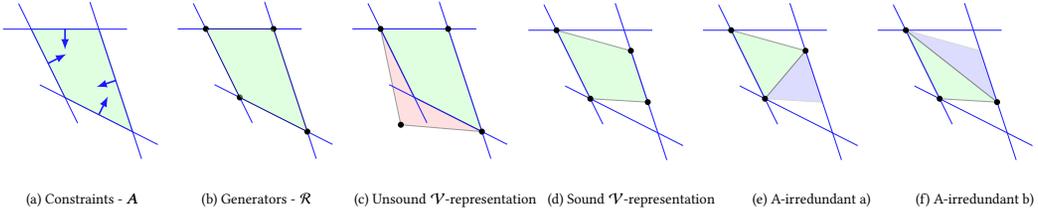
\begin{figure}[t]
	\centering
	\scalebox{0.75}{\input{figures/PDD_fig.tex}}
	\vspace{-1mm}
	\caption{Illustration of the Partial Double Description. Input constraints $\bm{A}$ (a), exact vertex enumeration $\bc{R}_{DD}$ (b), unsound partial vertex enumeration violating the PDD definition (c), partial or approximate vertex enumeration $\bc{R}_{PDD}$ (d), and A-irredundant versions of the partial vertex enumeration (e).
	}
	\label{fig:PDD_ill}
	\vspace{-1mm}
\end{figure}

\subsection{\PDDMl (\PDDM)}

We now introduce the new \PDDM for computing \emph{fast, precise, and sound} over-approximations of the convex hull of two polyhedra. This is in contrast to existing approximation methods, which either optimize for closer approximations \cite{bentley1982approximation,khosravani2013simple,zhong2014finding,sartipizadeh2016computing} but sacrifice the soundness required for verification, or have exponential complexity \cite{xu1998approximate}, making them too expensive for our application. %

\paragraph{Double description method}
The well-known Double Description Method (DDM) \cite{motzkin1953double,fukuda1995double} for computing the convex hull of two polyhedra in Double Description works as follows: (i) translate both polyhedra to their dual representation (explained in \Secref{sec:PDDM}), (ii) intersect them in dual space by adding the constraints of one to the other, one-at-a-time, computing full Double Descriptions at every intermediate step, and (iii) translate the result back to primal space. Every step of adding an additional constraint generates quadratically many new vertices, leading to an overall increase exponential in the number of constraints (in dual space).

\paragraph{Partial double description}
We introduce the Partial Double Description (PDD), which guarantees soundness and also allows an approximate much cheaper intersection in dual space. We combine an exact \hrep, as their intersection is trivial, with an under-approximating\footnote{An under-approximation in dual space corresponds to an over-approximation in primal space, due to inclusion reversion.} \vrep, as their exact intersection carries exponential cost. We illustrate this in \Figref{fig:PDD_ill}, where we show the constraints $\bm{A}$ describing a polytope in (a), the corresponding exact \vrep in (b), an unsound approximate \vrep in (c), and three sound ones in (d), (e), and (f). Note that this definition of the \PDD allows many different \vreps for a given \hrep (see (d), (e), and (f) in \Figref{fig:PDD_ill}) some of which are quite imprecise (see (e) and (f)).

\begin{figure}[t]
	\centering		
	\scalebox{0.62}{\input{figures/overview_PDDM_figure}}
	\vspace{-2.5mm}
	\caption{\PDDMl for a $2$-dimensional example. The input polytopes (1\St column) are translated to their dual representation (2\Nd column), then all their constraints are added to the other dual polytope (3\Rd column). The points are separated based on whether they are included in the intersection of the \hreps. Now ray-shooting is used to discover vertices on the rays between points in the intersection (blue points) to those outside (red points) by intersecting the rays with the constraints added in the previous step (4\Th column). The vertices of both \vreps are then combined (5\Th column) before A-irredundancy is enforced (6\Th column) and the result is translated back to primal space (7\Th column).
	}
	\label{fig:overview_PDDM}
	\vspace{-4mm}
\end{figure}
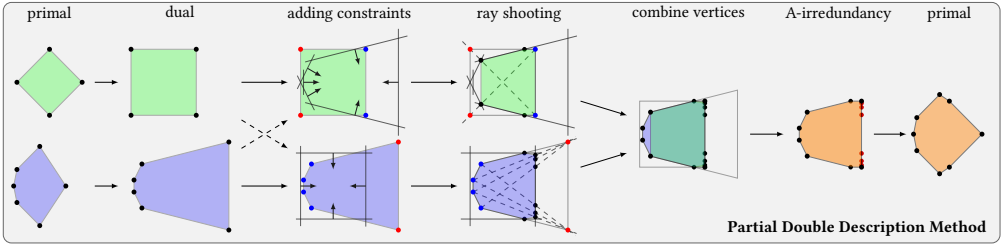

\paragraph{Partial double description method}
Now, we define the \PDDM to compute approximate convex hulls in \PDD leveraging two key ideas: (i) instead of intersecting in dual space by adding the constraints of one polytope to the other one-at-a-time (as per DDM), we add them all in a single step. Crucially, this leads to an \emph{overall} number of vertices at most quadratic (instead of exponential) in the number of original vertices (in dual space), and (ii) this single-step approach is asymmetric and we can greatly increase the intersection accuracy, by performing it in both directions and combining the resulting vertices. Overall, our approach yields a polynomial complexity (see Theorem \ref{lem:complexity}) algorithm for sound convex hull approximations (see Theorem \ref{lem:soundness}), guarantees exactness for low dimensions (see Theorem~\ref{lem:inter_approx_exact}), and empirically is two orders of magnitude faster for the challenging cases in our experiments (see \Figref{fig:speedup}), while losing precision only slowly as dimensionality increases (see \Figref{fig:volume}).
We illustrate the \PDDMl in \Figref{fig:overview_PDDM} and provide more technical details in~\Secref{sec:PDDM}.

\subsection{Layerwise Abstraction}

So far we have seen how to compute the multi-neuron convex approximation for a single group of $k$ activations. To compute the final abstraction of the whole activation layer, we combine the constraints forming the \hreps of the computed output polyhedra of each group, thereby obtaining the \hrep of the polytope describing the layerwise over-approximation.

%% file: figures/overview_input_constraints.tex
\begin{tikzpicture}
	\tikzset{>=latex}
	
	\def\dx{1.1}
	\def\yo{-2.15}
	
	\node (B)
	[draw=black!60, fill=black!05, rectangle, rounded corners=2pt,
	minimum width=4.4cm, minimum height=2.10cm,
	anchor=north
	] at (0, 0) {
	};
	\node [anchor=north,align=center,font=\scriptsize] at (-\dx,0.0) { Exact Input \\ Polytope $\bc{S}$};
	\node [anchor=north,align=center,font=\scriptsize] at (\dx,0.0) { (Approximate) \\ Projection $\bc{P}^i$};

	\node[anchor=south] (p0) at (-\dx, \yo) {
		\begin{tikzpicture}[scale=0.3]

	  		\coordinate (O) at (0, 0, 0);
			\coordinate (tc) at (2, 0, 0);
			\coordinate (bc) at (-2, 0, 0);
			\coordinate (mn) at (0, -2, 0);
			\coordinate (me) at (0, 0, 2);
			\coordinate (ms) at (0, 1.2, 0);
			\coordinate (ts) at (-1, 1, 0);
			\coordinate (mw) at (0, 0, -2);
			
			\coordinate (xp) at (2.5, 0, 0);
			\coordinate (xn) at (-2.5, 0, 0);
			\coordinate (yp) at (0, 2.0, 0);
			\coordinate (yn) at (0, -2.2, 0);
			\coordinate (zp) at (0, 0, 3.8);
			\coordinate (zn) at (0, 0, -2.5);

			\draw[->,opacity=0.7] (xn) -- (xp) ;
			\draw[->,opacity=0.7] (yn) -- (yp) ;
			\draw[->,opacity=0.7] (zn) -- (zp) ;
			
			\draw[fill=green!90, opacity=0.2] (tc) -- (mn) -- (me) -- cycle;
			\draw[fill=green!90, opacity=0.2] (tc) -- (mn) -- (mw) -- cycle;
			\draw[fill=green!90, opacity=0.2] (tc) -- (ms) -- (me) -- cycle;
			\draw[fill=green!90, opacity=0.2] (tc) -- (ms) -- (mw) -- cycle;
			\draw[fill=green!90, opacity=0.2] (bc) -- (ts) -- (me) -- cycle;
			\draw[fill=green!90, opacity=0.2] (bc) -- (ts) -- (mw) -- cycle;
			\draw[fill=green!90, opacity=0.2] (bc) -- (mn) -- (me) -- cycle;
			\draw[fill=green!90, opacity=0.2] (bc) -- (mn) -- (mw) -- cycle;
			\draw[fill=green!90, opacity=0.2] (ts) -- (ms) -- (me) -- cycle;
			\draw[fill=green!90, opacity=0.2] (ts) -- (ms) -- (mw) -- cycle;
			
			\draw[opacity=0.5] (mn) -- (tc); 
			\draw[opacity=0.5] (mn) -- (bc); 
			\draw[opacity=0.5] (mn) -- (me); 
			\draw[opacity=0.5] (me) -- (ms); 
			\draw[opacity=0.5] (me) -- (ts);
			\draw[opacity=0.5] (me) -- (tc);
			\draw[opacity=0.5] (me) -- (bc);
			\draw[opacity=0.5] (tc) -- (ms);
			\draw[opacity=0.5] (bc) -- (ts);
			\draw[opacity=0.5] (ts) -- (ms);

			\node [font=\scriptsize, align=center, scale=0.8] at (2.5,-0.1,0) {$x_1$};
			\node [font=\scriptsize, align=center, scale=0.8] at (0.6,1.5,0) {$x_2$};
			\node [font=\scriptsize, align=center, scale=0.8] at (-0.8,-0.6,3) {$x_3$};

		\end{tikzpicture}
	};
	
	\node[anchor=south] (p1) at (\dx, \yo) {
		\begin{tikzpicture}[scale=0.3]

			\coordinate (tc) at (2, 0, 0);
			\coordinate (ms) at (0.8, 1.2, 0);
			\coordinate (bs) at (-0.8, 1.2, 0);
			\coordinate (bc) at (-2, 0, 0);
			\coordinate (mn) at (0, -2, 0);

			\coordinate (ms1) at (0, 1.2, 0);
			\coordinate (ts1) at (-1, 1, 0);

			\draw[fill=blue!90, opacity=0.3] (tc) -- (ms) -- (bs) -- (bc) -- (mn) -- cycle;
			\draw[fill=green!90, opacity=0.4] (tc) -- (ms1) -- (ts1) -- (bc) -- (mn) -- cycle;
			
			\coordinate (xp) at (2.5, 0, 0);
			\coordinate (xn) at (-2.5, 0, 0);
			\coordinate (yp) at (0, 2.0, 0);
			\coordinate (yn) at (0, -2.2, 0);
			
			\draw[->,opacity=0.7] (xn) -- (xp) ;
			\draw[->,opacity=0.7] (yn) -- (yp) ;

			\node [font=\scriptsize, align=center, scale=0.8] at (2.5,-0.1,0) {$x_1$};
			\node [font=\scriptsize, align=center, scale=0.8] at (0.6,1.5,0) {$x_2$};
			
		\end{tikzpicture}
	};

	\draw[->] ($(p1.west)+(-0.3,-0.06)$) -- ($(p1.west)+(0,-0.06)$);

\end{tikzpicture}

%% file: figures/overview_figure_new.tex
\begin{tikzpicture}
  \tikzset{>=latex}

	\def\xO{0.0}

  \node (B)
	[draw=black!60, fill=black!05, rectangle, rounded corners=2pt,
	minimum width=2.4cm, minimum height=3.6cm,
	anchor=north
	] at (\xO, 1.05) {
	};
	\node [anchor=north] at (\xO,0.95) {\scriptsize Input Polytope $\bc{P}^i$};

    \node (poly11) at (\xO+0.12, -0.1) {
	\begin{tikzpicture}[scale=0.25]

		\coordinate (tc) at (2, 0, 0);
		\coordinate (bc) at (-2, 0, 0);
		\coordinate (mn) at (0, -2, 0);
		\coordinate (ms) at (0.8, 1.2, 0);
		\coordinate (bs) at (-0.8, 1.2, 0);

		\draw[fill=blue!90, opacity=0.4] (tc) -- (ms) -- (bs) -- (bc) -- (mn) -- cycle;
		
		\coordinate (xp) at (2.8, 0, 0);
		\coordinate (xn) at (-2.5, 0, 0);
		\coordinate (yp) at (0, 2.2, 0);
		\coordinate (yn) at (0, -2.5, 0);
		
		\draw[->,opacity=0.5] (xn) -- (xp) ;
		\draw[->,opacity=0.5] (yn) -- (yp) ;
		\node [font=\scriptsize, align=center, scale=1.0] at (2.7,0.5,0) {$x_1$};
		\node [font=\scriptsize, align=center, scale=1.0] at (0.7,2.2,0) {$x_2$};
		
	\end{tikzpicture}
};
	\node [anchor=north, align=center] at (\xO,-0.74) {
		\scriptsize $x_1+x_2 \geq -2$\\
		\scriptsize $-x_1+x_2 \geq -2$\\
		\scriptsize $x_1-x_2 \geq -2$\\
		\scriptsize $\vdots$};

  \def\xc{\xO+1.65};
  \def\xcc{\xO+14.35};
  \def\xinputs{0.6};
  \def\xfirstarrow{1.2};
  \def\xmid{2.3};
  \def\xmidn{3.7};
  \def\xmide{5.1};
  \def\xsecondarrow{2.7};
  \def\xthirdarrow{4.6};
  \def\xmidl{6.5};
  \def\xmidll{7.9};

  \def\cO{-0.8}

  \path let 
	\p1 = (B.north),
	\p2 = (B.south),
	\n1 = {0.5*(\y1+\y2)}
	in
	node (C)
	[draw=black!60, fill=black!05, rectangle, rounded corners=4pt,
	minimum width=10.5cm, minimum height=4.2cm,
	anchor=west
	] at (\xc, \n1) {};

  \draw [->] (B) -- (C);

  \node (input) at (\xc+\xinputs+0.1, \cO) {
    \begin{tikzpicture}[scale=0.15]

		\coordinate (tc) at (2, 0, 0);
		\coordinate (bc) at (-2, 0, 0);
		\coordinate (mn) at (0, -2, 0);
		\coordinate (ms) at (0.8, 1.2, 0);
		\coordinate (bs) at (-0.8, 1.2, 0);

		\draw[fill=blue!90, opacity=0.4] (tc) -- (ms) -- (bs) -- (bc) -- (mn) -- cycle;
		
		\coordinate (xp) at (3.5, 0, 0);
		\coordinate (xn) at (-2.5, 0, 0);
		\coordinate (yp) at (0, 3.5, 0);
		\coordinate (yn) at (0, -2.5, 0);
		\coordinate (zp) at (0, 0, 5.0);
		\coordinate (zn) at (0, 0, -2.5);
		
		\draw[->,opacity=0.5] (xn) -- (xp) ;
		\draw[->,opacity=0.5] (yn) -- (yp) ;
		
		\node [font=\scriptsize, align=center, scale=0.8] at (3.5,0.8,0) {$x_1$};
		\node [font=\scriptsize, align=center, scale=0.8] at (1.0,3.5,0) {$x_2$};

    \end{tikzpicture}
  };

  \node (c1) at (\xc+\xmid, \cO+1.) {
    \begin{tikzpicture}[scale=0.15]

		\coordinate (tc) at (2, 0, 0);
		\coordinate (mn) at (0, -2, 0);
		\coordinate (ms) at (0, 1.2, 0);
		\coordinate (ts) at (0.8, 1.2, 0);

		\draw[fill=blue!90, opacity=0.4] (tc) -- (mn) -- (ms) -- (ts) -- cycle;
		
		\coordinate (xp) at (2.5, 0, 0);
		\coordinate (xn) at (-2.5, 0, 0);
		\coordinate (yp) at (0, 2.5, 0);
		\coordinate (yn) at (0, -2.5, 0);
		\coordinate (zp) at (0, 0, 2.5);
		\coordinate (zn) at (0, 0, -2.5);
		
		\draw[-,opacity=0.5] (xn) -- (xp) ;
		\draw[-,opacity=0.5] (yn) -- (yp) ;

    \end{tikzpicture}
  };

  \node (c2) at (\xc+\xmid, \cO-1.0) {
	\begin{tikzpicture}[scale=0.15]
		
		\coordinate (bc) at (-2, 0, 0);
		\coordinate (mn) at (0, -2, 0);
		\coordinate (ms) at (0, 1.2, 0);
		\coordinate (bs) at (-0.8, 1.2, 0);

		\draw[fill=blue!90, opacity=0.4] (bc) -- (mn) -- (ms) -- (bs) -- cycle;
		
		\coordinate (xp) at (2.5, 0, 0);
		\coordinate (xn) at (-2.5, 0, 0);
		\coordinate (yp) at (0, 2.5, 0);
		\coordinate (yn) at (0, -2.5, 0);
		\coordinate (zp) at (0, 0, 2.5);
		\coordinate (zn) at (0, 0, -2.5);
		
		\draw[-,opacity=0.5] (xn) -- (xp) ;
		\draw[-,opacity=0.5] (yn) -- (yp) ;
		
	\end{tikzpicture}
  };

  \draw [->] (input) -- (c1);
  \draw [->] (input) -- (c2);
  
  \node (c11) at (\xc+\xmidn, \cO+1.4) {
    \begin{tikzpicture}[scale=0.15]

  		\coordinate (O) at (0, 0, 0);
		\coordinate (tc) at (2, 0, 0);
		\coordinate (ms) at (0, 1.2, 0);
		\coordinate (ts) at (0.8, 1.2, 0);

		\draw[fill=blue!90, opacity=0.4] (tc) -- (O) -- (ms) -- (ts) -- cycle;
		
		\coordinate (xp) at (2.5, 0, 0);
		\coordinate (xn) at (-2.5, 0, 0);
		\coordinate (yp) at (0, 2.5, 0);
		\coordinate (yn) at (0, -2.5, 0);
		\coordinate (zp) at (0, 0, 2.5);
		\coordinate (zn) at (0, 0, -2.5);
		
		\draw[-,opacity=0.5] (xn) -- (xp) ;
		\draw[-,opacity=0.5] (yn) -- (yp) ;

    \end{tikzpicture}
  };

  \node (c12) at (\xc+\xmidn, \cO+0.5) {
	\begin{tikzpicture}[scale=0.15]
		
		\coordinate (O) at (0, 0, 0);
		\coordinate (tc) at (2, 0, 0);
		\coordinate (bc) at (-2, 0, 0);
		\coordinate (mn) at (0, -2, 0);
		\coordinate (me) at (0, 0, 2);
		\coordinate (ms) at (0, 2, 0);
		\coordinate (mw) at (0, 0, -2);

		\draw[fill=blue!90, opacity=0.4] (tc) -- (mn) -- (O) -- cycle;
		
		\coordinate (xp) at (2.5, 0, 0);
		\coordinate (xn) at (-2.5, 0, 0);
		\coordinate (yp) at (0, 2.5, 0);
		\coordinate (yn) at (0, -2.5, 0);
		\coordinate (zp) at (0, 0, 2.5);
		\coordinate (zn) at (0, 0, -2.5);
		
		\draw[-,opacity=0.5] (xn) -- (xp) ;
		\draw[-,opacity=0.5] (yn) -- (yp) ;
		
	\end{tikzpicture}
};

  \node (c21) at (\xc+\xmidn, \cO-0.5) {
	\begin{tikzpicture}[scale=0.15]
		
		\coordinate (O) at (0, 0, 0);
		\coordinate (bc) at (-2, 0, 0);
		\coordinate (ms) at (0, 1.2, 0);
		\coordinate (bs) at (-1.20, 1.2, 0);

		\draw[fill=blue!90, opacity=0.4] (bc) -- (O) -- (ms) -- (bs) -- cycle;
		
		\coordinate (xp) at (2.5, 0, 0);
		\coordinate (xn) at (-2.5, 0, 0);
		\coordinate (yp) at (0, 2.5, 0);
		\coordinate (yn) at (0, -2.5, 0);
		\coordinate (zp) at (0, 0, 2.5);
		\coordinate (zn) at (0, 0, -2.5);
		
		\draw[-,opacity=0.5] (xn) -- (xp) ;
		\draw[-,opacity=0.5] (yn) -- (yp) ;
		
	\end{tikzpicture}
};

\node (c22) at (\xc+\xmidn, \cO-1.4) {
	\begin{tikzpicture}[scale=0.15]
		
		\coordinate (O) at (0, 0, 0);
		\coordinate (tc) at (2, 0, 0);
		\coordinate (bc) at (-2, 0, 0);
		\coordinate (mn) at (0, -2, 0);
		\coordinate (me) at (0, 0, 2);
		\coordinate (ms) at (0, 2, 0);
		\coordinate (mw) at (0, 0, -2);

		\draw[fill=blue!90, opacity=0.4] (bc) -- (mn) -- (O) -- cycle;
		
		\coordinate (xp) at (2.5, 0, 0);
		\coordinate (xn) at (-2.5, 0, 0);
		\coordinate (yp) at (0, 2.5, 0);
		\coordinate (yn) at (0, -2.5, 0);
		\coordinate (zp) at (0, 0, 2.5);
		\coordinate (zn) at (0, 0, -2.5);
		
		\draw[-,opacity=0.5] (xn) -- (xp) ;
		\draw[-,opacity=0.5] (yn) -- (yp) ;
		
	\end{tikzpicture}
  };

  \draw [->] (c1) -- (c11);
  \draw [->] (c1) -- (c12);
  \draw [->] (c2) -- (c21);
  \draw [->] (c2) -- (c22);

    \node (c11e) at (\xc+\xmide, \cO+1.4) {
  	\begin{tikzpicture}[scale=0.15]
  		
  		\coordinate (Of) at (0, 0, 2);
  		\coordinate (tcf) at (2, 0, 2);
  		\coordinate (msf) at (0, 1.2, 2);
  		\coordinate (tsf) at (0.8, 1.2, 2);
  		\coordinate (Ob) at (0, 0, -2);
		\coordinate (tcb) at (2, 0, -2);
		\coordinate (msb) at (0, 1.2, -2);
		\coordinate (tsb) at (0.8, 1.2, -2);

  		\draw[fill=red!50!blue, opacity=0.3] (tcf) -- (tcb) -- (tsb) -- (tsf) -- cycle;
  		\draw[fill=red!50!blue, opacity=0.3] (tsf) -- (tsb) -- (msb) -- (msf) -- cycle;
  		\draw[fill=red!50!blue, opacity=0.3] (tcf) -- (tcb) -- (Ob) -- (Of) -- cycle;
  		\draw[fill=red!50!blue, opacity=0.3] (Of) -- (Ob) -- (msb) -- (msf) -- cycle;
  		
  		\coordinate (xp) at (2.5, 0, 0);
  		\coordinate (xn) at (-2.5, 0, 0);
  		\coordinate (yp) at (0, 2.5, 0);
  		\coordinate (yn) at (0, -2.5, 0);
  		\coordinate (zp) at (0, 0, 2.5);
  		\coordinate (zn) at (0, 0, -2.5);
  		
  		\draw[-,opacity=0.5] (xn) -- (xp) ;
  		\draw[-,opacity=0.5] (yn) -- (yp) ;
  		\draw[-,opacity=0.5] (zn) -- (zp) ;
  		
  	\end{tikzpicture}
  };
  
  \node (c12e) at (\xc+\xmide, \cO+0.5) {
  	\begin{tikzpicture}[scale=0.15]
  		
  		\coordinate (Of) at (0, 0, 2);
  		\coordinate (tcf) at (2, 0, 2);
  		\coordinate (mnf) at (0, -2, 2);
  		\coordinate (Ob) at (0, 0, -2);
		\coordinate (tcb) at (2, 0, -2);
		\coordinate (mnb) at (0, -2, -2);

  		\draw[fill=red!50!blue, opacity=0.3] (tcf) -- (tcb) -- (Ob) -- (Of) -- cycle;
  		\draw[fill=red!50!blue, opacity=0.3] (tcf) -- (tcb) -- (mnb) -- (mnf) -- cycle;
  		\draw[fill=red!50!blue, opacity=0.3] (mnf) -- (mnb) -- (Ob) -- (Of) -- cycle;
  		
  		\coordinate (xp) at (2.5, 0, 0);
  		\coordinate (xn) at (-2.5, 0, 0);
  		\coordinate (yp) at (0, 2.5, 0);
  		\coordinate (yn) at (0, -2.5, 0);
  		\coordinate (zp) at (0, 0, 2.5);
  		\coordinate (zn) at (0, 0, -2.5);
  		
  		\draw[-,opacity=0.5] (xn) -- (xp) ;
  		\draw[-,opacity=0.5] (yn) -- (yp) ;
  		\draw[-,opacity=0.5] (zn) -- (zp) ;
  		
  	\end{tikzpicture}
  };

  \node (c21e) at (\xc+\xmide, \cO-0.5) {
  	\begin{tikzpicture}[scale=0.15]
  		
  		\coordinate (Of) at (0, 0, 2);
  		\coordinate (bcf) at (-2, 0, 2);
  		\coordinate (msf) at (0, 1.2, 2);
  		\coordinate (bsf) at (-0.8, 1.2, 2);
  		\coordinate (Ob) at (0, 0, -2);
		\coordinate (bcb) at (-2, 0, -2);
		\coordinate (msb) at (0, 1.2, -2);
		\coordinate (bsb) at (-0.8, 1.2, -2);

  		\draw[fill=red!50!blue, opacity=0.3] (bcf) -- (Of) -- (Ob) --(bcb) -- cycle;
  		\draw[fill=red!50!blue, opacity=0.3] (bcf) -- (bsf) -- (bsb) -- (bcb) -- cycle;
  		\draw[fill=red!50!blue, opacity=0.3] (bsf) -- (msf) -- (msb) -- (bsb) -- cycle;
  		\draw[fill=red!50!blue, opacity=0.3] (msf) -- (Of) -- (Ob) -- (msb) -- cycle;
  		
  		\coordinate (xp) at (2.5, 0, 0);
  		\coordinate (xn) at (-2.5, 0, 0);
  		\coordinate (yp) at (0, 2.5, 0);
  		\coordinate (yn) at (0, -2.5, 0);
  		\coordinate (zp) at (0, 0, 2.5);
  		\coordinate (zn) at (0, 0, -2.5);
  		
  		\draw[-,opacity=0.5] (xn) -- (xp) ;
  		\draw[-,opacity=0.5] (yn) -- (yp) ;
  		\draw[-,opacity=0.5] (zn) -- (zp) ;
  		
  	\end{tikzpicture}
  };
  
  \node (c22e) at (\xc+\xmide, \cO-1.4) {
  	\begin{tikzpicture}[scale=0.15]
  		
  		\coordinate (Of) at (0, 0, 2);
  		\coordinate (bcf) at (-2, 0, 2);
  		\coordinate (mnf) at (0, -2, 2);
  		\coordinate (Ob) at (0, 0, -2);
		\coordinate (bcb) at (-2, 0, -2);
		\coordinate (mnb) at (0, -2, -2);

  		\draw[fill=red!50!blue, opacity=0.3] (bcf) -- (bcb) -- (Ob) -- (Of) -- cycle;
  		\draw[fill=red!50!blue, opacity=0.3] (bcf) -- (bcb) -- (mnb) -- (mnf)-- cycle;
  		\draw[fill=red!50!blue, opacity=0.3] (mnf) -- (mnb) -- (Ob) -- (Of) -- cycle;
  		
  		\coordinate (xp) at (2.5, 0, 0);
  		\coordinate (xn) at (-2.5, 0, 0);
  		\coordinate (yp) at (0, 2.5, 0);
  		\coordinate (yn) at (0, -2.5, 0);
  		\coordinate (zp) at (0, 0, 2.5);
  		\coordinate (zn) at (0, 0, -2.5);
  		
  		\draw[-,opacity=0.5] (xn) -- (xp) ;
  		\draw[-,opacity=0.5] (yn) -- (yp) ;
  		\draw[-,opacity=0.5] (zn) -- (zp) ;
  		
  	\end{tikzpicture}
  };
  
  \draw [->] (c11) -- (c11e);
  \draw [->] (c12) -- (c12e);
  \draw [->] (c21) -- (c21e);
  \draw [->] (c22) -- (c22e);

  \node (c11l) at (\xc+\xmidl, \cO+1.4) {
	\begin{tikzpicture}[scale=0.15]
		
		\coordinate (O) at (0, 0, 0);
		\coordinate (tc) at (2, 0, 0);
		\coordinate (mse) at (0, 1.2, 1.2);
		\coordinate (tse) at (0.8, 1.2, 1.2);
		
		\draw[fill=red!90, opacity=0.5] (tc) -- (O) -- (mse) --(tse) -- cycle;
		
		\coordinate (xp) at (2.5, 0, 0);
		\coordinate (xn) at (-2.5, 0, 0);
		\coordinate (yp) at (0, 2.5, 0);
		\coordinate (yn) at (0, -2.5, 0);
		\coordinate (zp) at (0, 0, 2.5);
		\coordinate (zn) at (0, 0, -2.5);
		
		\draw[-,opacity=0.5] (xn) -- (xp) ;
		\draw[-,opacity=0.5] (yn) -- (yp) ;
		\draw[-,opacity=0.5] (zn) -- (zp) ;
		
	\end{tikzpicture}
  };

  \node (c12l) at (\xc+\xmidl, \cO+0.5) {
	\begin{tikzpicture}[scale=0.15]
		
		\coordinate (O) at (0, 0, 0);
		\coordinate (tc) at (2, 0, 0);
		\coordinate (bc) at (-2, 0, 0);
		\coordinate (mn) at (0, -2, 0);
		\coordinate (me) at (0, 0, 2);
		\coordinate (ms) at (0, 2, 0);
		\coordinate (mw) at (0, 0, -2);

		\draw[fill=red!90, opacity=0.5] (tc) -- (mn) -- (O) -- cycle;
		
		\coordinate (xp) at (2.5, 0, 0);
		\coordinate (xn) at (-2.5, 0, 0);
		\coordinate (yp) at (0, 2.5, 0);
		\coordinate (yn) at (0, -2.5, 0);
		\coordinate (zp) at (0, 0, 2.5);
		\coordinate (zn) at (0, 0, -2.5);
		
		\draw[-,opacity=0.5] (xn) -- (xp) ;
		\draw[-,opacity=0.5] (yn) -- (yp) ;
		\draw[-,opacity=0.5] (zn) -- (zp) ;
		
	\end{tikzpicture}
  };

  \node (c21l) at (\xc+\xmidl, \cO-0.5) {
	\begin{tikzpicture}[scale=0.15]
		
		\coordinate (O) at (0, 0, 0);
		\coordinate (tc) at (2, 0, 0);
		\coordinate (bc) at (-2, 0, 0);
		\coordinate (mn) at (0, -2, 0);
		\coordinate (me) at (0, 0, 2);
		\coordinate (ms) at (0, 2, 0);
		\coordinate (mse) at (0, 1.2, 1.2);
		\coordinate (bse) at (-0.8, 1.2,1.2);
		\coordinate (mw) at (0, 0, -2);

		\draw[fill=red!90, opacity=0.5] (bc) -- (O) -- (mse) -- (bse) -- cycle;

		\coordinate (xp) at (2.5, 0, 0);
		\coordinate (xn) at (-2.5, 0, 0);
		\coordinate (yp) at (0, 2.5, 0);
		\coordinate (yn) at (0, -2.5, 0);
		\coordinate (zp) at (0, 0, 2.5);
		\coordinate (zn) at (0, 0, -2.5);
		
		\draw[-,opacity=0.5] (xn) -- (xp) ;
		\draw[-,opacity=0.5] (yn) -- (yp) ;
		\draw[-,opacity=0.5] (zn) -- (zp) ;
		
	\end{tikzpicture}
  };

  \node (c22l) at (\xc+\xmidl, \cO-1.4) {
	\begin{tikzpicture}[scale=0.15]
		
		\coordinate (O) at (0, 0, 0);
		\coordinate (tc) at (2, 0, 0);
		\coordinate (bc) at (-2, 0, 0);
		\coordinate (mn) at (0, -2, 0);
		\coordinate (me) at (0, 0, 2);
		\coordinate (ms) at (0, 2, 0);
		\coordinate (mw) at (0, 0, -2);

		\draw[fill=red!90, opacity=0.5] (bc) -- (mn) -- (O) -- cycle;
		
		\coordinate (xp) at (2.5, 0, 0);
		\coordinate (xn) at (-2.5, 0, 0);
		\coordinate (yp) at (0, 2.5, 0);
		\coordinate (yn) at (0, -2.5, 0);
		\coordinate (zp) at (0, 0, 2.5);
		\coordinate (zn) at (0, 0, -2.5);
		
		\draw[-,opacity=0.5] (xn) -- (xp) ;
		\draw[-,opacity=0.5] (yn) -- (yp) ;
		\draw[-,opacity=0.5] (zn) -- (zp) ;
		
	\end{tikzpicture}
  };

	\draw[->] (c11e) -- (c11l);
	\draw[->] (c12e) -- (c12l);
	\draw[->] (c21e) -- (c21l);
	\draw[->] (c22e) -- (c22l);

  \node (c1l) at (\xc+\xmidll, \cO+1.) {
	\begin{tikzpicture}[scale=0.15]
		
		\coordinate (O) at (0, 0, 0);
		\coordinate (tc) at (2, 0, 0);
		\coordinate (bc) at (-2, 0, 0);
		\coordinate (mn) at (0, -2, 0);
		\coordinate (me) at (0, 0, 2);
		\coordinate (ms) at (0, 2, 0);
		\coordinate (mse) at (0, 1.2, 1.2);
		\coordinate (tse) at (0.8, 1.2, 1.2);
		\coordinate (mw) at (0, 0, -2);

		\draw[fill=yellow!90, opacity=0.4] (tc) -- (mn) -- (O) -- cycle;
		\draw[fill=yellow!90, opacity=0.4] (tc) -- (tse) -- (mse) -- (O) -- cycle;
		\draw[fill=yellow!90, opacity=0.4] (tc) -- (tse) -- (mn) -- cycle;
		\draw[fill=yellow!90, opacity=0.4] (mse) -- (tse) -- (mn) -- cycle;
		\draw[fill=yellow!90, opacity=0.4] (O) -- (mse) -- (mn) -- cycle;
		
		\coordinate (xp) at (3.2, 0, 0);
		\coordinate (xn) at (-2.5, 0, 0);
		\coordinate (yp) at (0, 2.5, 0);
		\coordinate (yn) at (0, -2.5, 0);
		\coordinate (zp) at (0, 0, 4.0);
		\coordinate (zn) at (0, 0, -2.5);
		\draw[->,opacity=0.5] (xn) -- (xp) ;
		\draw[->,opacity=0.5] (yn) -- (yp) ;
		\draw[->,opacity=0.5] (zn) -- (zp) ;
		\node [font=\scriptsize, align=center, scale=0.8] at (3.6,0.7,0) {$x_1$};
		\node [font=\scriptsize, align=center, scale=0.8] at (1.2,3.2,0) {$x_2$};
		\node [font=\scriptsize, align=center, scale=0.8] at (0.3,0.3,7.1) {$y_2$};
		
	\end{tikzpicture}
};

\node (c2l) at (\xc+\xmidll, \cO-1.0) {
	\begin{tikzpicture}[scale=0.15]
		
		\coordinate (O) at (0, 0, 0);
		\coordinate (tc) at (2, 0, 0);
		\coordinate (bc) at (-2, 0, 0);
		\coordinate (mn) at (0, -2, 0);
		\coordinate (me) at (0, 0, 2);
		\coordinate (mse) at (0, 1.2, 1.2);
		\coordinate (mw) at (0, 0, -2);
		\coordinate (bse) at (-0.8, 1.2,1.2);
		
		\draw[fill=yellow!90, opacity=0.4] (bc) -- (mn) -- (O) -- cycle;
		\draw[fill=yellow!90, opacity=0.4] (bc) -- (bse) -- (mse) -- (O) -- cycle;
		\draw[fill=yellow!90, opacity=0.4] (bc) -- (bse) -- (mn) -- cycle;
		\draw[fill=yellow!90, opacity=0.4] (mse) -- (bse) -- (mn) -- cycle;
		\draw[fill=yellow!90, opacity=0.4] (O) -- (mse) -- (mn) -- cycle;
		
		\coordinate (xp) at (2.5, 0, 0);
		\coordinate (xn) at (-2.5, 0, 0);
		\coordinate (yp) at (0, 2.5, 0);
		\coordinate (yn) at (0, -2.5, 0);
		\coordinate (zp) at (0, 0, 5.0);
		\coordinate (zn) at (0, 0, -2.5);
		\draw[->,opacity=0.5] (xn) -- (xp) ;
		\draw[->,opacity=0.5] (yn) -- (yp) ;
		\draw[->,opacity=0.5] (zn) -- (zp) ;
		\node [font=\scriptsize, align=center, scale=0.8] at (3.2,0.5,0) {$x_1$};
		\node [font=\scriptsize, align=center, scale=0.8] at (0.9,2.7,0) {$x_2$};
		\node [font=\scriptsize, align=center, scale=0.8] at (0.3,0.3,7.1) {$y_2$};
		
	\end{tikzpicture}
};

    \node [font=\scriptsize, align=center, scale=0.8, anchor=south west] (box_EE) at ($(c1l.north east)+(-0.4, -0.15)$) {extend,\\bound};

  \node [font=\scriptsize, align=center, scale=0.8, anchor= west] (box_E) at ($(c1l)+(1.0, -0.0)$) {$\bc{K}'_1$};
  \node [font=\scriptsize, align=center, scale=0.8, anchor= west] (box_F) at ($(c2l)+(1.0, 0.0)$) {$\bc{K}'_2$};

    \node [font=\scriptsize, align=center, scale=0.8, anchor=south west] (box_EE) at ($(box_E.north east)+(-0.05, -0.5)$) {convex\\hull};

[draw=black, dashed, rectangle, rounded corners=2pt,
minimum width=3.75cm, minimum height=0.84cm,
anchor=south west
]

  \draw [->] ($(c11l.east)+(0.0,-0.1)$) -- ($(c1l.west)+(0.2,+0.1)$);
  \draw [->] ($(c12l.east)+(0.0,0.1)$) -- ($(c1l.west)+(0.2,-0.15)$);
  \draw [->] ($(c21l.east)+(0.0,-0.1)$) -- ($(c2l.west)+(0.2,+0.13)$);
  \draw [->] ($(c22l.east)+(0.0,0.1)$) -- ($(c2l.west)+(0.2,-0.1)$);
  
  \draw [->] ($(c1l.east)+(-0.1,0)$) -- ($(box_E.west)+(-0.,0)$);
  \draw [->] ($(c2l.east)+(-0.1,0)$) -- ($(box_F.west)+(-0.,0)$);
  
  \draw [->,] ($(box_E.east)+(0.05,-0.45)$) -- (\xc+9.8,\cO+0.2);
  \draw [->,] ($(box_F.east)+(0.05,0.45)$) -- (\xc+9.8,\cO-0.2);
  
  \node [align=center, scale=0.8, anchor=west] (t2) at (\xc+0.1, \cO+1.85) {Splitting};
  \node [align=center, scale=0.8, anchor=west] (t2) at (\xc+0.05, \cO-1.92) {\scriptsize Omitting superscript $\cdot^i$ here.};
  \node [align=center, scale=0.8, anchor=east] (t2) at (\xc+10.5, \cO+1.85) {Lifting};
  \node [font=\scriptsize, align=center, scale=0.8] (t2) at (\xc+\xmid+0.75, \cO+1.95) {split};
  \node [font=\scriptsize, align=center, scale=0.8] (t2) at (\xc+\xmidl+0.75, \cO+1.85) {convex \\ hull };
  \node [font=\scriptsize, align=center, scale=0.8] (t2) at (\xc+\xmidn+0.75, \cO+1.95) {extend};
  \node [font=\scriptsize, align=center, scale=0.8] (t2) at (\xc+\xmide+0.80, \cO+1.95) {bound};

   	\node [font=\scriptsize, align=center, scale=0.8] (t2) at (\xc+\xinputs-0.3, \cO+0.15) {$\bc{P}$};
   	
  	\node [font=\scriptsize, align=center, scale=0.8] (t2) at (\xc+\xmid-0.3, \cO+1.25) {$\bc{P}_{1}$};
	\node [font=\scriptsize, align=center, scale=0.8] (t2) at (\xc+\xmid-0.3, \cO-1.25) {$\bc{P}_{2}$};

	\node [font=\scriptsize, align=center, scale=0.8] (t2) at (\xc+\xmidn+0.4, \cO+1.6) {$\bc{P}_{1,1}$};
	\node [font=\scriptsize, align=center, scale=0.8] (t2) at (\xc+\xmidn+0.4, \cO+0.25) {$\bc{P}_{1,2}$};
	\node [font=\scriptsize, align=center, scale=0.8] (t2) at (\xc+\xmidn-0.35, \cO-0.20) {$\bc{P}_{2,1}$};
	\node [font=\scriptsize, align=center, scale=0.8] (t2) at (\xc+\xmidn-0.25, \cO-1.7) {$\bc{P}_{2,2}$};
	
	\node [font=\scriptsize, align=center, scale=0.8] (t2) at (\xc+\xmide+0.6, \cO+1.6) {$\bc{P}_{1,1}'$};
	\node [font=\scriptsize, align=center, scale=0.8] (t2) at (\xc+\xmide+0.4, \cO+0.25) {$\bc{P}_{1,2}'$};
	\node [font=\scriptsize, align=center, scale=0.8] (t2) at (\xc+\xmide-0.4, \cO-0.25) {$\bc{P}_{2,1}'$};
	\node [font=\scriptsize, align=center, scale=0.8] (t2) at (\xc+\xmide-0.5, \cO-1.65) {$\bc{P}_{2,2}'$};
	
	\node [font=\scriptsize, align=center, scale=0.8] (t2) at (\xc+\xmidl+0.3, \cO+1.6) {$\bc{K}_{1,1}'$};
	\node [font=\scriptsize, align=center, scale=0.8] (t2) at (\xc+\xmidl+0.35, \cO+0.20) {$\bc{K}_{1,2}'$};
	\node [font=\scriptsize, align=center, scale=0.8] (t2) at (\xc+\xmidl-0.4, \cO-0.25) {$\bc{K}_{2,1}'$};
	\node [font=\scriptsize, align=center, scale=0.8] (t2) at (\xc+\xmidl-0.4, \cO-1.65) {$\bc{K}_{2,2}'$};
	
	\node [font=\scriptsize, align=center, scale=0.8] (t2) at (\xc+\xmidll+0.3, \cO+0.75) {$\bc{K}_{1}$};
	\node [font=\scriptsize, align=center, scale=0.8] (t2) at (\xc+\xmidll-0.3, \cO-0.75) {$\bc{K}_{2}$};

	\node [font=\scriptsize, align=center, scale=0.8] (t2) at (\xc+\xmidll+2.1, \cO) {$\bc{K}$};

\path let 
  \p1 = (B.north),
  \p2 = (B.south),
  \n1 = {0.5*(\y1+\y2)} 
  in
  node (E)
  [draw=black!60, fill=black!05, rectangle, rounded corners=4pt,
    minimum width=3.4cm, minimum height=3.6cm,
    align=center
  ] at (\xcc, \n1) {};
  
  \node[anchor=north , scale=0.5] at (\xcc,0.4){
	  \begin{tikzpicture}[scale=0.5]
		
		\coordinate (O) at (0, 0, 0);
		\coordinate (tc) at (2, 0, 0);
		\coordinate (bc) at (-2, 0, 0);
		\coordinate (mn) at (0, -2, 0);
		\coordinate (me) at (0, 0, 2);
		\coordinate (mse) at (0.8, 1.2, 1.2);
		\coordinate (mw) at (0, 0, -2);
		\coordinate (bse) at (-0.8, 1.2,1.2);
		
		\coordinate (xp) at (2.5, 0, 0);
		\coordinate (xn) at (-2.5, 0, 0);
		\coordinate (yp) at (0, 1.5, 0);
		\coordinate (yn) at (0, -2.5, 0);
		\coordinate (zp) at (0, 0, 3.5);
		\coordinate (zn) at (0, 0, -2.5);
		\coordinate (zm) at (0, 0, {2*1.2/3.2});
		
		\draw[-,opacity=0.8] (zn) -- (zm) ;
		\draw[-,opacity=0.8] (xn) -- (xp) ;
		\draw[-,opacity=0.8] (yn) -- (yp) ;
		
		\draw[fill=green!90, opacity=0.3] (bc) -- (mn) -- (tc) -- cycle;
		\draw[fill=green!90, opacity=0.3] (bc) -- (bse) -- (mse) -- (tc) -- cycle;
		\draw[fill=green!90, opacity=0.3] (bc) -- (bse) -- (mn) -- cycle;
		\draw[fill=green!90, opacity=0.3] (mse) -- (bse) -- (mn) -- cycle;
		\draw[fill=green!90, opacity=0.3] (tc) -- (mse) -- (mn) -- cycle;

		\draw[opacity=0.5] (tc) -- (mn);  
		\draw[opacity=0.5] (tc) -- (mse);
		\draw[opacity=0.5] (bc) -- (mn);  
		\draw[opacity=0.5] (bc) -- (bse);
		\draw[opacity=0.5] (bse) -- (mse);  
		\draw[opacity=0.5] (mn) -- (mse);
		\draw[opacity=0.5] (mn) -- (bse);
		
		\draw[-,opacity=0.8] (zm) -- (zp) ;
		
		\node [font=\scriptsize, align=center, scale=1.2] at (2.5,0.3,0) {$x_1$};
		\node [font=\scriptsize, align=center, scale=1.2] at (0.5,1.5,0) {$x_2$};
		\node [font=\scriptsize, align=center, scale=1.2] at (-0.3,0.3,4.1) {$y_2$};

	\end{tikzpicture}
  };

	\node [anchor=north] at (\xcc,1) {\scriptsize Output Constraints $\bc{K}^i$};
	\node [anchor=north, align=center] at (\xcc,-0.7) {
		\scriptsize $x_1+x_2-2 y_1 - 2 y_2 \geq -2$\\
		\scriptsize $0.375 x_2 - y_2 \geq -0.75$\\
		\scriptsize $-x_1+y_1 \geq 0$\\
		\scriptsize $\vdots$};

   \draw [->] (C) -- (E);

\end{tikzpicture}

%% file: figures/precision_cmp.tex
\begin{tikzpicture}
	\tikzset{>=latex}
	
	\def\dx{1.7}
	\def\yo{-1.5}
	
	\node
	[draw=black!60, fill=black!05, rectangle, rounded corners=2pt,
	minimum width=6.8cm, minimum height=2.8cm,
	anchor=north
	] at (0, 0) {
	};
	\node [anchor=north,align=center,] at (-\dx,0.0) { 2-neuron};
	\node [anchor=north,align=center,] at (\dx,0.0) { single-neuron};

	\node (c2l) at (-\dx, \yo) {
		\begin{tikzpicture}[scale=0.5]
			
			\coordinate (O) at (0, 0, 0);
			\coordinate (tc) at (2, 0, 0);
			\coordinate (bc) at (-2, 0, 0);
			\coordinate (mn) at (0, -2, 0);
			\coordinate (me) at (0, 0, 2);
			\coordinate (mse) at (0.8, 1.2, 1.2);
			\coordinate (mw) at (0, 0, -2);
			\coordinate (bse) at (-0.8, 1.2,1.2);

			\coordinate (xp) at (2.5, 0, 0);
			\coordinate (xn) at (-2.5, 0, 0);
			\coordinate (yp) at (0, 1.5, 0);
			\coordinate (yn) at (0, -2.5, 0);
			\coordinate (zp) at (0, 0, 3.5);
			\coordinate (zn) at (0, 0, -2.5);
			\coordinate (zm) at (0, 0, {2*1.2/3.2});

			\draw[-,opacity=0.8] (zn) -- (zm) ;
			\draw[->,opacity=0.8] (xn) -- (xp) ;
			\draw[->,opacity=0.8] (yn) -- (yp) ;
			
			\draw[fill=yellow!90, opacity=0.3] (bc) -- (mn) -- (tc) -- cycle;
			\draw[fill=yellow!90, opacity=0.3] (bc) -- (bse) -- (mse) -- (tc) -- cycle;
			\draw[fill=yellow!90, opacity=0.3] (bc) -- (bse) -- (mn) -- cycle;
			\draw[fill=yellow!90, opacity=0.3] (mse) -- (bse) -- (mn) -- cycle;
			\draw[fill=yellow!90, opacity=0.3] (tc) -- (mse) -- (mn) -- cycle;
			
			\draw[opacity=0.5] (tc) -- (mn);  
			\draw[opacity=0.5] (tc) -- (mse);
			\draw[opacity=0.5] (bc) -- (mn);  
			\draw[opacity=0.5] (bc) -- (bse);
			\draw[opacity=0.5] (bse) -- (mse);  
			\draw[opacity=0.5] (mn) -- (mse);
			\draw[opacity=0.5] (mn) -- (bse);

		\draw[->,opacity=0.8] (zm) -- (zp) ;
			
		\node [font=\scriptsize, align=center, scale=1.] at (2.5,0.3,0) {$x_1$};
		\node [font=\scriptsize, align=center, scale=1.] at (0.5,1.5,0) {$x_2$};
		\node [font=\scriptsize, align=center, scale=1.] at (-0.3,0.3,4.1) {$y_2$};
			
		\end{tikzpicture}
	};

	\node (c2l) at (\dx, \yo) {
	\begin{tikzpicture}[scale=0.5]
		
		\coordinate (O) at (2, 0, 0);
		\coordinate (bc) at (-2, 0, 0);
		\coordinate (mn) at (-2, -2, 0);
		\coordinate (mse) at (2, 1.2, 1.2);
		\coordinate (mw) at (2, -2, 0);
		\coordinate (bse) at (-2, 1.2,1.2);
		
		\coordinate (xp) at (2.5, 0, 0);
		\coordinate (xn) at (-2.5, 0, 0);
		\coordinate (yp) at (0, 1.5, 0);
		\coordinate (yn) at (0, -2.5, 0);
		\coordinate (zp) at (0, 0, 3.5);
		\coordinate (zm) at (0, 0, {2*1.2/3.2});
		\coordinate (zn) at (0, 0, -2.5);
		
		\draw[->,opacity=0.8] (xn) -- (xp) ;
		\draw[->,opacity=0.8] (yn) -- (yp) ;
		\draw[-,opacity=0.8] (zn) -- (zm) ;

		\draw[fill=red!90, opacity=0.3] (bc) -- (bse) -- (mn) -- cycle;
		\draw[fill=red!90, opacity=0.3] (bc) -- (mn) -- (mw) -- (O) -- cycle;
		\draw[fill=red!90, opacity=0.3] (bc) -- (bse) -- (mse) -- (O) -- cycle;
		\draw[fill=red!90, opacity=0.3] (mse) -- (bse) -- (mn) -- (mw) -- cycle;
		\draw[fill=red!90, opacity=0.3] (O) -- (mse) -- (mw) -- cycle;

		\draw[->,opacity=0.8] (zm) -- (zp) ;
		
		\node [font=\scriptsize, align=center, scale=1.] at (2.5,0.3,0) {$x_1$};
\node [font=\scriptsize, align=center, scale=1.] at (0.5,1.5,0) {$x_2$};
\node [font=\scriptsize, align=center, scale=1.] at (-0.1,0.3,4.1) {$y_2$};

	\end{tikzpicture}
};
	
\end{tikzpicture}

%% file: figures/PDD_fig.tex
\begin{tikzpicture}
  \tikzset{>=latex}

  \def\xw{-0.5}
  \def\xlw{-2.1}
  \def\dx{3.1}
  \def\xle{2.5}
  \def\yn{0}
  \def\ys{-4}
  
  \node (Constraints) at (\xw + 0 * \dx, \yn) {
  	\begin{tikzpicture}[scale=0.6]
  		
  		\coordinate (ne) at (2, 0);
  		\coordinate (nw) at (0, 0);
  		\coordinate (se) at (3, -3);
  		\coordinate (sw) at (1, -2);

		\coordinate (n) at ($(ne)!0.5!(nw)$);
		\coordinate (nc) at ($(n)!0.6cm!90:(nw)$);
		\coordinate (w) at ($(nw)!0.5!(sw)$);
		\coordinate (wc) at ($(w)!0.6cm!90:(sw)$);
		\coordinate (s) at ($(sw)!0.5!(se)$);
		\coordinate (sc) at ($(s)!0.6cm!90:(se)$);
		\coordinate (e) at ($(se)!0.5!(ne)$);
		\coordinate (ec) at ($(e)!0.6cm!90:(ne)$);
		
		\draw[fill=green!30,opacity=0.4] (ne) -- (se) -- (sw) -- (nw) -- cycle;	
  		\draw[-,blue!90, shorten >= -0.5cm, shorten <= -0.5cm] (ne) -- (se);
  		\draw[-,blue!90, shorten >= -0.5cm, shorten <= -0.5cm] (nw) -- (sw);
  		\draw[-,blue!90, shorten >= -0.5cm, shorten <= -0.5cm] (nw) -- (ne);
  		\draw[-,blue!90, shorten >= -0.5cm, shorten <= -0.5cm] (sw) -- (se);

  		\draw[->,blue!90, thick] (n)--(nc);
  		\draw[->,blue!90, thick] (s)--(sc);
  		\draw[->,blue!90, thick] (e)--(ec);
  		\draw[->,blue!90, thick] (w)--(wc);
  	\end{tikzpicture}
  };
  \node [scale=1.05,align=center, anchor=north] at (\xw + 0 * \dx, {(0.5*\ys)+0.15}) {\scriptsize (a) Constraints - $\bm{A}$};

  \node (Vertex) at (\xw + 1 * \dx, \yn) {
	\begin{tikzpicture}[scale=0.6]
		
  		\coordinate (ne) at (2, 0);
		\coordinate (nw) at (0, 0);
		\coordinate (se) at (3, -3);
		\coordinate (sw) at (1, -2);
		
  		\draw[-,blue!90, shorten >= -0.5cm, shorten <= -0.5cm] (ne) -- (se);
		\draw[-,blue!90, shorten >= -0.5cm, shorten <= -0.5cm] (nw) -- (sw);
		\draw[-,blue!90, shorten >= -0.5cm, shorten <= -0.5cm] (nw) -- (ne);
		\draw[-,blue!90, shorten >= -0.5cm, shorten <= -0.5cm] (sw) -- (se);
		\node[circle, fill=black, minimum size=3pt,inner sep=0pt, outer sep=0pt] at (ne) {};
		\node[circle, fill=black, minimum size=3pt,inner sep=0pt, outer sep=0pt] at (nw) {};
		\node[circle, fill=black, minimum size=3pt,inner sep=0pt, outer sep=0pt] at (se) {};
		\node[circle, fill=black, minimum size=3pt,inner sep=0pt, outer sep=0pt] at (sw) {};
		
		\draw[fill=green!30,opacity=0.4] (ne) -- (se) -- (sw) -- (nw) -- cycle;	
	\end{tikzpicture}
	};
  \node [scale=1.05,align=center, anchor=north] at (\xw + 1 * \dx, {(0.5*\ys)+0.15}) {\scriptsize (b) Generators - $\bc{R}$};
  
    \node (Vertex) at (\xw + 2 * \dx, \yn) {
  	\begin{tikzpicture}[scale=0.6]
  		
  		\coordinate (ne) at (2, 0);
  		\coordinate (nw) at (0, 0);
  		\coordinate (se) at (3, -3);
  		\coordinate (sw) at (1, -2);
  		
  		\path (ne) -- (se) coordinate[pos=1] (pe);
  		\path (nw) -- (sw) coordinate[pos=1] (pw);
  		\path (ne) -- (pw) coordinate[pos=1.4] (pww);
  		
  		\draw[fill=green!30,opacity=0.4] (ne) -- (pe) -- (pw) -- (nw) -- cycle;
  		\draw[fill=red!30,opacity=0.4] (nw) -- (pw) -- (pe) -- (pww) -- cycle;

  		\draw[-,blue!90, shorten >= -0.5cm, shorten <= -0.5cm] (ne) -- (se);
  		\draw[-,blue!90, shorten >= -0.5cm, shorten <= -0.5cm] (nw) -- (sw);
  		\draw[-,blue!90, shorten >= -0.5cm, shorten <= -0.5cm] (nw) -- (ne);
  		\draw[-,blue!90, shorten >= -0.5cm, shorten <= -0.5cm] (sw) -- (se);
  		\node[circle, fill=black, minimum size=3pt,inner sep=0pt, outer sep=0pt] at (ne) {};
  		\node[circle, fill=black, minimum size=3pt,inner sep=0pt, outer sep=0pt] at (nw) {};
  		\node[circle, fill=black, minimum size=3pt,inner sep=0pt, outer sep=0pt] at (pe) {};
  		\node[circle, fill=black, minimum size=3pt,inner sep=0pt, outer sep=0pt] at (pww) {};		
  	\end{tikzpicture}
  };
  
  \node [scale=1.05,align=center, anchor=north] at (\xw + 2 * \dx+0.25, {(0.5*\ys)+0.15}) {\scriptsize (c) Unsound  \vrep};

  \node (Vertex) at (\xw + 3 * \dx, \yn) {
	\begin{tikzpicture}[scale=0.6]
		
  		\coordinate (ne) at (2, 0);
		\coordinate (nw) at (0, 0);
		\coordinate (se) at (3, -3);
		\coordinate (sw) at (1, -2);
		
		\path (ne) -- (se) coordinate[pos=0.2] (pe);
		\path (ne) -- (se) coordinate[pos=0.7] (ppe);
		\path (nw) -- (sw) coordinate[pos=1.0] (pw);
		
		\draw[fill=green!30,opacity=0.4] (pe) -- (ppe) -- (pw) -- (nw) -- cycle;	
		
		\draw[-,blue!90, shorten >= -0.5cm, shorten <= -0.5cm] (ne) -- (se);
		\draw[-,blue!90, shorten >= -0.5cm, shorten <= -0.5cm] (sw) -- (se);
		\draw[-,blue!90, shorten >= -0.5cm, shorten <= -0.5cm] (nw) -- (sw);
		\draw[-,blue!90, shorten >= -0.5cm, shorten <= -0.5cm] (nw) -- (ne);
		\node[circle, fill=black, minimum size=3pt,inner sep=0pt, outer sep=0pt] at (ppe) {};
		\node[circle, fill=black, minimum size=3pt,inner sep=0pt, outer sep=0pt] at (nw) {};
		\node[circle, fill=black, minimum size=3pt,inner sep=0pt, outer sep=0pt] at (pe) {};
		\node[circle, fill=black, minimum size=3pt,inner sep=0pt, outer sep=0pt] at (pw) {};

	\end{tikzpicture}
	};

  \node [scale=1.05,align=center, anchor=north] at (\xw + 3 * \dx +0.45, {(0.5*\ys)+0.15}) {\scriptsize (d) Sound \vrep};

   \node (Vertex) at (\xw + 4 * \dx, \yn) {
 	\begin{tikzpicture}[scale=0.6]
 		
 		\coordinate (ne) at (2, 0);
 		\coordinate (nw) at (0, 0);
 		\coordinate (se) at (3, -3);
 		\coordinate (sw) at (1, -2);
 		
 		\path (ne) -- (se) coordinate[pos=0.2] (pe);
 		\path (ne) -- (se) coordinate[pos=0.7] (ppe);
 		\path (nw) -- (sw) coordinate[pos=1.0] (pw);
 		
		\draw[fill=blue!90,opacity=0.15] (ppe) -- (pw) -- (pe) -- cycle;	
 		\draw[fill=green!30,opacity=0.4] (pe) -- (pw) -- (nw) -- cycle;
 		
 		\draw[-,blue!90, shorten >= -0.5cm, shorten <= -0.5cm] (ne) -- (se);
 		\draw[-,blue!90, shorten >= -0.5cm, shorten <= -0.5cm] (sw) -- (se);
 		\draw[-,blue!90, shorten >= -0.5cm, shorten <= -0.5cm] (nw) -- (sw);
 		\draw[-,blue!90, shorten >= -0.5cm, shorten <= -0.5cm] (nw) -- (ne);
 		\node[circle, fill=black, minimum size=3pt,inner sep=0pt, outer sep=0pt] at (nw) {};
 		\node[circle, fill=black, minimum size=3pt,inner sep=0pt, outer sep=0pt] at (pe) {};
 		\node[circle, fill=black, minimum size=3pt,inner sep=0pt, outer sep=0pt] at (pw) {};

 	\end{tikzpicture}
 };
 
 	\node [scale=1.05,align=center, anchor=north] at (\xw + 4 * \dx+0.55, {(0.5*\ys)+0.15}) {\scriptsize (e) A-irredundant a)};
   \node (Vertex) at (\xw + 5 * \dx, \yn) {
 	\begin{tikzpicture}[scale=0.6]
 		
 		\coordinate (ne) at (2, 0);
 		\coordinate (nw) at (0, 0);
 		\coordinate (se) at (3, -3);
 		\coordinate (sw) at (1, -2);
 		
 		\path (ne) -- (se) coordinate[pos=0.2] (pe);
 		\path (ne) -- (se) coordinate[pos=0.7] (ppe);
 		\path (nw) -- (sw) coordinate[pos=1.0] (pw);
 		
		\draw[fill=blue!90,opacity=0.15] (ppe) -- (nw) -- (pe) -- cycle;	
 		\draw[fill=green!30,opacity=0.4] (ppe) -- (pw) -- (nw) -- cycle;	
 		
 		\draw[-,blue!90, shorten >= -0.5cm, shorten <= -0.5cm] (ne) -- (se);
 		\draw[-,blue!90, shorten >= -0.5cm, shorten <= -0.5cm] (sw) -- (se);
 		\draw[-,blue!90, shorten >= -0.5cm, shorten <= -0.5cm] (nw) -- (sw);
 		\draw[-,blue!90, shorten >= -0.5cm, shorten <= -0.5cm] (nw) -- (ne);
 		\node[circle, fill=black, minimum size=3pt,inner sep=0pt, outer sep=0pt] at (ppe) {};
 		\node[circle, fill=black, minimum size=3pt,inner sep=0pt, outer sep=0pt] at (nw) {};
 		\node[circle, fill=black, minimum size=3pt,inner sep=0pt, outer sep=0pt] at (pw) {};

 	\end{tikzpicture}
 };
 
 \node [scale=1.05,align=center, anchor=north] at (\xw + 5 * \dx+0.35, {(0.5*\ys)+0.15}) {\scriptsize (f) A-irredundant b)};
\end{tikzpicture}

%% file: figures/overview_PDDM_figure.tex
\begin{tikzpicture}
  \tikzset{>=latex}

	\def\xO{0.0}

	\def\xoo{0}
	\def\yoo{0}
	\def\dyu{-1.7}
	\def\dyl{-3.9}
	\def\xio{1.0}
	\def\xo{2.8}
	
	\node (convexHH) [draw=black!60, rectangle,
		minimum width=21.4cm, minimum height=5.1cm,
		align=center, scale=1.0, rounded corners=5pt,
		anchor=north west,
		fill=black!05
		] at (0,0) {};

\node (hi1) at (\xoo+\xio, \yoo + \dyu) {
	\begin{tikzpicture}[scale=0.7]
		
		\coordinate (n) at (0, 1);
		\coordinate (e) at (1, 0);
		\coordinate (s) at (0, -1);
		\coordinate (w) at (-1, 0);
		
		\draw[fill=green!90, opacity=0.3] (n) -- (e) -- (s) -- (w) -- cycle;
		
		\coordinate (xp) at (1.2, 0, 0);
		\coordinate (xn) at (-1.2, 0, 0);
		\coordinate (yp) at (0, 1.2, 0);
		\coordinate (yn) at (0, -1.2, 0);
	
		\draw[-,opacity=0.0] (xn) -- (xp) ;
		\draw[-,opacity=0.0] (yn) -- (yp) ;
		\node[circle, black, fill, minimum size=3pt,inner sep=0pt, outer sep=0pt] at (n) {};
		\node[circle, black, fill, minimum size=3pt,inner sep=0pt, outer sep=0pt] at (e) {};
		\node[circle, black, fill, minimum size=3pt,inner sep=0pt, outer sep=0pt] at (w) {};
		\node[circle, black, fill, minimum size=3pt,inner sep=0pt, outer sep=0pt] at (s) {};		
		
	\end{tikzpicture}
};

\node (hi2) at (\xoo+\xio, \yoo + \dyl) {
	\begin{tikzpicture}[scale=0.7]
		
		\coordinate (n) at (-0.3, 1.2);
		\coordinate (nw) at (-1.0, 0.5);
		\coordinate (e) at (0.5, 0.0);
		\coordinate (sw) at (-1.0, -0.5);
		\coordinate (s) at (-0.3, -1.2);
		\coordinate (w) at (-1.1, 0);
		
		\draw[fill=blue!90, opacity=0.3] (nw) -- (n) -- (e) -- (s) -- (sw) -- (w) -- cycle;
		
		\coordinate (xp) at (1.2, 0, 0);
		\coordinate (xn) at (-1.2, 0, 0);
		\coordinate (yp) at (0, 1.2, 0);
		\coordinate (yn) at (0, -1.2, 0);
		
		\draw[-,opacity=0.0] (xn) -- (xp) ;
		\draw[-,opacity=0.0] (yn) -- (yp) ;
		\node[circle, black, fill, minimum size=3pt,inner sep=0pt, outer sep=0pt] at (n) {};
		\node[circle, black, fill, minimum size=3pt,inner sep=0pt, outer sep=0pt] at (e) {};
		\node[circle, black, fill, minimum size=3pt,inner sep=0pt, outer sep=0pt] at (w) {};
		\node[circle, black, fill, minimum size=3pt,inner sep=0pt, outer sep=0pt] at (s) {};		
		\node[circle, black, fill, minimum size=3pt,inner sep=0pt, outer sep=0pt] at (sw) {};		
		\node[circle, black, fill, minimum size=3pt,inner sep=0pt, outer sep=0pt] at (nw) {};		
	
	\end{tikzpicture}
};

\node[align=center] at (\xoo+\xio,\yoo -0.25) {primal};

\node (hi1d) at (\xoo+\xio+\xo, \yoo + \dyu) {
	\begin{tikzpicture}[scale=0.7]
		
		\coordinate (ne) at (1, 1);
		\coordinate (nw) at (-1, 1);
		\coordinate (se) at (1, -1);
		\coordinate (sw) at (-1, -1);
		
		\draw[fill=green!90, opacity=0.3] (ne) -- (se) -- (sw) -- (nw) -- cycle;
		
		\coordinate (xp) at (2.2, 0, 0);
		\coordinate (xn) at (-1.2, 0, 0);
		\coordinate (yp) at (0, 1.2, 0);
		\coordinate (yn) at (0, -1.2, 0);
		
		\draw[-,opacity=0.0] (xn) -- (xp) ;
		\draw[-,opacity=0.0] (yn) -- (yp) ;
		\node[circle, black, fill, minimum size=3pt,inner sep=0pt, outer sep=0pt] at (nw) {};
		\node[circle, black, fill, minimum size=3pt,inner sep=0pt, outer sep=0pt] at (ne) {};
		\node[circle, black, fill, minimum size=3pt,inner sep=0pt, outer sep=0pt] at (sw) {};
		\node[circle, black, fill, minimum size=3pt,inner sep=0pt, outer sep=0pt] at (se) {};		
		
	\end{tikzpicture}
};

\node (hi2d) at (\xoo+\xio+\xo, \yoo + \dyl) {
	\begin{tikzpicture}[scale=0.7]
		
		\coordinate (ne) at (2.0, 1.6/1.2);
		\coordinate (se) at (2, -1.6/1.2);
		\coordinate (n) at (-1/1.5, 1/1.5);
		\coordinate (s) at (-1/1.5, -1/1.5);
		\coordinate (nw) at (-1/1.1, 0.2/1.1);
		\coordinate (sw) at (-1/1.1, -0.2/1.1);
		
		\draw[fill=blue!90, opacity=0.3] (ne) -- (se) -- (s) -- (sw) -- (nw) -- (n) -- cycle;
		
		\coordinate (xp) at (2.2, 0, 0);
		\coordinate (xn) at (-1.2, 0, 0);
		\coordinate (yp) at (0, 1.2, 0);
		\coordinate (yn) at (0, -1.2, 0);
		
		\draw[-,opacity=0.0] (xn) -- (xp) ;
		\draw[-,opacity=0.0] (yn) -- (yp) ;
		\node[circle, black, fill, minimum size=3pt,inner sep=0pt, outer sep=0pt] at (n) {};
		\node[circle, black, fill, minimum size=3pt,inner sep=0pt, outer sep=0pt] at (ne) {};
		\node[circle, black, fill, minimum size=3pt,inner sep=0pt, outer sep=0pt] at (nw) {};
		\node[circle, black, fill, minimum size=3pt,inner sep=0pt, outer sep=0pt] at (s) {};		
		\node[circle, black, fill, minimum size=3pt,inner sep=0pt, outer sep=0pt] at (sw) {};		
		\node[circle, black, fill, minimum size=3pt,inner sep=0pt, outer sep=0pt] at (se) {};		
		
	\end{tikzpicture}
};

\node[align=center] at (\xoo+\xio+\xo,\yoo -0.20) {dual};

\node (h1di) at (\xoo+\xio+2.3*\xo, \yoo + \dyu) {
	\begin{tikzpicture}[scale=0.7]
		
		\coordinate (ne) at (1, 1);
		\coordinate (nw) at (-1, 1);
		\coordinate (se) at (1, -1);
		\coordinate (sw) at (-1, -1);
		
		\coordinate (ne1) at (2.0, 1.6/1.2);
		\coordinate (se1) at (2, -1.6/1.2);
		\coordinate (n1) at (-1/1.5, 1/1.5);
		\coordinate (s1) at (-1/1.5, -1/1.5);
		\coordinate (nw1) at (-1/1.1, 0.2/1.1);
		\coordinate (sw1) at (-1/1.1, -0.2/1.1);
		
		\draw[fill=green!90, opacity=0.3, name path=f1] (ne) -- (se) -- (sw) -- (nw) -- cycle;
		
		\draw[black!90, opacity=0.7, shorten <= -0.2cm, shorten >= -0.2cm, name path=l1] (se1) -- (ne1);
		\draw[black!90, opacity=0.7, shorten <= -0.15cm, name path=l2] (se1) -- (s1)  -- ++(165:.2cm);
		\draw[black!90, opacity=0.7, name path=l3] (sw1) -- (nw1) -- ++(90:0.2cm);
		\draw[black!90, opacity=0.7, name path=l7] (sw1) -- ++(270:0.2cm);
		\draw[black!90, opacity=0.7, shorten <= -0.2cm, name path=l4] (ne1)  -- (n1) -- ++(195:.2cm);
		\draw[black!90, opacity=0.7, shorten <= -0.3cm, name path=l5] (nw1)  -- (n1)   -- ++(63:.2cm);
		\draw[black!90, opacity=0.7, shorten <= -0.3cm, name path=l6] (sw1)  -- (s1)  -- ++(297:.2cm);

		\draw[->, black!90] ($(ne1)!0.5!(n1)$) -- ++(-75:0.5cm);
		\draw[->, black!90] ($(nw1)!0.5!(n1)$) -- ++(-30:0.5cm);
		\draw[->, black!90] ($(sw1)!0.5!(s1)$) -- ++(30:0.5cm);
		\draw[->, black!90] ($(sw1)!0.5!(nw1)$) -- ++(0:0.5cm);
		\draw[->, black!90] ($(ne1)!0.5!(se1)$) -- ++(180:0.5cm);
		\draw[->, black!90] ($(se1)!0.5!(s1)$) -- ++(75:0.5cm);

		\draw[-,opacity=0.0] (xn) -- (xp) ;
		\draw[-,opacity=0.0] (yn) -- (yp) ;
		
		\coordinate (xp) at (2.2, 0, 0);
		\coordinate (xn) at (-1.2, 0, 0);
		\coordinate (yp) at (0, 1.2, 0);
		\coordinate (yn) at (0, -1.2, 0);

		\node[circle, red, fill, minimum size=3pt,inner sep=0pt, outer sep=0pt] at (nw) {};
		\node[circle, blue, fill, minimum size=3pt,inner sep=0pt, outer sep=0pt] at (ne) {};
		\node[circle, red, fill, minimum size=3pt,inner sep=0pt, outer sep=0pt] at (sw) {};
		\node[circle, blue, fill, minimum size=3pt,inner sep=0pt, outer sep=0pt] at (se) {};		
		
	\end{tikzpicture}
};

\node (h2di) at (\xoo+\xio+2.3*\xo, \yoo + \dyl) {
	\begin{tikzpicture}[scale=0.7]
		
		\coordinate (ne) at (2.0, 1.6/1.2);
		\coordinate (se) at (2, -1.6/1.2);
		\coordinate (n) at (-1/1.5, 1/1.5);
		\coordinate (s) at (-1/1.5, -1/1.5);
		\coordinate (nw) at (-1/1.1, 0.2/1.1);
		\coordinate (sw) at (-1/1.1, -0.2/1.1);
		
		\coordinate (ne1) at (1, 1.2);
		\coordinate (nw1) at (-1, 1);
		\coordinate (se1) at (1, -1.2);
		\coordinate (sw1) at (-1, -1);
		
		\draw[fill=blue!90, opacity=0.3, name path=f1] (ne) -- (se) -- (s) -- (sw) -- (nw) -- (n) -- cycle;
		
		\draw[black!90, opacity=0.7, shorten <= -0.02cm, shorten >= -0.02cm, name path=l1] (ne1) -- (se1);
		\draw[black!90, opacity=0.7, shorten <= -0.1cm, name path=l2] (sw1) -- ++(0:2.2cm);
		\draw[black!90, opacity=0.7, shorten <= -0.15cm, shorten >= -0.15cm, name path=l3] (sw1) -- (nw1);
		\draw[black!90, opacity=0.7, shorten <= -0.1cm, name path=l4] (nw1) -- ++(0:2.2cm);

		\draw[->, black!90] (nw1)++(0:1.0cm) -- ++(-90:0.5cm);
		\draw[->, black!90] (sw1)++(0:1.0cm) -- ++(90:0.5cm);
		\draw[->, black!90] ($(ne1)!0.5!(se1)$) -- ++(180:0.5cm);
		\draw[->, black!90] ($(sw1)!0.5!(nw1)$) -- ++(0:0.5cm);

		\coordinate (xp) at (2.2, 0, 0);
		\coordinate (xn) at (-1.2, 0, 0);
		\coordinate (yp) at (0, 1.2, 0);
		\coordinate (yn) at (0, -1.2, 0);
		
		\draw[-,opacity=0.0] (xn) -- (xp) ;
		\draw[-,opacity=0.0] (yn) -- (yp) ;
		\node[circle, blue, fill, minimum size=3pt,inner sep=0pt, outer sep=0pt] at (n) {};
		\node[circle, red, fill, minimum size=3pt,inner sep=0pt, outer sep=0pt] at (ne) {};
		\node[circle, blue, fill, minimum size=3pt,inner sep=0pt, outer sep=0pt] at (nw) {};
		\node[circle, blue, fill, minimum size=3pt,inner sep=0pt, outer sep=0pt] at (s) {};		
		\node[circle, blue, fill, minimum size=3pt,inner sep=0pt, outer sep=0pt] at (sw) {};		
		\node[circle, red, fill, minimum size=3pt,inner sep=0pt, outer sep=0pt] at (se) {};		
		
	\end{tikzpicture}
};

\node[align=center] at (\xoo+\xio+2.3*\xo,\yoo -0.25) {adding constraints};

\node (h1dii) at (\xoo+\xio+3.6*\xo, \yoo + \dyu) {
	\begin{tikzpicture}[scale=0.7]
		
		\coordinate (ne) at (1, 1);
		\coordinate (nw) at (-1, 1);
		\coordinate (se) at (1, -1);
		\coordinate (sw) at (-1, -1);
		
		\coordinate (ne1) at (2.0, 1.6/1.2);
		\coordinate (se1) at (2, -1.6/1.2);
		\coordinate (n1) at (-1/1.5, 1/1.5);
		\coordinate (s1) at (-1/1.5, -1/1.5);
		\coordinate (nw1) at (-1/1.1, 0.2/1.1);
		\coordinate (sw1) at (-1/1.1, -0.2/1.1);
		
		\draw[opacity=0.4, name path=f1] (ne) -- (se) -- (sw) -- (nw) -- cycle;

		\draw[black!90, opacity=0.7, shorten <= -0.2cm, shorten >= -0.2cm, name path=l1] (se1) -- (ne1);
		\draw[black!90, opacity=0.7, shorten <= -0.15cm, name path=l2] (se1) -- (s1)  -- ++(165:.2cm);
		\draw[black!90, opacity=0.7, name path=l3] (sw1) -- (nw1) -- ++(90:0.2cm);
		\draw[black!90, opacity=0.7, name path=l7] (sw1) -- ++(270:0.2cm);
		\draw[black!90, opacity=0.7, shorten <= -0.2cm, name path=l4] (ne1)  -- (n1) -- ++(195:.2cm);
		\draw[black!90, opacity=0.7, shorten <= -0.3cm, name path=l5] (nw1)  -- (n1)   -- ++(63:.2cm);
		\draw[black!90, opacity=0.7, shorten <= -0.3cm, name path=l6] (sw1)  -- (s1)  -- ++(297:.2cm);
		
		\draw[black!90, opacity=0.7, shorten >= -0.3cm, name path=l8, dashed] (se)  -- (nw);
		\draw[black!90, opacity=0.7, shorten >= -0.3cm, name path=l9, dashed] (ne)  -- (sw);
		
		\draw[-,opacity=0.0] (xn) -- (xp) ;
		\draw[-,opacity=0.0] (yn) -- (yp) ;
		
		\coordinate (xp) at (2.2, 0, 0);
		\coordinate (xn) at (-1.2, 0, 0);
		\coordinate (yp) at (0, 1.2, 0);
		\coordinate (yn) at (0, -1.2, 0);

		\fill[name intersections={of=l2 and f1,total=\t, by=i1}]
		(intersection-1) [black] circle(2pt) node {};
		\fill[name intersections={of=l4 and f1,total=\t, by=i2}]
		(intersection-1) [black] circle(2pt) node {};

		\draw[fill=green!90, opacity=0.3] (ne) -- (se) -- (i1) -- (s1) --(n1) -- (i2) -- cycle;

		\node[circle, black, fill, minimum size=3pt,inner sep=0pt, outer sep=0pt] at (n1) {};
		\node[circle, black, fill, minimum size=3pt,inner sep=0pt, outer sep=0pt] at (s1) {};

		\node[circle, red, fill, minimum size=3pt,inner sep=0pt, outer sep=0pt] at (nw) {};
		\node[circle, blue, fill, minimum size=3pt,inner sep=0pt, outer sep=0pt] at (ne) {};
		\node[circle, red, fill, minimum size=3pt,inner sep=0pt, outer sep=0pt] at (sw) {};
		\node[circle, blue, fill, minimum size=3pt,inner sep=0pt, outer sep=0pt] at (se) {};		
		
	\end{tikzpicture}
};

\node (h2dii) at (\xoo+\xio+3.6*\xo, \yoo + \dyl) {
	\begin{tikzpicture}[scale=0.7]
		
		\coordinate (ne) at (2.0, 1.6/1.2);
		\coordinate (se) at (2, -1.6/1.2);
		\coordinate (n) at (-1/1.5, 1/1.5);
		\coordinate (s) at (-1/1.5, -1/1.5);
		\coordinate (nw) at (-1/1.1, 0.2/1.1);
		\coordinate (sw) at (-1/1.1, -0.2/1.1);
		
		\coordinate (ne1) at (1, 1.2);
		\coordinate (nw1) at (-1, 1);
		\coordinate (se1) at (1, -1.2);
		\coordinate (sw1) at (-1, -1);
		
		\draw[opacity=0.4, name path=f1] (ne) -- (se) -- (s) -- (sw) -- (nw) -- (n) -- cycle;

		\draw[black!90, opacity=0.7, shorten <= -0.02cm, shorten >= -0.02cm, name path=l1] (ne1) -- (se1);
		\draw[black!90, opacity=0.7, shorten <= -0.2cm, name path=l2] (sw1) -- ++(0:2.2cm);
		\draw[black!90, opacity=0.7, shorten <= -0.15cm, shorten >= -0.15cm, name path=l3] (sw1) -- (nw1);
		\draw[black!90, opacity=0.7, shorten <= -0.2cm, name path=l4] (nw1) -- ++(0:2.2cm);
		
		\draw[black!90, opacity=0.7, shorten >= -0.2cm, name path=l5, dashed] (nw) -- (ne);
		\draw[black!90, opacity=0.7, shorten >= -0.2cm, name path=l6, dashed] (sw) -- (ne);
		\draw[black!90, opacity=0.7, shorten >= -0.2cm, name path=l7, dashed] (s) -- (ne);

		\draw[black!90, opacity=0.7, shorten >= -0.2cm, name path=l8, dashed] (sw) -- (se);
		\draw[black!90, opacity=0.7, shorten >= -0.2cm, name path=l9, dashed] (nw) -- (se);
		\draw[black!90, opacity=0.7, shorten >= -0.2cm, name path=l10, dashed] (n) -- (se);

		\fill[name intersections={of=l5 and l1,total=\t, by=i1}]
			(intersection-1) [black] circle(2pt) node {};
		\fill[name intersections={of=l6 and l1,total=\t}]
			(intersection-1) [black] circle(2pt) node {};
		\fill[name intersections={of=l7 and l1,total=\t}]
			(intersection-1) [black] circle(2pt) node {};
			
		\fill[name intersections={of=l8 and l1,total=\t, by=i2}]
			(intersection-1) [black] circle(2pt) node {};
		\fill[name intersections={of=l9 and l1,total=\t}]
			(intersection-1) [black] circle(2pt) node {};
		\fill[name intersections={of=l10 and l1,total=\t}]
			(intersection-1) [black] circle(2pt) node {};
	
		\fill[name intersections={of=l2 and f1,total=\t,by=i3}]
			(intersection-1) [black] circle(2pt) node {};
		\fill[name intersections={of=l4 and f1,total=\t, by=i4}]
		(intersection-1) [black] circle(2pt) node {};

		\draw[fill=blue!90, opacity=0.3] (n) -- (nw) -- (sw) -- (s)  -- (i3)  -- (i2) -- (i1) -- (i4) -- cycle;
		
		\coordinate (xp) at (2.2, 0, 0);
		\coordinate (xn) at (-1.2, 0, 0);
		\coordinate (yp) at (0, 1.2, 0);
		\coordinate (yn) at (0, -1.2, 0);
		
		\draw[-,opacity=0.0] (xn) -- (xp) ;
		\draw[-,opacity=0.0] (yn) -- (yp) ;
		\node[circle, blue, fill, minimum size=3pt,inner sep=0pt, outer sep=0pt] at (n) {};
		\node[circle, red, fill, minimum size=3pt,inner sep=0pt, outer sep=0pt] at (ne) {};
		\node[circle, blue, fill, minimum size=3pt,inner sep=0pt, outer sep=0pt] at (nw) {};
		\node[circle, blue, fill, minimum size=3pt,inner sep=0pt, outer sep=0pt] at (s) {};		
		\node[circle, blue, fill, minimum size=3pt,inner sep=0pt, outer sep=0pt] at (sw) {};		
		\node[circle, red, fill, minimum size=3pt,inner sep=0pt, outer sep=0pt] at (se) {};		
		
	\end{tikzpicture}
};

\node[align=center] at (\xoo+\xio+3.6*\xo,\yoo -0.25) {ray shooting};

\node (hdii) at (\xoo+\xio+4.9*\xo, \yoo + 0.5*\dyl + 0.5*\dyu) {
	\begin{tikzpicture}[scale=0.7]
		
		\coordinate (ne) at (2.0, 1.6/1.2);
		\coordinate (se) at (2, -1.6/1.2);
		\coordinate (n) at (-1/1.5, 1/1.5);
		\coordinate (s) at (-1/1.5, -1/1.5);
		\coordinate (nw) at (-1/1.1, 0.2/1.1);
		\coordinate (sw) at (-1/1.1, -0.2/1.1);
		
		\coordinate (ne1) at (1, 1);
		\coordinate (nw1) at (-1, 1);
		\coordinate (se1) at (1, -1);
		\coordinate (sw1) at (-1, -1);
		
		\draw[opacity=0.35, name path=f1] (ne) -- (se) -- (s) -- (sw) -- (nw) -- (n) -- cycle;
		\draw[opacity=0.35, name path=f2] (ne1) -- (se1) -- (sw1) -- (nw1) -- cycle;
		\path[black!90, opacity=0.7, shorten <= -0.1cm, shorten >= -0.1cm, name path=l1] (ne1) -- (se1);
		\path[black!90, opacity=0.7, shorten <= -0.2cm, name path=l2] (sw1) -- ++(0:2.2cm);
		\path[black!90, opacity=0.7, shorten <= -0.2cm, shorten >= -0.2cm, name path=l3] (sw1) -- (nw1);
		\path[black!90, opacity=0.7, shorten <= -0.2cm, name path=l4] (nw1) -- ++(0:2.2cm);
		
		\path[black!90, opacity=0.7, shorten >= -0.2cm, name path=l5, dashed] (nw) -- (ne);
		\path[black!90, opacity=0.7, shorten >= -0.2cm, name path=l6, dashed] (sw) -- (ne);
		\path[black!90, opacity=0.7, shorten >= -0.2cm, name path=l7, dashed] (s) -- (ne);
		
		\path[black!90, opacity=0.7, shorten >= -0.2cm, name path=l8, dashed] (sw) -- (se);
		\path[black!90, opacity=0.7, shorten >= -0.2cm, name path=l9, dashed] (nw) -- (se);
		\path[black!90, opacity=0.7, shorten >= -0.2cm, name path=l10, dashed] (n) -- (se);

		\fill[name intersections={of=l5 and l1,total=\t}]
		(intersection-1) [black] circle(2pt) node {};
		\fill[name intersections={of=l6 and l1,total=\t}]
		(intersection-1) [black] circle(2pt) node {};
		\fill[name intersections={of=l7 and l1,total=\t}]
		(intersection-1) [black] circle(2pt) node {};
		
		\fill[name intersections={of=l8 and l1,total=\t}]
		(intersection-1) [black] circle(2pt) node {};
		\fill[name intersections={of=l9 and l1,total=\t}]
		(intersection-1) [black] circle(2pt) node {};
		\fill[name intersections={of=l10 and l1,total=\t}]
		(intersection-1) [black] circle(2pt) node {};
		
		\fill[name intersections={of=l2 and f1,total=\t, by=i3}]
			(intersection-1) [black] circle(2pt) node {};
		\fill[name intersections={of=l4 and f1,total=\t, by=i4}]
			(intersection-1) [black] circle(2pt) node {};
		
		\draw[fill=blue!90, opacity=0.3] (n) -- (nw) -- (sw) -- (s)  -- (i3)  -- (i2) -- (i1) -- (0.65,1) -- cycle;
		\draw[fill=green!90, opacity=0.3] (s) -- (i3) -- (se1) -- (ne1) -- (i4) -- (n) -- cycle;
		
		\coordinate (xp) at (2.2, 0, 0);
		\coordinate (xn) at (-1.2, 0, 0);
		\coordinate (yp) at (0, 1.2, 0);
		\coordinate (yn) at (0, -1.2, 0);

		\draw[-,opacity=0.0] (xn) -- (xp) ;
		\draw[-,opacity=0.0] (yn) -- (yp) ;

		\node[circle, black, fill, minimum size=3pt,inner sep=0pt, outer sep=0pt] at (n) {};
		\node[circle, black, fill, minimum size=3pt,inner sep=0pt, outer sep=0pt] at (nw) {};
		\node[circle, black, fill, minimum size=3pt,inner sep=0pt, outer sep=0pt] at (s) {};		
		\node[circle, black, fill, minimum size=3pt,inner sep=0pt, outer sep=0pt] at (sw) {};

		\node[circle, black, fill, minimum size=3pt,inner sep=0pt, outer sep=0pt] at (ne1) {};
		\node[circle, black, fill, minimum size=3pt,inner sep=0pt, outer sep=0pt] at (se1) {};	
		
	\end{tikzpicture}
};

\node[align=center] at (\xoo+\xio+4.9*\xo,\yoo -0.20) {combine vertices};

\node (hdiii) at (\xoo+\xio+6.1*\xo, \yoo + 0.5*\dyl + 0.5*\dyu) {
	\begin{tikzpicture}[scale=0.7]
		
		\coordinate (ne) at (2.0, 1.6/1.2);
		\coordinate (se) at (2, -1.6/1.2);
		\coordinate (n) at (-1/1.5, 1/1.5);
		\coordinate (s) at (-1/1.5, -1/1.5);
		\coordinate (nw) at (-1/1.1, 0.2/1.1);
		\coordinate (sw) at (-1/1.1, -0.2/1.1);
		
		\coordinate (ne1) at (1, 1);
		\coordinate (nw1) at (-1, 1);
		\coordinate (se1) at (1, -1);
		\coordinate (sw1) at (-1, -1);

		\path[name path=f1] (ne) -- (se) -- (s) -- (sw) -- (nw) -- (n) -- cycle;
		\path[name path=f2] (ne1) -- (se1) -- (sw1) -- (nw1) -- cycle;
		\path[black!90, opacity=0.7, shorten <= -0.1cm, shorten >= -0.1cm, name path=l1] (ne1) -- (se1);
		\path[black!90, opacity=0.7, shorten <= -0.2cm, name path=l2] (sw1) -- ++(0:3.2cm);
		\path[black!90, opacity=0.7, shorten <= -0.2cm, shorten >= -0.2cm, name path=l3] (sw1) -- (nw1);
		\path[black!90, opacity=0.7, shorten <= -0.2cm, name path=l4] (nw1) -- ++(0:3.2cm);
		
		\path[black!90, opacity=0.7, shorten >= -0.2cm, name path=l5, dashed] (nw) -- (ne);
		\path[black!90, opacity=0.7, shorten >= -0.2cm, name path=l6, dashed] (sw) -- (ne);
		\path[black!90, opacity=0.7, shorten >= -0.2cm, name path=l7, dashed] (s) -- (ne);
		
		\path[black!90, opacity=0.7, shorten >= -0.2cm, name path=l8, dashed] (sw) -- (se);
		\path[black!90, opacity=0.7, shorten >= -0.2cm, name path=l9, dashed] (nw) -- (se);
		\path[black!90, opacity=0.7, shorten >= -0.2cm, name path=l10, dashed] (n) -- (se);

		\fill[name intersections={of=l5 and l1,total=\t}]
		(intersection-1) [red!80!black] circle(2pt) node {};
		\fill[name intersections={of=l6 and l1,total=\t}]
		(intersection-1) [red!80!black] circle(2pt) node {};
		\fill[name intersections={of=l7 and l1,total=\t}]
		(intersection-1) [red!80!black] circle(2pt) node {};
		
		\fill[name intersections={of=l8 and l1,total=\t}]
		(intersection-1) [red!80!black] circle(2pt) node {};
		\fill[name intersections={of=l9 and l1,total=\t}]
		(intersection-1) [red!80!black] circle(2pt) node {};
		\fill[name intersections={of=l10 and l1,total=\t}]
		(intersection-1) [red!80!black] circle(2pt) node {};
		
		\fill[name intersections={of=l2 and f1, by=Ea}]
		(intersection-2) [black] circle(2pt) node {};
		\fill[name intersections={of=l4 and f1,by=Eb}]
		(intersection-2) [black] circle(2pt) node {};

		\draw[fill=orange!, opacity=0.5, name path=f1] (nw) -- (n)  -- (intersection of nw1--ne1 and n--ne) -- (ne1) -- (se1) -- (intersection of sw1--se1 and s--se) -- (s) -- (sw) -- cycle;

		\coordinate (xp) at (1.2, 0, 0);
		\coordinate (xn) at (-1.2, 0, 0);
		\coordinate (yp) at (0, 1.2, 0);
		\coordinate (yn) at (0, -1.2, 0);

		\draw[-,opacity=0.0] (xn) -- (xp) ;
		\draw[-,opacity=0.0] (yn) -- (yp) ;
		
		\node[circle, black, fill, minimum size=3pt,inner sep=0pt, outer sep=0pt] at (n) {};
		\node[circle, black, fill, minimum size=3pt,inner sep=0pt, outer sep=0pt] at (nw) {};
		\node[circle, black, fill, minimum size=3pt,inner sep=0pt, outer sep=0pt] at (s) {};		
		\node[circle, black, fill, minimum size=3pt,inner sep=0pt, outer sep=0pt] at (sw) {};

		\node[circle, black, fill, minimum size=3pt,inner sep=0pt, outer sep=0pt] at (ne1) {};
		\node[circle, black, fill, minimum size=3pt,inner sep=0pt, outer sep=0pt] at (se1) {};	
		
	\end{tikzpicture}
};

\node[align=center] at (\xoo+\xio+6.05*\xo,\yoo -0.25) {A-irredundancy};

\node (hpi) at (\xoo+\xio+6.9*\xo, \yoo + 0.5*\dyl + 0.5*\dyu) {
	\begin{tikzpicture}[scale=0.7]
		
		\coordinate (nn) at (0, 1);
		\coordinate (e) at (1, 0);
		\coordinate (ss) at (0, -1);
		\coordinate (w) at (-1, 0);
		
		\coordinate (n) at (-0.3, 1.2);
		\coordinate (nw) at (-1.0, 0.5);
		\coordinate (e) at (1, 0.0);
		\coordinate (sw) at (-1.0, -0.5);
		\coordinate (s) at (-0.3, -1.2);
		\coordinate (w) at (-1.1, 0);

		\draw[fill=orange!90!, opacity=0.5] (nn) -- (n) -- (nw) -- (w) -- (sw) -- (s) -- (ss) -- (e) -- cycle;

		\coordinate (xp) at (1.2, 0, 0);
		\coordinate (xn) at (-1.2, 0, 0);
		\coordinate (yp) at (0, 1.2, 0);
		\coordinate (yn) at (0, -1.2, 0);

		\draw[-,opacity=0.0] (xn) -- (xp) ;
		\draw[-,opacity=0.0] (yn) -- (yp) ;
		
		\node[circle, black, fill, minimum size=3pt,inner sep=0pt, outer sep=0pt] at (n) {};
		\node[circle, black, fill, minimum size=3pt,inner sep=0pt, outer sep=0pt] at (nw) {};
		\node[circle, black, fill, minimum size=3pt,inner sep=0pt, outer sep=0pt] at (s) {};		
		\node[circle, black, fill, minimum size=3pt,inner sep=0pt, outer sep=0pt] at (sw) {};		
		\node[circle, black, fill, minimum size=3pt,inner sep=0pt, outer sep=0pt] at (nn) {};
		\node[circle, black, fill, minimum size=3pt,inner sep=0pt, outer sep=0pt] at (ss) {};	
		\node[circle, black, fill, minimum size=3pt,inner sep=0pt, outer sep=0pt] at (w) {};
		\node[circle, black, fill, minimum size=3pt,inner sep=0pt, outer sep=0pt] at (e) {};		
		
	\end{tikzpicture}
};

\node[align=center] at (\xoo+\xio+6.9*\xo,\yoo -0.25) {primal};

\draw[->] (hi1) -- (hi1d);
\draw[->] (hi2) -- (hi2d);

\draw[->] (hi1d) -- (h1di);
\draw[->, dashed] (hi2d) -- (h1di);
\draw[->, dashed] (hi1d) -- (h2di);
\draw[->] (hi2d) -- (h2di);

\draw[->] (h1di) -- (h1dii);
\draw[->] (h2di) -- (h2dii);

\draw[->] (h1dii) -- (hdii);
\draw[->] (h2dii) -- (hdii);

\draw[->] (hdii) -- (hdiii);

\draw[<-, shorten >= -0.65cm] (hdiii) -- (hpi);

\node[align=center, anchor=east] at (\xoo+\xio+6.8*\xo+1.2,\yoo -4.8) {\textbf{Partial Double Description Method}};

\end{tikzpicture}

%% file: technical.tex
\section{The Partial Double Description Method}\label{sec:PDDM}
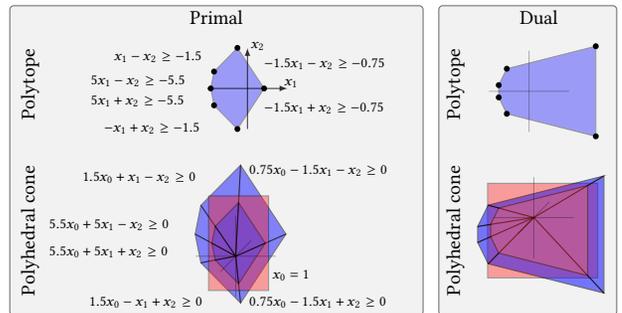
\begin{wrapfigure}[14]{r}{0.60\textwidth}
	\centering
	\vspace{-4.5mm}
	\scalebox{0.74}{\input{figures/polyhedral_figure_2}}
	\vspace{-2mm}
	\caption[]{Top: polytope in primal (left) and dual (right) space. Bottom: equivalent polyhedral cones in homogenized coordinates. In red: the plane the cone can be intersected with to recover the polytope.
	}
	\label{fig:poly_dual}
	\vspace{-3mm}
\end{wrapfigure}
In this section, we explain our \PDDM for computing convex hull approximations in greater detail. First, we introduce the needed notion of duality and our novel \PDDl (\PDD) representation for polyhedra. Then, we explain the \PDDM step by step as illustrated in~\Figref{fig:overview_PDDM}.

The \PDDM computes the convex hull of two $d$-dimensional polytopes $\bc{P}_1 = \bc{P}(\bm{A}_1,\bm{b}_1)$ and \mbox{$\bc{P}_2 = \bc{P}(\bm{A}_2,\bm{b}_2)$}, but uses the equivalent homogenized representation (see Section~\ref{sec:poly}) of $(d+1)$-dimensional cones {$\bc{P}_1'=\bc{P}(\bm{A}_1')$} and $\bc{P}_2'=\bc{P}(\bm{A}_2')$. Vertices in the original polytope now correspond to rays in the cone. In the following explanations we will use either term, depending on convenience.
The original polytope can be recovered from the cone, by intersecting it with the hyperplane $x'_0 = 1$ in primal, or with $x'_0 = -1$ in dual space (explained next) as visualized in~\Figref{fig:poly_dual}.

\paragraph{Duality}
The dual $\overline{\bc{P}}$ of a polytope $\bc{P}$ with a minimal set (containing no redundancy) of extremal vertices $\bc{R}$ enclosing the origin but not containing it in its boundary (to ensure a bounded dual) is defined as

\begin{equation}
	\overline{\bc{P}} = \{\bm{y} \in \R^d \; | \; \bm{x}^\top \bm{y} \leq 1 \; \forall \bm{x} \in \bc{P} \}
	= \bigcap_{x\in\bc{R}} \{\bm{y} \in \R^d \; | \; \bm{x}^\top \bm{y} \leq 1\},
	\label{eqn:dual_def}
\end{equation}
and for polyhedral cones $\bc{P}'$ as \cite{genov2015convex}
\begin{equation}
	\overline{\bc{P}'} = \{\bm{y}' \in \R^{d+1} \; | \; \bm{x}'^\top \bm{y}' \leq 0 \; \forall \bm{x}' \in \bc{P}' \}.
	\label{eqn:dual_def_cone}
\end{equation}
\Figref{fig:poly_dual} shows an example of the dual of a polytope.
Important for the remaining section are four properties of the transform between primal and dual.
\begin{inparaenum}[1)]
\item The dual of a polyhedron is also a polyhedron.
\item It is inclusion reversing: $\bc{P} \subset \bc{Q}$ if and only if $\overline{\bc{Q}} \subset \overline{\bc{P}}$,
\item the \vrep of the dual corresponds to the \hrep of the primal and vice versa: $\bc{P} = \bc{P}(\bm{A}',\bc{R}')$ implies $\overline{\bc{P}} = \bc{P}(\bc{R}'^\top,\bm{A}'^\top)$, where $(\cdot)^\top$ denotes transpose (note that this implies that the vertices of the primal correspond to the supporting hyperplanes of the dual and vice-versa), and
\item the dual of the dual of a polyhedron is the original primal polyhedron $\overline{\overline{\bc{P}}} = \bc{P}$.
\end{inparaenum}

\paragraph{Partial double description}
We leverage these duality properties in two ways: We translate the convex hull problem in primal space to an intersection problem in dual space (only involving a transpose given a DD or \PDD) where we compute a \vrep \emph{under-approximating} the intersection in dual space to obtain an \hrep \emph{over-approximating} the convex hull in primal space (using inclusion reversion). To compute these intersections efficiently, we introduce the \PDDl (\PDD) as a relaxation of the Double Description (DD) (Section~\ref{sec:poly}) as discussed in the overview.

Formally, the \PDD of a $(d+1)$-dimensional polyhedral cone is the pair of constraints and rays $(\bm{A}',\bc{R}')$ with $\bm{A}' \in \R ^{m \times (d+1)}$ and $\bc{R}' \in \R ^{n \times (d+1)}$ where the \vrep is an under-approximation of the \hrep or more formally, where for any row $\bm{r} \in \bc{R}'$ and constraint $\bm{a} \in \bm{A}'$, $\bm{a} \bm{r} \geq 0$ holds.

We call constraints $\bm{a}_j \in \bm{A}'$ \emph{active} for a given ray $\bm{r}_i \in \bc{R}'$, if they are fulfilled with equality, that is $\bm{a}_j\bm{r}_i=0$. We store this relationship as part of the PDD in what we call the incidence matrix \mbox{$\bc{I} \in \{0,1\}^{n \times m}$}: $\bc{I}_{i,j} = 1$ if $\bm{a}_j\bm{r}_i=0$ and $\bc{I}_{i,j} = 0$ otherwise. Further, we define the partial ordering on $\bc{I}$: $\bc{I}_i \subseteq \bc{I}_j$ iff $\bc{I}_{i,k} \leq \bc{I}_{j,k}$, $\forall$ $1\leq k\leq m$. Intuitively this corresponds to a row in the incidence matrix being only lesser than another if the set of active constraints of the associated ray is a strict subset of that of the other. Next, we describe \PDDM as illustrated in \Figref{fig:overview_PDDM}.

\subsection{Conversion to Dual} \label{sec:polyhedra_duality}

Given the two polyhedral cones $\bc{P}_1$ and $\bc{P}_2$ in \PDD representation $(\bm{A}_1',\bc{R}_1')$ and $(\bm{A}_2',\bc{R}_2')$ (1\St column in~\Figref{fig:overview_PDDM}), the first step of the \PDDM is to convert them to their dual space representations $(\bc{R}_1'^\top,\bm{A}_1'^\top)$ and $(\bc{R}_2'^\top,\bm{A}_2'^\top)$ \cite{furkuda20polyhedral} (2\Nd column).

\subsection{Intersection} \label{sec:eff_inter}
The next step in the \PDDM is the intersection in dual space (columns 3 to 5 in \Figref{fig:overview_PDDM}).
Recall that the standard approach (DDM) for the intersection of polyhedra in DD is to sequentially add the constraints of one polytope to the other, computing exact \vreps at every step. This however can increase the number of vertices quadratically in every step resulting in an exponential size of the intermediate representation. Instead, we add all constraints jointly in one step, leveraging our \PDD. In the following description of the intersection, we adopt the polytope (not cone) view and consider a general polytope $(\bm{A},\bc{R})$.

\begin{figure}[t]
	\centering
	\hspace{4mm}
	\scalebox{0.8}{\input{figures/PDD_mult_const_fig.tex}}
	\vspace{0mm}
	\caption{Adding a batch of three constraints (blue thick lines) to a polytope in \PDD. Vertices are separated into $\bc{R}'_+$ (black), $\bc{R}'_0$ (none), and $\bc{R}'_-$ (red). Ray-shooting discovers new vertices $\bc{R}'_*$ (blue), avoiding the superfluous green points, but missing an extremal vertex (yellow) (a). Exact intersection (b), result of joint constraint processing (c), and under-approximation after enforcing A-irredundancy (d).}
	\label{fig:PDD_mult_Const}
\end{figure}
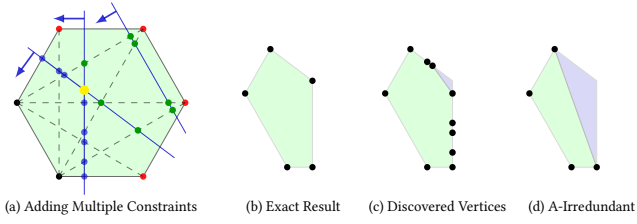

\paragraph{Batch intersection}
To intersect a polytope $(\bm{A},\bc{R})$ in PDD with a batch of constraints represented by the matrix $\bm{\widetilde{A}}$ and inducing the polyhedron $\bc{P}(\bm{\widetilde{A}})$, we separate the vertices in $\bc{R}$ into three sets depending on whether they satisfy all to-be-added constraints with inequality ($\bc{R}_+$), some only with equality ($\bc{R}_0$), or violate at least one ($\bc{R}_-$). This corresponds to these points lying inside, on the boundary of, or outside of the polyhedron $\bc{P}(\bm{\widetilde{A}})$. An example is shown in \Figref{fig:PDD_mult_Const}(a): the three added constraints are shown in blue and the vertices in $\bc{R}'_+$ (black), $\bc{R}'_0$ (none), and $\bc{R}'_-$ (red).

Now we employ a technique called ray-shooting \cite{marechal2017efficient} and shoot
 a ray $\overrightarrow{\bm{r}_+\bm{r}_-}$ from a vertex $\bm{r}_+ \in \bc{R}_+$ inside the intersection $\bc{P}(\bm{A} \cap \bm{\widetilde{A}})$ to a vertex $\bm{r}_- \in \bc{R}_-$ outside the intersection. We record the first hyperplane $\bc{H} = \{ \bm{x} \in \R^d \,|\, \bm{\widetilde{a}}_i\bm{x} = 0\}$ corresponding to one of the new constraints $\bm{\widetilde{a}}_i \in \bm{\widetilde{A}}$ that intersects with the ray $\overrightarrow{\bm{r}_+\bm{r}_-}$. We add the point $\bm{r}_*$ at which $\overrightarrow{\bm{r}_+\bm{r}_-}$ intersects $\bc{H}$ to the set of discovered points $\bc{R}_*$.
Doing so for all combinations of $(\bm{r}_+,\bm{r}_-) \in \bc{R}_+ \times \bc{R}_-$ yields the set of points
\begin{equation*}
	\bc{R}_* = \{\bm{r}_* = \overrightarrow{\bm{r}_+\bm{r}_-} \cap \bc{H} \;| \, (\bm{r}_+,\bm{r}_-) \in \bc{R}'_+ \times \bc{R}'_-\}.
\end{equation*}
The \vrep of the resulting intersection is now the union $ \bc{R}_+ \cup \bc{R}_0 \cup \bc{R}_*$. In~\mbox{\Figref{fig:PDD_mult_Const} (a)} the rays $\overrightarrow{\bm{r}_+\bm{r}_-}$ are dashed lines from all black to all red vertices and discover new vertices $\bc{R}_*$ (blue). Only using the first intersections, immediately discards the green points, however, we also do not discover the yellow point, which is an extremal vertex of the exact intersection (b), obtaining instead the under-approximation (c).

\begin{figure}[t]
	\centering
	\scalebox{0.79}{\input{figures/PDD_inter_fig.tex}}
	\vspace{-2mm}
	\caption{Boosting intersection precision by combining both directions of batch intersection. Input polytopes in \PDD with exact \hrep (black) and approximate \vrep ($\bc{P}_1$ green and $\bc{P}_2$ red) (a), batch intersection of $\bc{P}_1$ with the \hrep of $\bc{P}_2$ (b), batch intersection in the opposite direction (c), combining both intersections (d), and applying A-irredundancy (e).}
	\label{fig:PDD_inter}
	\vspace{-0.3cm}
\end{figure}

\paragraph{Boosting precision}
Batch intersection is asymmetric: The PDD of one polytope is intersected with the \hrep of another, to obtain an exact \hrep and under-approximating \vrep of the intersection (compare \Figref{fig:PDD_inter} (b) and (c)). By performing it in both directions, i.e., intersecting  $(\bc{R}_1'^\top,\bm{A}_1'^\top)$ with $(\bc{R}_2'^\top,\bm{A}_2'^\top)$ and vice-versa in our example, we obtain two different under-approximations of the intersection (see  \Figref{fig:PDD_inter} (b) and (c)). Their convex hull (obtained by the union of vertices) is still a sound under-approximation of the exact intersection and more precise than the individual under-approximations. This is illustrated in \Figref{fig:PDD_inter}, where the exact intersection (blue in (d)) of the two \hreps (grey in (a)) is recovered despite the union of the input \vreps (green and red in (a)) not covering it. This is due to the synergy between \PDD and \PDDM: the under-approximate \vrep of the first polytope is intersected with the exact \hrep of the second one and vice versa. We see the same behaviour in \Figref{fig:overview_PDDM}, where both uni-directional intersections (4\Th column) are under-approximations, but their union is exact (5\Th column).

Empirically we find that this is crucial to minimize the precision loss due to using approximations. Further, the intersection results are exact for small dimensions $d\leq 4$ of cones (see Theorem~\ref{lem:inter_approx_exact}).

\subsection{Enforcing A-Irredundancy}\label{sec:a_irredundancy}

Despite using batch intersection, the number of vertices can grow quickly when computing multiple convex hulls sequentially in the \SBLMl. Therefore, some notion of redundancy is needed to efficiently reduce the representation size.
The standard definitions of irredundancy are:
\begin{inparaenum}[1)]
\item the set of unique extremal rays of the cone $\bc{P}(\bm{A}')$ are irredundant, and
\item a ray $\bm{r}_i$ is irredundant if removing it leads to a different cone $\bc{P}(\bc{R}') \neq \bc{P}(\bc{R}'\setminus{\bm{r}_i})$.
\end{inparaenum}
For an exact DD, an irredundant representation does not lose precision and can be computed by retaining only rays with rank $d-1$ (which can be cheaply computed using the incidence matrix $\bc{I}$). However, a \PDD $(\bm{A}',\bc{R}')$ usually does not include all or even any extremal rays of the cone $\bc{P}(\bm{A}')$.
Consequently, enforcing the first irredundancy definition could remove all rays.
Enforcing the second definition is expensive to compute in the absence of a full set of extremal rays, as the full convex hull problem has to be solved to assess the removal af a ray.

Therefore, we propose \emph{A-irredundancy} requiring for all rays $\bm{r}_i \in \bc{R}'$ that there may not be another generator $\bm{r}_j \in \bc{R}'$ with a larger (by inclusion) active constraint set. Formally and using the partial ordering defined above, we require for an A-irredundant \PDD:
\begin{equation*}
	\bc{I}_i \not\subseteq\bc{I}_j,\quad\text{for all } i,j \in \{1,...,n\},\ i\neq j.
\end{equation*}
Any ray fulfilling a subset (including the same) constraints with equality as another ray, is removed until the above definition is satisfied to obtain an A-irredundant representation.
Extremal rays will always be retained as they have the maximum number of active constraints and there are never two with the same active set. Intuitively, this enforces that no two rays lie in the interior of the same face of the polyhedron.

We illustrate the effect of enforcing A-irredundancy once in~\Figref{fig:PDD_inter} where we use it to obtain the polytope \ref{fig:PDD_inter} (e) from \ref{fig:PDD_inter} (d) and see that all extremal rays are retained and no precision is lost.
In~\Figref{fig:PDD_mult_Const} we apply it to polytope \ref{fig:PDD_mult_Const} (c) where the \PDD misses one extremal vertex to obtain \ref{fig:PDD_mult_Const} (d) and see that here the resulting reduction in generator set size can come at the cost of a precision loss.
Enforcing A-irredundancy in the 6\Th column of~\Figref{fig:overview_PDDM} (removing the red vertices), recovers the minimal set of extremal rays. 
Note that for rays of equal incidence there are multiple possibilities which to retain, as is illustrated in \Figref{fig:PDD_ill} (e) and (f).

\subsection{Conversion to Primal}
Translating the A-irredundant \PDD obtained as described above, back to primal space concludes the \PDDM and yields the (generally) approximate convex hull of $\bc{P}_1$ and $\bc{P}_2$ illustrated in the 7\Th column of~\Figref{fig:overview_PDDM}.

\subsection{Formal Guarantees}
In this subsection, we first show that the \PDDM is sound and exact in low dimensions, before analysing its worst-case complexity.

\begin{wrapfigure}[20]{r}{0.56 \linewidth}
	\IncMargin{-0.4em}
	\begin{algorithm}[H]
		\SetAlgoLined
		\KwResult{Intersected polytope $(\bm{A}',{\bc{R}'_p})$}
		\KwIn{polytope $(\bm{A}_p,\bc{R}_p)$, constraint matrix ${\bm{A}_q}$}
		Initialize $\bc{R}_-,\bc{R}_0,\bc{R}_+,\bc{R}_* = \emptyset,\emptyset,\emptyset,\emptyset$\\
		\For{$r$ in $\bc{R}_p$}{
			\uIf{$min({\bm{A}_q}\bm{r})<0$}{Add $r$ to $\bc{R}_-$}
			\uElseIf{$min({\bm{A}_q}\bm{r})>0$}{Add $r$ to $\bc{R}_+$}
			\uElse{Add $r$ to $\bc{R}_0$}
		}
		\For{$r_+$ in $\bc{R}_+$}{
			\For{$r_-$ in $\bc{R}_-$}{
				Compute $r_*$ via ray-shooting from $r_+$ to $r_-$\\
				Add $r_*$ to $\bc{R}_*$
			}
		}
		Construct new PDD $(\bm{A}_p \cup \bm{A}_q, \bc{R}_0 \cup \bc{R}_+ \cup \bc{R}_*)$\\
		Make PDD A-irredundant\\
		\Return{PDD}
		\vspace{3mm}
		\caption{Batch Intersection}
		\label{alg:joint_inter}
	\end{algorithm}
	\DecMargin{-0.4em}
\end{wrapfigure}
\paragraph{Soundness guarantee} Computing a sound over-approximation of the convex hull of two polytopes in primal space, by inclusion-inversion, is equivalent to computing a sound under-approximation of the intersection of their dual space representations.
Since the primal-dual conversion employed in the \PDDM is exact, a sound under-approximation of the intersection of two polytopes in \PDD in dual space implies overall soundness.
Enforcing A-irredundancy on a polytope $\bc{P}$ to yield $\bc{Q}$ can only remove generators, yielding $\bc{Q} \subseteq \bc{P}$. It follows directly that $\bc{Q}$ is a sound under-approximation, if $\bc{P}$ is.
If both polytopes $\bc{P}'_q$ and $\bc{P}'_p$ generated by the vertex sets obtained for the two directions of batch intersection are sound under-approximations of the true intersection of the exact \hreps, it follows that their union $\bc{P}'$ is also a sound under-approximation.
Hence, the soundness of the \PDDM follows from the soundness of the batch intersection step:
\begin{restatable}{lemma}{soundness}
	The batch intersection $\bc{P}_p' = (\bm{A}',\bc{R}'_p)$ of a polytope $\bc{P}$ in \PDD $(\bm{A}_p,\bc{R}_p)$ with the exact constraints $\bm{A}_q$ of a polytope $\bc{Q}$ computed as described above and detailed in \Algref{alg:joint_inter}, is a sound under-approximation of the intersection of the two exact \hreps $\bm{A}_p$ and $\bm{A}_q$:
	\begin{align*}
	\{\bm{x} \in \R^d | \bm{A}' \bm{x} \geq 0 \} &= \{\bm{x} \in \R^d | \bm{A}_p \bm{x} \geq 0 \wedge \bm{A}_q \bm{x} \geq 0 \},\\
	\bigg\{\sum_{\bm{r}_i \in \bc{R}'_p} \lambda_i \bm{r}_i | \sum_{i} \lambda_i \leq 1, \lambda_i \in \R_0^+\bigg\} &\subseteq \{\bm{x} \in \R^d | \bm{A}_p \bm{x} \geq 0 \wedge \bm{A}_q \bm{x} \geq 0 \}.
	\end{align*}
	\label{lem:soundness}
\end{restatable}
\begin{proof}
	Recall that a \PDD consists of an exact \hrep and an under-approximate \vrep. 
	The intersection of two polytopes in \hrep is simply the union of all constraints, allowing for an exact intersection of the \hreps. 
	Hence, it remains to show that the resulting \vrep $\bc{R'}_p$ is a sound under-approximation of the \hrep $\bm{A}'$.
	For this, it is sufficient to show that, by construction, every vertex $\bm{r} \in \bc{R}'_p$ satisfies all constraints in $\bm{A}'$. Recall that $\bc{R}'_p$ is the union of three groups of vertices (see \Secref{sec:eff_inter} or \Algref{alg:joint_inter}):
	\begin{itemize}
		\item[$\bc{R}_+$] vertices of the generating set $\bc{R}_p$ that satisfy all constraints in $\bm{A}_q$ strictly,
		\item[$\bc{R}_0$] vertices of the generating set $\bc{R}_p$ that satisfy all constraints in $\bm{A}_q$, at least one with equality,
		\item[$\bc{R}_*$] the first intersections $\bm{r}_*$ of rays from a vertex in $\bm{r}_+ \in \bc{R}_+$ to a vertex in $\bm{r}_- \in \bc{R}_-$ (vertices in $\bc{R}_p$ not satisfying all constraint in $\bm{A}_q$) with the hyperplanes defined by $\bm{A}_q$. Since $\bm{r}_-$ lies outside $\bc{Q}$ while $\bm{r}_+$ lies inside, an intersection $\bm{r}_*$ is guaranteed to exist and lie between the two. By convexity of $\bc{P}$, $\bm{r}_*$ satisfies all constraints of $\bm{A}_p$. Further, since $\bm{r}_*$ is the first intersection of the ray with a constraint in $\bm{A}_q$ as seen from $\bm{r}_+$, which satisfies all constraints in $\bm{A}_q$, $\bm{r}_*$ also satisfies all constraints in $\bm{A}_q$.
	\end{itemize}
	Consequently, all vertices in the generating set $\bc{R}'_p$ satisfy all constraints of both $\bc{P}$ and $\bc{Q}$. It follows that $\bc{R}'_p \subseteq \bc{Q}\cap\bc{P}$ and hence that the generated polytope is a sound under-approximation.
\end{proof}

\begin{wrapfigure}[9]{r}{0.53 \linewidth}
	\IncMargin{-0.4em}
		\begin{algorithm}[H]
			\KwResult{Intersected polytope $(\bm{A}_p\cup\bm{A}_q,\bc{R}')$}
			\KwIn{polytope $(\bm{A}_p,\bc{R}_p)$ and $(\bm{A}_q,\bc{R}_q)$}
			Compute $(\bm{A}',\bc{R}_p') = (\bm{A}_p \cup \bm{A}_q,\bc{R}_p)$ with Alg.\ref{alg:joint_inter}\\
			Compute $(\bm{A}',\bc{R}_q') = (\bm{A}_p \cup \bm{A}_q,\bc{R}_q)$ with Alg.\ref{alg:joint_inter}\\
			Construct new PDD $(\bm{A}', \bc{R}_p' \cup \bc{R}_q')$\\
			Make PDD A-irredundant\\
			\Return{PDD}
			\caption{PDDM Intersection}
			\label{alg:inter_approx}
		\end{algorithm}
	\DecMargin{-0.4em}
\end{wrapfigure}
\paragraph{Exactness guarantee} Further, we can show that for relatively low dimensional polyhedra in Double Description, as they are often encountered during the first step of lifting in the SBLM, the \PDDM as described above is not only sound but actually exact. To this end, let us first show the following guarantee for the intersection of a cone in DD with a matrix of constraints:
\begin{lemma}
	Given a Double Description $(\bm{A}_p,\bc{R}_p)$ of a polyhedral cone and the constraint matrix $\bm{A}_q$, adding all constraints jointly as per \Algref{alg:joint_inter} is guaranteed to yield a double description $(\bm{A}_p\cup\bm{A}_q,\bc{R}_p')$ enumerating \emph{all} extremal rays $r'$ of the $\bm{A}_p\cup\bm{A}_q$-induced cone with one of the following properties:
	\begin{enumerate}
		\item $r'$ is extremal (rank $d-1$) in the $\bm{A}_p$-induced cone.
		\item $r'$ is of rank $d-2$ in the $\bm{A}_p$-induced cone.
	\end{enumerate}
	\label{lem:joint_inter}
\end{lemma}
\begin{proof}
	We can formally divide the rays of the new PDD $\bc{R}'$ into the two non-overlapping sets: 
	\begin{itemize}
		\item $\bc{R}_+ \cup \bc{R}_0$: Rays in $\bc{R}_p$ not violating any constraint $ a\in \bm{A}_q$
		\item $\bc{R}_*$: Rays discovered by ray-shooting
	\end{itemize}
	Since $(\bm{A}_p,\bc{R}_p)$ is a DD of the $\bm{A}_p$-induced cone it enumerates all extremal rays. If $r'$ is extremal in both the $\bm{A}$-induced and the $\bm{A}_p\cup\bm{A}_q$-induced cones, it is included in $\bc{R}_p$ and does not violate any constraints. Therefore, it is included in the first group above and will be part of $\bc{R}_p'$, which concludes the proof of the first point.
	Any ray of rank $d-2$ can, by definition, be represented as a positive combination of two extremal rays, that is rays of rank $d-1$. 
	As we assume ray $r'$ to be extremal in the $\bm{A}_p\cup\bm{A}_q$-induced cone and therefore have rank $d-1$, it necessarily intersects at least one constraint $\bm{a} \in \bm{A}_q$ and is extremal to the $\bm{A}_p\cup\bm{a}$-induced cone. 
	Consequently exactly one of the extremal rays used to construct it has to lie on either side of thy hyperplane induced by constraint $\bm{a}$. 
	Therefore, they will be included in the sets $\bc{R}_+$ and $\bc{R}_-$ and the intersection will be discovered as part of the ray-shooting, concluding the proof of the second point.
\end{proof}

Using this result, we can proof the following guarantee for intersections of two cones in DD using our batch intersection and precision boosting approach, described in \Secref{sec:eff_inter} and~\Algref{alg:inter_approx}:
\begin{lemma}
	Given the double descriptions $(\bm{A}_p,\bc{R}_p)$ and $(\bm{A}_q,\bc{R}_q)$ of two polyhedral cones, their intersection computed as per \Algref{alg:inter_approx} is guaranteed to be a partial double description $(\bm{A}_p\cup\bm{A}_q,\bc{R}')$ enumerating \emph{all} extremal rays $r'$ of the $(\bm{A}_p\cup\bm{A}_q)$-induced cone with one of the following properties:
	\begin{enumerate}
		\item $r'$ is extremal in the $\bm{A}_p$-induced cone.
		\item $r'$ is extremal in the $\bm{A}_q$-induced cone.
		\item $r'$ is of rank $d-2$ in the $\bm{A}_p$-induced cone.
		\item $r'$ is of rank $d-2$ in the $\bm{A}_q$-induced cone.
	\end{enumerate}
	\label{lem:inter_approx}
\end{lemma}
\begin{proof}
	The proof follows directly from applying Lemma \ref{lem:joint_inter} to both applications of \Algref{alg:joint_inter}, the insight that every extremal ray discovered by either will be included in the final generating set $\bc{R}'$ and the observation that the intersection of the exact \hreps, trivially is the union of their respective constraints, leading to a valid partial double description.
\end{proof}

Using these results, we can in turn proof that the intersection of two polyhedral cones of up to dimension $4$ in DD using the approach described above is exact:
\begin{restatable}{lemma}{interapproxexact}
	Given the Double Descriptions $(\bm{A}_p,\bc{R}_p)$ and $(\bm{A}_q,\bc{R}_q)$ of two polyhedral cones $\bc{P}$ and $\bc{Q}$ of dimension $d \leq 4$, the PDD of their intersection $(\bm{A}_p\cup\bm{A}_q,\bc{R}')$ computed as described above and detailed in \Algref{alg:inter_approx} is an exact DD with an irredundant generating set $\bc{R}'$.
	\label{lem:inter_approx_exact}
\end{restatable}

\begin{proof}
	For briefness sake, we will only show the proof for $d=4$ here.
	Let $\bc{R}^*$ be the set of extremal rays of the $(\bm{A}_p \cup \bm{A}_q)$-induced polyhedral cone. Consequently $\bm{r}^* \in \bc{R}^*$ has the rank $d-1=3$ in this cone and therefore it fulfills 3 linearly independent constraints in $\bm{A}_p \cup \bm{A}_q$ with equality. This leads to the following four exhaustive options:
	\begin{enumerate}
		\item all 3 constraints are part of $\bm{A}_p$, $\bm{r}^*$ is extremal in $\bc{P}_p$,
		\item all 3 constraints are part of $\bm{A}_q$, $\bm{r}^*$ is extremal in $\bc{P}_q$,
		\item 2 constraints are part of $\bm{A}_p$ and 1 of $\bm{A}_q$, $\bm{r}^*$ is of rank $d-2=2$ in $\bc{P}_p$,
		\item 2 constraints are part of $\bm{A}_q$ and 1 of $\bm{A}_p$, $\bm{r}^*$ is of rank $d-2=2$ in $\bc{P}_q$.
	\end{enumerate}
	All of those are enumerated by \Algref{lem:inter_approx}. Hence, $\bc{R}'$ will include all extremal rays of the $(\bm{A}_p \cup \bm{A}_q)$-induced cone. In this case A-irredundancy is equivalent to irredundancy.
\end{proof}

\paragraph{Complexity analysis} Finally, we can show that computing an over-approximation of the convex hull of two $d$-dimensional, bounded polytopes in \PDD using the \PDDM has polynomial complexity:
\begin{restatable}{lemma}{complexity}
	Given the \PDD of two $d$-dimensional, bounded polytopes with a \vrep of at most $n_v$ vertices and an \hrep of at most $n_a$ constraints, computing a sound over-approximation of their convex hull using the \PDDM as described above and detailed in \Algref{alg:inter_approx} has a worst-case time complexity of $\bc{O}( n_v \cdot n_a^4 + n_a^2 \log(n_a^2))$.
	\label{lem:complexity}
\end{restatable}

\begin{proof}
	The \PDDM can be broken down into its six components illustrated in \Figref{fig:overview_PDDM}:
	\begin{enumerate}
		\item Conversion from primal to dual representation (\Secref{sec:polyhedra_duality})
		\item Adding the constraints of one polytope to the other, or more concretely separation of vertices into the three sets $\bc{R}_+$, $\bc{R}_0$, and $\bc{R}_-$ (\Secref{sec:eff_inter} or first half of \Algref{alg:joint_inter})
		\item Discovery of new vertices via ray-shooting (\Secref{sec:eff_inter} or second half of \Algref{alg:joint_inter})
		\item Combining the vertices of the two intersection directions (\Secref{sec:eff_inter} or \Algref{alg:inter_approx})
		\item Enforcing of A-irredundancy (\Secref{sec:a_irredundancy} or \Algref{alg:inter_approx})
		\item Conversion from dual to primal representation (\Secref{sec:polyhedra_duality})
	\end{enumerate}
	Primal-dual conversions and combining of vertices can be computed in constant time, as this only involves computing the transpose and concatenation which can be done implicitly by changing the indexing of the corresponding matrices. Therefore, we will focus on the remaining three steps, which are all conducted in dual space. 
	
	In the following we assume the setting, of two $d$-dimensional, bounded polytopes which in dual-space are defined by $\bc{P}=(\bm{A}_p,\bc{R}_p)$ and $\bc{Q}=(\bm{A}_q,\bc{R}_q)$. For convenience's sake, we assume the number of vertices to be $n_v = \max(|\bc{R}_p|, |\bc{R}_q|)$ and number of constraints $n_a = \max(|\bm{A}_p|, |\bm{A}_q|)$. Note that their roles are reversed compared to a primal space representation.
	
	\paragraph{Adding constraints and separating vertices}
	Recall that in dual space we compute the intersection of the two polytopes $\bc{P}$ and $\bc{Q}$. The first step of intersecting $\bc{P}$ with $\bc{Q}$ is to split all points in $\bc{R}_p$ into the three groups $\bc{R}_+$, $\bc{R}_0$, and $\bc{R}_-$ defined in \Secref{sec:eff_inter} depending on whether the lie inside, on the border of or outside the polytope defined by $\bm{A}_q$ as per the first half of \Algref{alg:joint_inter}. 
	This requires (at worst) evaluating $\bm{a}_i \bm{r}_j - b_i \{>, =, <\} 0$ for all $\bm{r}_j \in \bc{R}_p$ and $\bm{a}_i, b_i \in \bm{A}_q$. 
	Where the addition and comparison are dominated by the $d$-dimensional dot-product between $\bm{a}_i$ and $\bm{r}_j$, leading to a total complexity of this step of order $\bc{O}(d \cdot n_a \cdot n_v)$. Note that incidence matrix columns corresponding to the new constraints are added and populated without any extra computation with $0$s for the vertices in $\bc{R}_+$ and $1$s for vertices in $\bc{R}_0$.
	
	\paragraph{Ray-shooting}
	Recall that to discover new generating vertices, the first intersections between the rays shot from all generating vertices of $\bc{P}$ lying inside $\bc{Q}$, $\bm{r}_+ \in \bc{R}_+$, to all vertices lying outside $\bc{Q}$, $\bm{r}_- \in \bc{R}_-$, and all constraints in $\bm{A}_q$ are computed. At worst there are no vertices in group $\bc{R}_0$ and all vertices are spread equally between $\bc{R}_+$ and $\bc{R}_-$, leading to ${n_v^2}/{4}$ rays to be intersected with $n_a$ constraints where each intersection corresponds to computing a ratio of dot-products and is order $\bc{O}(d)$. Selecting the first intersection for each ray is linear in the intersection number. Consequently, the ray-shooting process overall is $\bc{O}(d \cdot n_a \cdot n_v^2)$. 
	Note that this adds new incidence matrix rows corresponding to the new vertices $\bc{R}_*$, which can then be populated with the row obtained by the elementwise $and$ of the two vertices generating the ray and a $1$ in the column associated with the constraint of the first intersection which is linear $\bc{O}(n_v)$ and dominated by the previous term.
	
	\paragraph{Enforcing A-irredundancy}
	The intermediate state prior to enforcing A-irredundancy contains at most $n= 2 (n_v+n_v^2/4)$ vertices, consisting of the at most $n_v$ vertices in $\bc{R}_+$ and the at most $n_v^2/4$ vertices in $\bc{R}_*$, discovered during ray shooting, for both intersection directions. To enforce A-irredundancy, vertices are first sorted in descending order by the number of active constraints which is order $\bc{O}(n \log(n))$. Then starting with the first vertex, row-wise inclusion of the corresponding incidence matrix rows is checked for all following elements. Each check is $\bc{O}(n_a)$ and $({n^2-n})/{2}$ checks have to be performed in the worst case that is, if no element is removed. This leads to an overall complexity of $\bc{O}(n_a \cdot n_v^4 + n_v^2 \log(n_v^2))$ for enforcing A-irredundancy.
	
	\paragraph{\PDDM complexity}
	Putting the three elements together and observing $d<n_v$ for any $d$-dimensional, bounded polytope, we observe that both the ray-shooting and the separation of vertices get dominated by the last step of enforcing A-irredundancy. Swapping the roles of $n_v$ and $n_a$ to derive an expression in terms of primal space entities, we arrive at an overall complexity of $\bc{O}( n_v \cdot n_a^4 + n_a^2 \log(n_a^2))$.
\end{proof}

\section{\SBLMl} \label{sec:FastPoly}

\begin{wrapfigure}[14]{r}{0.59 \linewidth}
	\vspace{-10.5mm}
	\IncMargin{-0.4em}
	\begin{algorithm}[H]
		\SetAlgoLined
		\KwIn{Variable ordering $\bc{I}$, input polytope $\bc{P}$, set of bounding regions $\bc{D}$ and set of bounds $\bc{B}$}
		\KwOut{Jointly constraining polytope $\bc{K}$}
		\eIf{$|\bc{I}|>0$}{
			Get next output variable: $y \leftarrow \bc{I}_{0}$\\
			\ForEach{$\bc{D}^i,\bc{B}^i$ in $\bc{D},\bc{B}$}{
				Split region: $\bc{P}_i = \bc{P} \cap \bc{D}^i$\\
				Apply \SBLM: $\bc{K}_i \leftarrow \texttt{\SBLM}(\bc{I}_{1:end},\bc{P}_i,\bc{D},\bc{B})$\\
				Extend into space including $y$: $\bc{K}_i \leftarrow \bc{K}_i \times \R$\\
				Apply bounds $\bc{B}^i$: $\bc{K}_i \leftarrow \bc{K}_i \cap \bc{B}^i$\\
			}
			Compute convex hull: $\bc{K} = \texttt{\PDDM}(\{\bc{K}_i\}_i)$\\
			\Return{$\bc{K}$}
		}{
			\Return{$\bc{P}$}}
		\caption{\SBLMl (\texttt{\SBLM})}
		\label{alg:SBLM}
	\end{algorithm}
	\DecMargin{-0.4em}
\end{wrapfigure}

\sloppy
In this section, we explain the \SBLMl in greater detail.
Recall that we use the \SBLM to compute k-neuron abstractions, by approximating the convex hull ${\conv(\{(\bm{x},\bm{f}(\bm{x})) \, | \, \bm{x} \in \bc{P} \subseteq [l_x,u_x]^{k}\})}$ for a group of $k$ neurons and their activation functions $\bm{f}(\bm{x}) = [f_1(x_1),...,f_k({x_k})]^\top$, assuming that their inputs are constrained by the polytope $\bc{P}$.

At a high level, we first decompose the input polytope into regions where we can bound all activation functions tightly.
Then, we extend these regions into the output space and apply linear constraints corresponding to the (relaxed) activations.
Taking the convex hull of the resulting polytopes yields an \hrep encoding the k-neuron abstraction.

To increase the efficiency of this approach, we use a decomposition method we call \emph{splitting} and then recursively extend and bound the resulting polytopes by one output variable at a time, which we call \emph{lifting}. This minimizes the dimensionality in which we have to compute the convex hulls.
We formalize this in~\Algref{alg:SBLM} and explain both splitting and lifting below after stating the prerequisites for the \SBLM.

\subsection{Prerequisites}
For simplicities' sake, we assume just one type of activation function $f\colon \bb{D} \rightarrow \R$, with domain $\bb{D}$, is to be bounded. Now the \SBLM requires a set of intervals $\bc{D}^i$ (e.g., $x_j \leq 0, x_j \geq 0$ for ReLU), covering the domain $\bb{D}$ (e.g., $\R$ for ReLU), and a pair of tight linear constraints $\bc{B}^i$ upper and lower bounding the function output (e.g., $y_j \leq 0$ and $y_j \geq 0$, and $y_j \leq x_j$ and $y_j \geq x_j$, respectively, for ReLU) on each of the intervals obtained by intersecting the interval $[l_x,u_x]_i$ defined by the neuron-wise bounds with the intervals $\bc{D}^i$. More formally, we require the intervals
\begin{align*}
	\bc{D}^i &= [c_i,d_i], \quad c_i,d_i \in \overline{\R} \text{ and } c_i \leq d_i, \\
	\bb{D} &\subseteq \bigcup_i \bc{D}^i,
\end{align*}
with the affinely extended real numbers $\overline{\R} = \R \cup \{-\infty,\infty\}$ and the bounds on these intervals
\begin{align*}
	&\bc{B}^i = (a_i^\leq,a_i^\geq), \quad a_i^{\{\leq,\geq\}}(x) = a x + b, \; a,b\in \R \quad s.t.\\
	&a_i^\leq(x) \leq f(x) \leq a_i^\geq(x), \quad \forall \; x \in (\bc{D}^i \cap [l_x,u_x]_i),
\end{align*}
 to be provided to instantiate \SBLM and by extension \tool. We note that the bounds $c_i$ and $d_i$ of the bounding regions can depend on the concrete input bounds $l_x$ and $u_x$ and the slope $a$ and intercept $b$ of $ a_i^{\{\leq,\geq\}}$ can in turn depend on the corresponding concrete interval bounds $[\max(l_x,c_i),\min(u_x,d_i)]$.

\paragraph{Generalization}
While we focus on the univariate case using only two bounding regions $\bc{D}^1$ and $\bc{D}^2$ in the following, \SBLM and by extension \tool can be generalized to allow for neuron groups combining different multivariate activation functions $f\colon \bb{D} \subseteq \R^d \rightarrow \R$.
Further, more than one upper- and lower-bound $ \bc{B}^i$ per bounding region can be provided and $\bc{D}^i$ can be specified as polyhedral regions instead of as intervals, as long as their union covers the domain $\bb{D} \subseteq \bigcup_i \bc{D}^i$ of the individual functions $f$.

\subsection{Splitting the Input Polytope}
To apply the bounds $\bc{B}^i$, the input polytope $\bc{P}$ has to be split into the regions for which the bounds were specified.
These regions correspond to the intersection of $\bc{P}$ with the k-Cartesian product of the bounding regions $\bc{D}^i$, that is all combinations of neuron-wise bounding regions for the group of k neurons.
We choose an ordering of the output variables $\bc{I}$ and recursively split $\bc{P}$ by intersecting with the bounding regions associated with these output variables.

As every such split is equivalent on an abstract level, we will explain one case assuming the parent polytope $\bc{P}_1$, the output variable $y_j = f(x_j)$, and the corresponding bounding regions $\bc{D}^1_j = \{\bm{x} \in \R^k \,|\, x_j \geq c_1\}$ and $\bc{D}^2_j = \{\bm{x} \in \R^k \,|\, x_j \leq d_2\}$.
We compute the children nodes by intersecting $\bc{P}_1$ with $\bc{D}^1_j$ and $\bc{D}^2_j$ to obtain $\bc{P}_{1,1} = \bc{P}_1 \cap \bc{D}^1_j$ and $\bc{P}_{1,2} = \bc{P}_1 \cap \bc{D}^2_j$.
Starting with $\bc{P}$ at the root and recursively applying this splitting rule for every $y_j \in \bc{I}$, generates a polytope tree, which we call the decomposition tree, with $2^k$ leaf polytopes $\bc{P}_{\{1,2\}^k}$, which we call \emph{quadrants}. This is illustrated in the blue portion of the central panel in~\Figref{fig:overview}, where $\bc{D}^1$ and $\bc{D}^2$ are $\R_0^+$ and $\R_0^-$, respectively.

\subsection{Lifting}
We now extend these quadrants $\bc{P}_{\{1,2\}^k}$ to the output space and bound them using the corresponding constraints on the activation function $\bc{B}^i_j$, before taking their convex hull. This yields a polytope $\bc{K}$, jointly constraining the inputs and outputs of a neuron group. The constraints of its \hrep form the desired k-neuron abstraction.
We call this process \emph{lifting} and propose a recursive approach: We lift sibling polytopes on the decomposition tree until only the desired polytope $\bc{K}$ remains.

Again, we explain a single step of lifting, as they are equivalent.
We assume the sibling polytopes $\bc{K}_{1,1}$ and $\bc{K}_{1,2}$, corresponding to $\bc{P}_{1,1}$ and $\bc{P}_{1,2}$ in the decomposition tree, with the associated input- and output-variables $x_j$ and $y_j$, respectively, and the pairs of bounds $\bc{B}^1_j$ and $\bc{B}^2_j$ instantiated for $y_j$.
A single step consist of three parts:
\begin{itemize}
	\item extending $\bc{K}_{1,1}$ and $\bc{K}_{1,2}$ by the output variable $y_j$,
	\item bounding $y_j$ on the extended polytopes, by intersecting them with the constraints $\bc{B}^1_j$ and $\bc{B}^2_j$ to obtain $\bc{K}_{1,1}'$ and $\bc{K}_{1,2}'$,
	\item computing their (approximate) convex hull using the \PDDM: $\bc{K}_{1} = \conv(\bc{K}_{1,1}',\bc{K}_{1,2}')$.
\end{itemize}
Applying this lifting rule recursively to the decomposition tree starting with $\bc{K}_{\{1,2\}^k} = \bc{P}_{\{1,2\}^k}$, combines all $2^k$ quadrants into a single $2k$-dimensional polytope $\bc{K}$, jointly constraining the inputs and outputs, thereby concluding the \SBLMl. This is illustrated in the right portion of the central panel in~\Figref{fig:overview}.
The decompositional approach has two benefits:
Precision -- computing approximate convex hulls via the \PDDM is exact for polytopes of dimension up to $3$ and starts to lose precision only slowly as dimensionality increases.
Directly computing $2k$-dimensional convex hulls with \PDDM will therefore lose more precision than using our decomposed method.
Speed -- a lower-dimensional polytope with fewer constraints and generally also fewer vertices significantly reduces the runtime for the individual convex hull operations.
In fact, computing the convex hulls for the approximation of non-piecewise-linear functions directly in the input-output space is intractable even for groups of only size $k=3$, as the number of vertices increases exponentially with $k$ during the extension and bounding process in that case.

\subsection{Instantiation for Various Functions}
\begin{wrapfigure}[8]{r}{0.5 \textwidth}
	\centering
	\vspace{-10mm}
	\scalebox{0.87}{\input{figures/sigmoid}}
	\vspace{-3mm}
	\captionof{figure}{Interval-wise bounds for the Sigmoid function on the intervals $[l_x,c]$ and $[c,u_x]$. %
	}
	\label{fig:sig_bounds}
\end{wrapfigure}
We instantiate \SBLM for common network functions next.

\paragraph{ReLU}
We can capture all univariate, piecewise-linear functions, such as ReLU, exactly on the intervals $\bc{D}^i$ where they are linear.
Further, if the neuron-wise bounds $[l_x,u_x]$ only contain one such linear region, the neuron behaves linearly, can be encoded exactly and is excluded from the k-neuron abstraction.
Therefore, we consider $y=max(x,0)$ with $x \in [l_x,u_x]$ for $l_x < 0 < u_x$.
We choose $\bc{D}^1 = [-\infty,0]$ and $\bc{D}^2 = [0,\infty]$, with $\bc{B}^1 = (y\geq 0 , \; y \leq 0)$ and $\bc{B}^2 =(y\geq x , \; y \leq x)$, obtaining exact bounds on both intervals.

\paragraph{Tanh and Sigmoid}
Let $f$ be an S-curve function with domain $[l_x,u_x]$, that is $f''(x) \geq 0$ for $x\leq0$, $f''(x) \leq 0$ for $x\geq0$ and $f'(x)>0$ for $x \in [l_x,u_x]$.
Both Sigmoid $\sigma(x) = \frac{e^x}{e^x+1}$ and Tanh $\tanh(x)=\frac{e^x-e^{-x}}{e^x+e^{-x}}$ have these properties.
We split the domain at $c \in [l_x, u_x]$ into $\bc{D}^1 = [-\infty, c]$ and $\bc{D}^2 = [c, \infty]$, choosing $c$ to minimize the area between upper and lower bound in the input-output plane,
using the bounds from \citet{singh2019abstract}:
\begin{alignat*}{3}
	f(x) &\leq\,& a^{\leq} = &\, f(u_d) + (x-u_d)& \begin{cases}
		\frac{f(u_d)-f(l_d)}{u_d-l_d}, & \text{if} \quad u_d \leq 0, \\
		\min(f'(u_d),f'(l_d)), & \text{else},
	\end{cases}\\
	f(x)&\geq\,&  a^{\geq} = &\, f(l_d) + (x-l_d)& \begin{cases}
	\frac{f(u_d)-f(l_d)}{u_d-l_d}, & \text{if} \quad l_d \geq 0, \\
	\min(f'(u_d),f'(l_d)), & \text{else},
\end{cases}
\end{alignat*}
where we denote the lower bound of the intersection $\bc{D}^i \cap [l_x,u_x]$ as $l_d$ and the upper one as $u_d$.
We show these bounds in~\Figref{fig:sig_bounds} for the Sigmoid function and, for illustration purposes, a non-optimal $c$. In practice, we choose $c$ to minimize the area of the abstraction of a single neuron in the input-output plane.

\begin{wrapfigure}[9]{r}{0.49 \textwidth}
	\centering
	\vspace{-7mm}
	\scalebox{0.81}{\input{figures/max_pool}}
	\vspace{-4mm}
	\caption{Polyhedral bounding regions $\bc{D}^i$ and corresponding bounds $\bc{B}^i$ for the $2d$ MaxPool function on the input region $[l_{x_1}, u_{x_1}] \times [l_{x_2}, u_{x_2}]$.
	}
	\label{fig:maxpool}
\end{wrapfigure}
\paragraph{MaxPool}
Let MaxPool be the multivariate function $y=\max(x_1,x_2,...,x_d)$ on the domain $\bm{x} \in \bc{P} \subseteq [l_x,u_x]^d$.
Note that here the generalized formulation is required.
We chose the polyhedral bounding regions $\bc{D}^i = \{\bm{x}\in \R^{d} | x_i \geq x_j, \; 1 \leq j \leq d, \; i \neq j\}_i$, separating the domain into the $d$ regions where one variable dominates all others (illustrated for $d=2$ in \Figref{fig:maxpool}).
On each of these regions, MaxPool can be bounded exactly with $y \leq x_i$ and $y \geq x_i$.
During the splitting process, this increased number of bounding regions leads to a decomposition tree where every parent node has $d$ child nodes.

\section{\tool Verification Framework}\label{sec:framework}

\tool is based on three high-level steps: (i) accumulate a set of constraints encoding a (convex) abstraction of the network for a given pre-condition (as discussed so far), (ii) define a linear optimization objective representing the post-condition, and (iii) use an LP or MILP solver to derive a bound on this optimization objective. If this bound exceeds a threshold depending on the post-condition, certification succeeds, otherwise, if the optimal solution violates this bound, it could be a true counterexample or a false positive due to approximation. Hence, we evaluate any such possible counterexample with the concrete network to determine whether it is a true counterexample.

While all affine layers are encoded exactly, two considerations have to be balanced when encoding non-linear activation layers with \tool: more precise encodings (e.g., considering more or larger neuron groups) improve the optimal bound of the optimization problem, but the increased number of constraints can make this problem impractical to solve. We navigate this trade-off by leveraging abstraction refinement -- using increasingly more precise but also more costly methods until we are able to either decide a property (verify or falsify) or reach a timeout.

\subsection{Abstraction Refinement Approaches}\label{sec:absrefine}

Fundamentally, we can refine our abstraction in three ways: (i) compute tighter abstractions of the group-wise inputs, (ii) compute tighter layer-wise multi-neuron constraints for the given input abstraction from (i), and (iii) encode part of the network using an exact MILP encoding.

\paragraph{Input bound refinement} Since \SBLM and \PDDM abstract a group of neurons for a given polyhedral input region, the tightness of the resulting constraints depends directly on the tightness of the input abstraction. These are computed using a fast, incomplete verifier (e.g., \citep{muller2020neural, Lirpa:20, singh2019abstract}) based on single-neuron abstractions and can be tightened significantly by computing more precise neuron-wise bounds \cite{singh2019boosting} using an LP or MILP encoding.

\paragraph{Tighten multi-neuron constraints}
The layer-wise tightness of our multi-neuron constraints depends on (i) the tightness of the group-wise constraints, mostly determined by the quality of the input region, and (ii) on capturing the important neuron-interdependencies with the chosen groups. Using larger neuron groups (increasing $k$) and considering more groupings by allowing more overlap (increasing $s$) and partitioning the neurons into fewer sets before grouping (increasing $n_s$), allows capturing more and more complex interactions. While the constraints themselves can be computed quickly, the resulting LP problems become harder to solve.

\paragraph{Network encoding} \tool encodes non-linear activations in four different ways: (i) exact encoding via equality constraints for stable (those exhibiting linear behavior) piecewise-linear activations, (ii) single-neuron constraints, (iii) multi-neuron constraints computed via \SBLM and \PDDM, and (iv) exact (for piecewise-linear functions) MILP encodings. While stable activations are always encoded exactly and all unstable activations are encoded using both the single- and multi-neuron constraints, we only selectively use a MILP encoding on the (typically relatively narrow) last layers of convolutional networks due to their large computational cost.

\subsection{Abstraction Refinement Cascade}

\tool leverages our multi-neuron constraints as part of an abstraction refinement cascade using increasingly more precise and expensive approaches:
We first attempt verification using single-neuron constraints via \DeepPoly \cite{singh2019abstract} or \GPUpoly \cite{muller2020neural}. If this fails, we encode all activation layers using our multi-neuron constraints and solve the resulting LP. If this also fails, we attempt to decide the property by tightening the multi-neuron constraints \Secref{sec:absrefine}, encoding the final network layer(s) using MILP, and refining individual neuron bounds.

%% file: figures/polyhedral_figure_2.tex
\begin{tikzpicture}
	\tikzset{>=latex}
	
	\def\bh{5.5cm}
	\def\yl{-2.3}
	\def\xe{5.1}
	\def\dx{2.0}
	\def\xaa{-0.9}
	\def\yaa{-0.6}
	
	\node (A) 
		  [draw=black!60, fill=black!05, rectangle,
		    text width=7.2cm, align=center, scale=1.0, rounded corners=2pt,
		    minimum height=\bh,
		    anchor=north
		  ] at (0, 0){} ;
	\node (A1) [scale=1.0,align=center] at (0.0, -0.22) {Primal};
	\node (A2) [scale=1.0,align=center, rotate=90] at (-3.35, -1.4) {Polytope};
	\node (A3) [scale=1.0,align=center, rotate=90] at (-3.35, \yl-1.7) {Polyhedral cone};

 	\node (B)
		[draw=black!60, fill=black!05, rectangle, rounded corners=2pt,
		minimum width=3.2cm, minimum height=\bh,
		anchor=north,
		] at (3.60+\dx, 0) {
		};
	
	\node (B1) [scale=1.0,align=center] at (3.8+\dx, -0.22) {Dual};
	\node (B2) [scale=1.0,align=center, rotate=90] at (2.3+\dx, -1.4) {Polytope};
	\node (B3) [scale=1.0,align=center, rotate=90] at (2.3+\dx, \yl-1.7) {Polyhedral cone};

  \node [anchor=north west](polyA1) at (-2.5, -0.4) {
	\begin{tikzpicture}[scale=0.6]
		
		\coordinate (n) at (-0.3, 1.2);
		\coordinate (nw) at (-1.0, 0.5);
		\coordinate (e) at (0.5, 0.0);
		\coordinate (sw) at (-1.0, -0.5);
		\coordinate (s) at (-0.3, -1.2);
		\coordinate (w) at (-1.1, 0);
		
		\draw[fill=blue!90, opacity=0.4] (nw) -- (n) -- (e) -- (s) -- (sw) -- (w) -- cycle;
		
		\coordinate (xp) at (1.2, 0, 0);
		\coordinate (xn) at (-1.2, 0, 0);
		\coordinate (yp) at (0, 1.2, 0);
		\coordinate (yn) at (0, -1.2, 0);
		
		\draw[->,opacity=0.8] (xn) -- (xp) ;
		\draw[->,opacity=0.8] (yn) -- (yp) ;
		
		\node (PA1a) [scale=1.0,align=center] at (0.9, 0.50) {\scriptsize$x_1$};
		\node (PA1b) [scale=1.0,align=center] at (-0.1, 1.60) {\scriptsize$x_2$};
		
		\node[anchor=center,circle, black, fill, minimum size=3pt,inner sep=0pt, outer sep=0pt] at (n) {};
		\node[anchor=center,circle, black, fill, minimum size=3pt,inner sep=0pt, outer sep=0pt] at (e) {};
		\node[anchor=center,circle, black, fill, minimum size=3pt,inner sep=0pt, outer sep=0pt] at (w) {};
		\node[anchor=center,circle, black, fill, minimum size=3pt,inner sep=0pt, outer sep=0pt] at (s) {};		
		\node[anchor=center,circle, black, fill, minimum size=3pt,inner sep=0pt, outer sep=0pt] at (sw) {};		
		\node[anchor=center,circle, black, fill, minimum size=3pt,inner sep=0pt, outer sep=0pt] at (nw) {};	
		
		\node (PA1b) [scale=1.0,align=center] at (0.3, 1.10) {\scriptsize$-1.5x_1-x_2 \geq -0.75$};
		\node (PA1b) [scale=1.0,align=center] at (0.3, -0.20) {\scriptsize$-1.5x_1+x_2 \geq -0.75$};
		\node (PA1b) [scale=1.0,align=center] at (-4.5, -0.80) {\scriptsize$-x_1+x_2 \geq -1.5$};
		\node (PA1b) [scale=1.0,align=center] at (-4.2, 1.30) {\scriptsize$x_1-x_2 \geq -1.5$};
		\node (PA1b) [scale=1.0,align=center] at (-4.9, 0.60) {\scriptsize$5x_1-x_2 \geq -5.5$};
		\node (PA1b) [scale=1.0,align=center] at (-4.9, 0.00) {\scriptsize$5x_1+x_2 \geq -5.5$};
	\end{tikzpicture}
};

  \node [anchor=north west ](polyA2) at (-3.25, \yl-0.3) {
  	\begin{tikzpicture}[scale=0.6]
  		
  		\coordinate (O) at (0, 0, 0);
		\coordinate (n) at (1.7*-0.3, 1.7*1.2, -1.7);
		\coordinate (nw) at (1.7*-1.0, 1.7*0.5, -1.7);
		\coordinate (e) at (1.7*0.5, 1.7*0.0, -1.7);
		\coordinate (sw) at (1.7*-1.0, 1.7*-0.5, -1.7);
		\coordinate (s) at (1.7*-0.3, 1.7*-1.2, -1.7);
		\coordinate (w) at (1.7*-1.1, 1.7* 0, -1.7);

		\coordinate (n1) at (1.0*-0.3, 1.0*1.2, -1.0);
		\coordinate (nw1) at (1.0*-1.0, 1.0*0.5, -1.0);
		\coordinate (e1) at (1.0*0.5, 1.0*0.0, -1.0);
		\coordinate (sw1) at (1.0*-1.0, 1.0*-0.5, -1.0);
		\coordinate (s1) at (1.0*-0.3, 1.0*-1.2, -1.0);
		\coordinate (w1) at (1.0*-1.1, 1.0* 0, -1.0);
				
		\coordinate (p1) at (0.6,1.4,-1);
		\coordinate (p2) at (-1.2,1.4,-1);
		\coordinate (p3) at (0.6,-1.4,-1);
		\coordinate (p4) at (-1.2,-1.4,-1);

		\draw[fill=blue!90, opacity=0.4] (nw) -- (n) -- (e) -- (s) -- (sw) -- (w) -- cycle;

  		\draw[fill=blue!90, opacity=0.2] (nw) -- (n) -- (n1) -- (nw1) -- cycle;
		\draw[fill=blue!90, opacity=0.2] (n) -- (e) -- (e1) -- (n1) -- cycle;
		\draw[fill=blue!90, opacity=0.2] (e) -- (s) -- (s1) -- (e1) -- cycle;
		\draw[fill=blue!90, opacity=0.2] (s) -- (sw) -- (sw1) -- (s1) -- cycle;
		\draw[fill=blue!90, opacity=0.2] (sw) -- (w) -- (w1) -- (sw1) -- cycle;
		\draw[fill=blue!90, opacity=0.2] (w) -- (nw) -- (nw1) -- (w1) -- cycle;

  		\draw[black, opacity=0.6] (nw) -- (nw1);
		\draw[black, opacity=0.6] (n) -- (n1);
		\draw[black, opacity=0.6] (e) -- (e1);
		\draw[black, opacity=0.6] (s) -- (s1);
		\draw[black, opacity=0.6] (sw) -- (sw1);
		\draw[black, opacity=0.6] (w) -- (w1);

		\draw[fill=red!90, opacity=0.4] (p1) -- (p2) -- (p4) -- (p3) -- cycle;

  		\draw[fill=blue!90, opacity=0.4] (O) -- (nw1) -- (n1) -- cycle;
		\draw[fill=blue!90, opacity=0.4] (O) -- (n1) -- (e1) -- cycle;
		\draw[fill=blue!90, opacity=0.4] (O) -- (e1) -- (s1) -- cycle;
		\draw[fill=blue!90, opacity=0.4] (O) -- (s1) -- (sw1) -- cycle;
		\draw[fill=blue!90, opacity=0.4] (O) -- (sw1) -- (w1) -- cycle;
		\draw[fill=blue!90, opacity=0.4] (O) -- (w1) -- (nw1) -- cycle;
		
  		\draw[black, opacity=0.6] (O) -- (nw1);
		\draw[black, opacity=0.6] (O) -- (n1);
		\draw[black, opacity=0.6] (O) -- (e1);
		\draw[black, opacity=0.6] (O) -- (s1);
		\draw[black, opacity=0.6] (O) -- (sw1);
		\draw[black, opacity=0.6] (O) -- (w1);
  			
		\coordinate (xp) at (1.2, 0, 0);
		\coordinate (xn) at (-1.2, 0, 0);
		\coordinate (yp) at (0, 1.2, 0);
		\coordinate (yn) at (0, -1.2, 0);
		\coordinate (zp) at (0, 0, 1.2);
		\coordinate (zn) at (0, 0, -1.2);				
		
		\draw[-,opacity=0.3] (xn) -- (xp) ;
		\draw[-,opacity=0.3] (yn) -- (yp) ;
		\draw[-,opacity=0.3] (zn) -- (zp) ;
		
		\node (PA1b) [scale=1.0,align=center] at (0.2, 2.90) {\scriptsize$0.75x_0 -1.5x_1-x_2 \geq 0$};
		\node (PA1b) [scale=1.0,align=center] at (0.2, -1.00) {\scriptsize$0.75x_0 -1.5x_1+x_2 \geq 0$};
		\node (PA1b) [scale=1.0,align=center] at (-4.6, -1.00) {\scriptsize$1.5x_0-x_1+x_2 \geq 0$};
		\node (PA1b) [scale=1.0,align=center] at (-4.8, 2.70) {\scriptsize$1.5x_0 + x_1-x_2 \geq 0$};
		\node (PA1b) [scale=1.0,align=center] at (-5.8, 1.30) {\scriptsize$5.5x_0 + 5x_1-x_2 \geq 0$};
		\node (PA1b) [scale=1.0,align=center] at (-5.8, 0.50) {\scriptsize$5.5x_0 + 5x_1+x_2 \geq 0$};
		
		\node (PA1b) [scale=1.0,align=center] at (0.9, -0.20) {\scriptsize$x_0 = 1$};

  	\end{tikzpicture}
  };

\node [anchor=north east](polyB1) at (\xe+\dx-0.1, -0.55) {
	\begin{tikzpicture}[scale=0.6]
		
		\coordinate (ne) at (2.0, 1.6/1.2);
		\coordinate (se) at (2, -1.6/1.2);
		\coordinate (n) at (-1/1.5, 1/1.5);
		\coordinate (s) at (-1/1.5, -1/1.5);
		\coordinate (nw) at (-1/1.1, 0.2/1.1);
		\coordinate (sw) at (-1/1.1, -0.2/1.1);
		
		\draw[fill=blue!90, opacity=0.4] (ne) -- (se) -- (s) -- (sw) -- (nw) -- (n) -- cycle;
		
		\coordinate (xp) at (1.2, 0);
		\coordinate (xn) at (-1.2, 0);
		\coordinate (yp) at (0, 1.2);
		\coordinate (yn) at (0, -1.2);
		
		\draw[-,opacity=0.3] (xn) -- (xp) ;
		\draw[-,opacity=0.3] (yn) -- (yp) ;
		
		\node[anchor=center,circle, black, fill, minimum size=3pt,inner sep=0pt, outer sep=0pt] at (n) {};
		\node[anchor=center,circle, black, fill, minimum size=3pt,inner sep=0pt, outer sep=0pt] at (se) {};
		\node[anchor=center,circle, black, fill, minimum size=3pt,inner sep=0pt, outer sep=0pt] at (ne) {};
		\node[anchor=center,circle, black, fill, minimum size=3pt,inner sep=0pt, outer sep=0pt] at (s) {};		
		\node[anchor=center,circle, black, fill, minimum size=3pt,inner sep=0pt, outer sep=0pt] at (sw) {};		
		\node[anchor=center,circle, black, fill, minimum size=3pt,inner sep=0pt, outer sep=0pt] at (nw) {};	
		
	\end{tikzpicture}
};

\node [anchor=north east ](polyB2) at (\xe+\dx, \yl-0.6) {
	\begin{tikzpicture}[scale=0.6]

  		\coordinate (O) at (0, 0, 0);
		\coordinate (ne) at (1.3*2.0, 1.3*1.6/1.2, 1.3);
		\coordinate (se) at (1.3*2, 1.3*-1.6/1.2, 1.3);
		\coordinate (n) at (1.3*-1/1.5, 1.3*1/1.5, 1.3);
		\coordinate (s) at (1.3*-1/1.5, 1.3*-1/1.5, 1.3);
		\coordinate (nw) at (1.3*-1/1.1, 1.3*0.2/1.1, 1.3);
		\coordinate (sw) at (1.3*-1/1.1, 1.3*-0.2/1.1, 1.3);

		\coordinate (ne1) at (1.0*2.0, 1.0*1.6/1.2, 1.0);
		\coordinate (se1) at (1.0*2, 1.0*-1.6/1.2, 1.0);
		\coordinate (n1) at (1.0*-1/1.5, 1.0*1/1.5, 1.0);
		\coordinate (s1) at (1.0*-1/1.5, 1.0*-1/1.5, 1.0);
		\coordinate (nw1) at (1.0*-1/1.1, 1.0*0.2/1.1, 1.0);
		\coordinate (sw1) at (1.0*-1/1.1, 1.0*-0.2/1.1, 1.0);
		
		\coordinate (p1) at (2.3,1.4,1);
		\coordinate (p2) at (-1.0,1.4,1);
		\coordinate (p3) at (2.3,-1.4,1);
		\coordinate (p4) at (-1.0,-1.4,1);

		\draw[fill=blue!90, opacity=0.3] (ne) -- (se) -- (s) -- (sw) -- (nw) -- (n) -- cycle;
		
		\draw[fill=blue!90, opacity=0.3] (O) -- (nw1) -- (n1) -- cycle;
		\draw[fill=blue!90, opacity=0.3] (O) -- (n1) -- (ne1) -- cycle;
		\draw[fill=blue!90, opacity=0.3] (O) -- (ne1) -- (se1) -- cycle;
		\draw[fill=blue!90, opacity=0.3] (O) -- (s1) -- (sw1) -- cycle;
		\draw[fill=blue!90, opacity=0.3] (O) -- (sw1) -- (nw1) -- cycle;
		\draw[fill=blue!90, opacity=0.3] (O) -- (se1) -- (s1) -- cycle;

  		\draw[black, opacity=0.6] (O) -- (nw1);
		\draw[black, opacity=0.6] (O) -- (ne1);	
		\draw[black, opacity=0.6] (O) -- (s1);
		\draw[black, opacity=0.6] (O) -- (se1);
		\draw[black, opacity=0.6] (O) -- (sw1);
		\draw[black, opacity=0.6] (O) -- (n1);	

		\draw[fill=red!90, opacity=0.4] (p1) -- (p2) -- (p4) -- (p3) -- cycle;

		\draw[fill=blue!90, opacity=0.4] (nw) -- (n) -- (n1) -- (nw1) -- cycle;
		\draw[fill=blue!90, opacity=0.4] (n) -- (ne) -- (ne1) -- (n1) -- cycle;
		\draw[fill=blue!90, opacity=0.4] (ne) -- (se) -- (se1) -- (ne1) -- cycle;
		\draw[fill=blue!90, opacity=0.4] (s) -- (se) -- (se1) -- (s1) -- cycle;
		\draw[fill=blue!90, opacity=0.4] (s) -- (sw) -- (sw1) -- (s1) -- cycle;
		\draw[fill=blue!90, opacity=0.4] (sw) -- (nw) -- (nw1) -- (sw1) -- cycle;
		
		\draw[black, opacity=0.6] (nw) -- (nw1);
		\draw[black, opacity=0.6] (ne) -- (ne1);
		\draw[black, opacity=0.6] (s) -- (s1);
		\draw[black, opacity=0.6] (se) -- (se1);
		\draw[black, opacity=0.6] (sw) -- (sw1);
		\draw[black, opacity=0.6] (n) -- (n1);		
		
		\coordinate (xp) at (1.2, 0, 0);
		\coordinate (xn) at (-1.2, 0, 0);
		\coordinate (yp) at (0, 1.2, 0);
		\coordinate (yn) at (0, -1.2, 0);
		\coordinate (zp) at (0, 0, 1.2);
		\coordinate (zn) at (0, 0, -1.2);				
		
		\draw[-,opacity=0.3] (xn) -- (xp) ;
		\draw[-,opacity=0.3] (yn) -- (yp) ;
		\draw[-,opacity=0.3] (zn) -- (zp) ;
		
	\end{tikzpicture}
	};

\end{tikzpicture}

%% file: figures/PDD_mult_const_fig.tex
\begin{tikzpicture}
  \tikzset{>=latex}

  \def\xw{0}
  \def\xe{3.7}
  \def\dx{2.4}
  \def\yn{0}
  \def\ys{-4}

	\node (Constraints) at (\xw, \yn) {
		\begin{tikzpicture}[scale=0.7]
			
  		\coordinate (e) at (00:2);
		\coordinate (se) at (60:2);
		\coordinate (sw) at (120:2);
		\coordinate (w) at (180:2);
		\coordinate (nw) at (240:2);
		\coordinate (ne) at (300:2);
			
		\coordinate (n) at ($(ne)!0.4!(e)$);
		\coordinate (s) at ($(sw)!0.4!(w)$);
		\coordinate (p) at ($(n)!1.07!(s)$);
		\coordinate (pp) at ($(p)!0.2!270:(n)$);
		
		\coordinate (n2) at ($(ne)!0.8!(e)$);
		\coordinate (s2) at ($(sw)!0.8!(se)$);
		\coordinate (p2) at ($(n2)!1.2!(s2)$);
		\coordinate (pp2) at ($(p2)!0.2!270:(n2)$);
		
		\coordinate (n3) at ($(ne)!0.7!(nw)$);
		\coordinate (s3) at ($(sw)!0.3!(se)$);
		\coordinate (p3) at ($(n3)!1.07!(s3)$);
		\coordinate (pp3) at ($(p3)!0.2!270:(n3)$);

		\draw[fill=green!90,opacity=0.15] (w) -- (s) -- (intersection of n--s and n3--s3) -- (n3) -- (nw) --cycle;
  		\draw[fill=green!90,opacity=0.15] (sw) -- (s) --(intersection of n--s and n3--s3) -- (n3) -- (nw) -- (ne) -- (e) --(se)--cycle;
  		\draw[opacity=0.5,name path=f1] (e) -- (se) -- (sw) -- (w) -- (nw) -- (ne) --cycle;

		\node[circle, fill=red!90, minimum size=3pt,inner sep=0pt, outer sep=0pt] at (ne) {};
		\node[circle, fill=red!90, minimum size=3pt,inner sep=0pt, outer sep=0pt] at (sw) {};
		\node[circle, fill=red!90, minimum size=3pt,inner sep=0pt, outer sep=0pt] at (se) {};
		\node[circle, fill=red!90, minimum size=3pt,inner sep=0pt, outer sep=0pt] at (e) {};
		\node[circle, fill=black, minimum size=3pt,inner sep=0pt, outer sep=0pt] at (w) {};
		\node[circle, fill=black, minimum size=3pt,inner sep=0pt, outer sep=0pt] at (nw) {};
  		
  		\draw[-,blue!80!black,opacity=0.8, shorten >= -0.3cm, shorten <= -0.3cm, name path=i1] (n) -- (s);
  		\draw[->,blue!80!black,opacity=0.8, thick] (p) -- (pp);
  		\draw[-,blue!80!black,opacity=0.8, shorten >= -0.5cm, shorten <= -0.3cm, name path=i2] (n2) -- (s2);
  		\draw[->,blue!80!black,opacity=0.8, thick] (p2) -- (pp2);
  		\draw[-,blue!80!black,opacity=0.8, shorten >= -0.3cm, shorten <= -0.3cm, name path=i3] (n3) -- (s3);
  		\draw[->,blue!80!black,opacity=0.8, thick] (p3) -- (pp3);
  		
  		\draw[dashed,black!90,opacity=0.6, name path=p1] (w) -- (e);
 		\draw[dashed,black!90,opacity=0.6, name path=p2] (w) -- (se);
 		\draw[dashed,black!90,opacity=0.6, name path=p3] (w) -- (ne);
 		\draw[dashed,black!90,opacity=0.6, name path=p4] (nw) -- (e);
 		\draw[dashed,black!90,opacity=0.6, name path=p5] (nw) -- (sw);
 		\draw[dashed,black!90,opacity=0.6, name path=p6] (nw) -- (se);
 		\fill[name intersections={of=i3 and p1}]
				(intersection-1) [blue!80!black,opacity=0.6] circle(2.2pt) node {};
		\fill[name intersections={of=i1 and p2}]
				(intersection-1) [blue!80!black,opacity=0.6] circle(2.2pt) node {};
		\fill[name intersections={of=i3 and p3}]
				(intersection-1) [blue!80!black,opacity=0.6] circle(2.2pt) node {};
		\fill[name intersections={of=i3 and p4}]
				(intersection-1) [blue!80!black,opacity=0.6] circle(2.2pt) node {};
		\fill[name intersections={of=i1 and p5}]
				(intersection-1) [blue!80!black,opacity=0.6] circle(2.2pt) node {};
		\fill[name intersections={of=i3 and p6}]
				(intersection-1) [blue!80!black,opacity=0.6] circle(2.2pt) node {};

		\fill[name intersections={of=i1 and p1}]
			(intersection-1) [green!60!black,opacity=1] circle(2.2pt) node {};
		\fill[name intersections={of=i1 and p4}]
			(intersection-1) [green!60!black,opacity=1] circle(2.2pt) node {};
			
		\fill[name intersections={of=i3 and p2}]
			(intersection-1) [green!60!black,opacity=1] circle(2.2pt) node {};
	
		\fill[name intersections={of=i2 and p1}]
			(intersection-1) [green!60!black,opacity=1] circle(2.2pt) node {};
		\fill[name intersections={of=i2 and p2}]
			(intersection-1) [green!60!black,opacity=1] circle(2.2pt) node {};
		\fill[name intersections={of=i2 and p6}]
			(intersection-1) [green!60!black,opacity=1] circle(2.2pt) node {};
		\fill[name intersections={of=i2 and p4}]
			(intersection-1) [green!60!black,opacity=1] circle(2.2pt) node {};
			
		\fill[name intersections={of=i3 and i1}]
			(intersection-1) [yellow!100!black,opacity=1] circle(3.2pt) node {};

		\node[circle, fill=blue!80!black,opacity=0.6, minimum size=3pt,inner sep=0pt, outer sep=0pt] at (n3) {};
		\node[circle, fill=blue!80!black,opacity=0.6, minimum size=3pt,inner sep=0pt, outer sep=0pt] at (s) {};
			
		\end{tikzpicture}
	};

  \node [scale=1.0,align=center, anchor=north] at (\xw, {-1.6}) {\scriptsize (a) Adding Multiple Constraints};

	\node (Constraints) at (\xe, \yn) {
	\begin{tikzpicture}[scale=0.7]
		
		\coordinate (e) at (00:2);
		\coordinate (se) at (60:2);
		\coordinate (sw) at (120:2);
		\coordinate (w) at (180:2);
		\coordinate (nw) at (240:2);
		\coordinate (ne) at (300:2);
		
		\coordinate (n) at ($(ne)!0.4!(e)$);
		\coordinate (s) at ($(sw)!0.4!(w)$);
		\coordinate (p) at ($(n)!1.1!(s)$);
		\coordinate (pp) at ($(p)!0.2!270:(n)$);
		
		\coordinate (n2) at ($(ne)!0.8!(e)$);
		\coordinate (s2) at ($(sw)!0.8!(se)$);
		\coordinate (p2) at ($(n2)!1.2!(s2)$);
		\coordinate (pp2) at ($(p2)!0.2!270:(n2)$);
		
		\coordinate (n3) at ($(ne)!0.7!(nw)$);
		\coordinate (s3) at ($(sw)!0.3!(se)$);
		\coordinate (p3) at ($(n3)!1.2!(s3)$);
		\coordinate (pp3) at ($(p3)!0.2!270:(n3)$);

		\draw[fill=green!90,opacity=0.15] (w) -- (s) -- (intersection of n--s and n3--s3) -- (n3) -- (nw) --cycle;

		\node[circle, fill=black, minimum size=3pt,inner sep=0pt, outer sep=0pt] at (w) {};
		\node[circle, fill=black, minimum size=3pt,inner sep=0pt, outer sep=0pt] at (nw) {};
		
		\path[-,blue!80!black,opacity=0.6, shorten >= -0.5cm, shorten <= 0.5cm, name path=i1] (n) -- (s);
		\path[-,blue!80!black,opacity=0.6, shorten >= -0.5cm, shorten <= 0.5cm, name path=i2] (n2) -- (s2);

		\node[circle, fill=black, minimum size=3pt,inner sep=0pt, outer sep=0pt] at (n3) {};
		\node[circle, fill=black, minimum size=3pt,inner sep=0pt, outer sep=0pt] at (s) {};
		\node[circle, fill=black, minimum size=3pt,inner sep=0pt, outer sep=0pt] at (intersection of n--s and n3--s3) {};
	\end{tikzpicture}
};

  \node [scale=1.0,align=center, anchor=north] at (\xe-0.5, {-1.6}) {\scriptsize (b) Exact Result};

	\node (Constraints) at (\xe+\dx, \yn) {
	\begin{tikzpicture}[scale=0.7]
		
		\coordinate (e) at (00:2);
		\coordinate (se) at (60:2);
		\coordinate (sw) at (120:2);
		\coordinate (w) at (180:2);
		\coordinate (nw) at (240:2);
		\coordinate (ne) at (300:2);
		
		\coordinate (n) at ($(ne)!0.4!(e)$);
		\coordinate (s) at ($(sw)!0.4!(w)$);
		\coordinate (p) at ($(n)!1.1!(s)$);
		\coordinate (pp) at ($(p)!0.2!270:(n)$);
		
		\coordinate (n2) at ($(ne)!0.8!(e)$);
		\coordinate (s2) at ($(sw)!0.8!(se)$);
		\coordinate (p2) at ($(n2)!1.2!(s2)$);
		\coordinate (pp2) at ($(p2)!0.2!270:(n2)$);
		
		\coordinate (n3) at ($(ne)!0.7!(nw)$);
		\coordinate (s3) at ($(sw)!0.3!(se)$);
		\coordinate (p3) at ($(n3)!1.2!(s3)$);
		\coordinate (pp3) at ($(p3)!0.2!270:(n3)$);

		\draw[fill=green!90,opacity=0.15] (w) -- (s) -- (intersection of n--s and w--se) -- (intersection of n3--s3 and w--e) -- (n3) -- (nw) --cycle;
		\draw[fill=blue!80!black,opacity=0.6,opacity=0.15] (intersection of n--s and w--se) -- (intersection of n--s and n3--s3) -- (intersection of n3--s3 and w--e) --cycle;

		\node[circle, fill=black, minimum size=3pt,inner sep=0pt, outer sep=0pt] at (w) {};
		\node[circle, fill=black, minimum size=3pt,inner sep=0pt, outer sep=0pt] at (nw) {};
		
		\path[-,blue!80!black,opacity=0.6, shorten >= -0.5cm, shorten <= 0.5cm, name path=i1] (n) -- (s);
		\path[-,blue!80!black,opacity=0.6, shorten >= -0.5cm, shorten <= 0.5cm, name path=i2] (n2) -- (s2);
		\path[-,black, shorten >= -0.5cm, shorten <= -0.5cm, name path=i3] (n3) -- (s3);
		
		\path[dashed,black!90,opacity=0.6, name path=p1] (w) -- (e);
		\path[dashed,black!90,opacity=0.6, name path=p2] (w) -- (se);
		\path[dashed,black!90,opacity=0.6, name path=p3] (w) -- (ne);
		\path[dashed,black!90,opacity=0.6, name path=p4] (nw) -- (e);
		\path[dashed,black!90,opacity=0.6, name path=p5] (nw) -- (sw);
		\path[dashed,black!90,opacity=0.6, name path=p6] (nw) -- (se);
		\fill[name intersections={of=i3 and p1}]
		(intersection-1) [black] circle(2.2pt) node {};
		\fill[name intersections={of=i1 and p2}]
		(intersection-1) [black] circle(2.2pt) node {};
		\fill[name intersections={of=i3 and p3}]
		(intersection-1) [black] circle(2.2pt) node {};
		\fill[name intersections={of=i3 and p4}]
		(intersection-1) [black] circle(2.2pt) node {};
		\fill[name intersections={of=i1 and p5}]
		(intersection-1) [black] circle(2.2pt) node {};
		\fill[name intersections={of=i3 and p6}]
		(intersection-1) [black] circle(2.2pt) node {};

		\node[circle, fill=black, minimum size=3pt,inner sep=0pt, outer sep=0pt] at (n3) {};
		\node[circle, fill=black, minimum size=3pt,inner sep=0pt, outer sep=0pt] at (s) {};
	\end{tikzpicture}
};

  \node [scale=1.0,align=center, anchor=north] at (\xe-0.5+\dx, {-1.6}) {\scriptsize (c) Discovered Vertices};

	\node (Constraints) at (\xe+2*\dx, \yn) {
	\begin{tikzpicture}[scale=0.7]
		
		\coordinate (e) at (00:2);
		\coordinate (se) at (60:2);
		\coordinate (sw) at (120:2);
		\coordinate (w) at (180:2);
		\coordinate (nw) at (240:2);
		\coordinate (ne) at (300:2);
		
		\coordinate (n) at ($(ne)!0.4!(e)$);
		\coordinate (s) at ($(sw)!0.4!(w)$);
		\coordinate (p) at ($(n)!1.1!(s)$);
		\coordinate (pp) at ($(p)!0.2!270:(n)$);
		
		\coordinate (n2) at ($(ne)!0.8!(e)$);
		\coordinate (s2) at ($(sw)!0.8!(se)$);
		\coordinate (p2) at ($(n2)!1.2!(s2)$);
		\coordinate (pp2) at ($(p2)!0.2!270:(n2)$);
		
		\coordinate (n3) at ($(ne)!0.7!(nw)$);
		\coordinate (s3) at ($(sw)!0.3!(se)$);
		\coordinate (p3) at ($(n3)!1.2!(s3)$);
		\coordinate (pp3) at ($(p3)!0.2!270:(n3)$);

		\draw[fill=green!90,opacity=0.15] (w) -- (s) -- (n3) -- (nw) --cycle;
		\draw[fill=blue!80!black,opacity=0.6,opacity=0.15] (s) -- (intersection of n--s and n3--s3) -- (n3) --cycle;
		
		\node[circle, fill=black, minimum size=3pt,inner sep=0pt, outer sep=0pt] at (w) {};
		\node[circle, fill=black, minimum size=3pt,inner sep=0pt, outer sep=0pt] at (nw) {};
		
		\path[-,blue!80!black,opacity=0.6, shorten >= -0.5cm, shorten <= 0.5cm, name path=i1] (n) -- (s);
		\path[-,blue!80!black,opacity=0.6, shorten >= -0.5cm, shorten <= 0.5cm, name path=i2] (n2) -- (s2);
		\path[-,black, shorten >= -0.5cm, shorten <= -0.5cm, name path=i3] (n3) -- (s3);
		
		\path[dashed,black!90,opacity=0.6, name path=p1] (w) -- (e);
		\path[dashed,black!90,opacity=0.6, name path=p2] (w) -- (se);
		\path[dashed,black!90,opacity=0.6, name path=p3] (w) -- (ne);
		\path[dashed,black!90,opacity=0.6, name path=p4] (nw) -- (e);
		\path[dashed,black!90,opacity=0.6, name path=p5] (nw) -- (sw);
		\path[dashed,black!90,opacity=0.6, name path=p6] (nw) -- (se);
		
		\node[circle, fill=black, minimum size=3pt,inner sep=0pt, outer sep=0pt] at (n3) {};
		\node[circle, fill=black, minimum size=3pt,inner sep=0pt, outer sep=0pt] at (s) {};
	\end{tikzpicture}
};
  \node [scale=1.0,align=center, anchor=north] at ({\xe-0.5+2*\dx}, {-1.6}) {\scriptsize (d) A-Irredundant};
\end{tikzpicture}

%% file: figures/PDD_inter_fig.tex
\begin{tikzpicture}
  \tikzset{>=latex}

  \def\xw{0}
  \def\xe{3.3}
  \def\dx{2.4}
  \def\yn{0}
  \def\ys{-4.2}

	\node (Constraints) at (\xw, \yn) {
		\begin{tikzpicture}[scale=0.6]
			
  		\coordinate (e) at (00:2);
		\coordinate (ne) at (60:2);
		\coordinate (nw) at (120:2);
		\coordinate (w) at (180:2);
		\coordinate (sw) at (240:2);
		\coordinate (se) at (300:2);
		
  		\coordinate (ne1) at ($(30:2)+(1.2,-2)$);
		\coordinate (n1) at ($(90:2)+(1.2,-2)$);
		\coordinate (nw1) at ($(150:2)+(1.2,-2)$);
		\coordinate (sw1) at ($(210:2)+(1.2,-2)$);
		\coordinate (s1) at ($(270:2)+(1.2,-2)$);
		\coordinate (se1) at ($(330:2)+(1.2,-2)$);
			
		\coordinate (p) at ($(ne)!0.5!(e)$);
		\coordinate (pp) at ($(nw)!0.5!(w)$);
		
		\coordinate (p1) at ($(n1)!0.5!(ne1)$);			
		\coordinate (pp1) at ($(nw1)!0.7!(sw1)$);			

  		\draw[fill=green!90,opacity=0.3] (sw) -- (se) -- (p) -- (pp) --cycle;
  		\draw[fill=red!90,opacity=0.3] (pp1) -- (nw1) -- (n1) --  (p1) --cycle;
  		\draw[opacity=0.5,name path=f1] (e) -- (se) -- (sw) -- (w) -- (nw) -- (ne) --cycle;
  		\draw[opacity=0.5,name path=f1] (ne1) -- (n1) -- (nw1) -- (sw1) -- (s1) -- (se1) --cycle;

		\node[circle, fill=green!60!black, minimum size=3pt,inner sep=0pt, outer sep=0pt] at (sw) {};
		\node[circle, fill=green!60!black, minimum size=3pt,inner sep=0pt, outer sep=0pt] at (se) {};
		\node[circle, fill=green!60!black, minimum size=3pt,inner sep=0pt, outer sep=0pt] at (p) {};
		\node[circle, fill=green!60!black, minimum size=3pt,inner sep=0pt, outer sep=0pt] at (pp) {};

		\node[circle, fill=red!60!black, minimum size=3pt,inner sep=0pt, outer sep=0pt] at (pp1) {};
		\node[circle, fill=red!60!black, minimum size=3pt,inner sep=0pt, outer sep=0pt] at (nw1) {};
		\node[circle, fill=red!60!black, minimum size=3pt,inner sep=0pt, outer sep=0pt] at (n1) {};
		\node[circle, fill=red!60!black, minimum size=3pt,inner sep=0pt, outer sep=0pt] at (p1) {};

		\end{tikzpicture}
	};

  \node [scale=1.05,align=center, anchor=north] at (\xw, 0.5*\ys+0.2) {\scriptsize (a) Input PDDs};

	\node (Constraints) at (\xe, \yn) {
	\begin{tikzpicture}[scale=0.6]
		
		\coordinate (e) at (00:2);
		\coordinate (ne) at (60:2);
		\coordinate (nw) at (120:2);
		\coordinate (w) at (180:2);
		\coordinate (sw) at (240:2);
		\coordinate (se) at (300:2);
		\coordinate (p) at ($(ne)!0.5!(e)$);
		\coordinate (pp) at ($(nw)!0.5!(w)$);
		
		\coordinate (ne1) at ($(30:2)+(1.2,-2)$);
		\coordinate (n1) at ($(90:2)+(1.2,-2)$);
		\coordinate (nw1) at ($(150:2)+(1.2,-2)$);
		\coordinate (sw1) at ($(210:2)+(1.2,-2)$);
		\coordinate (s1) at ($(270:2)+(1.2,-2)$);
		\coordinate (se1) at ($(330:2)+(1.2,-2)$);
		\coordinate (p1) at ($(ne1)!0.5!(se1)$);
		
		\path[name path=f1] (sw) -- (se) -- (p) -- (pp) --cycle;
		\draw[opacity=0.5] (e) -- (se) -- (sw) -- (w) -- (nw) -- (ne) --cycle;

  		\draw[-,blue!80!black,opacity=0.6, shorten >= -0.3cm, shorten <= -0.3cm, name path=i3] (ne1) -- (n1);
  		\coordinate (i) at ($(ne1)!0.5!(n1)$);
  		\coordinate (ii) at ($(i)!0.5cm!90:(n1)$);
  		\draw[->,blue!80!black,opacity=0.6, thick] (i) -- (ii);
  		\draw[-,blue!80!black,opacity=0.6, shorten >= -0.3cm, shorten <= -0.3cm, name path=i1] (n1) -- (nw1);
  		\coordinate (i) at ($(n1)!0.5!(nw1)$);
		\coordinate (ii) at ($(i)!0.5cm!90:(nw1)$);
		\draw[->,blue!80!black,opacity=0.6, thick] (i) -- (ii);  		
		\draw[-,blue!80!black,opacity=0.6, shorten >= -0.3cm, shorten <= -0.3cm, name path=i2] (nw1) -- (sw1);
		\coordinate (i) at ($(nw1)!0.5!(sw1)$);
		\coordinate (ii) at ($(i)!0.5cm!90:(sw1)$);
		\draw[->,blue!80!black,opacity=0.6, thick] (i) -- (ii);  		
  		\draw[-,blue!80!black,opacity=0.6, shorten >= -0.3cm, shorten <= -0.3cm] (sw1) -- (s1);
		\coordinate (i) at ($(sw1)!0.5!(s1)$);
		\coordinate (ii) at ($(i)!0.5cm!90:(s1)$);
		\draw[->,blue!80!black,opacity=0.6, thick] (i) -- (ii);  		
  		\draw[-,blue!80!black,opacity=0.6, shorten >= -0.3cm, shorten <= -0.3cm] (s1) -- (se1);
		\coordinate (i) at ($(s1)!0.5!(se1)$);
		\coordinate (ii) at ($(i)!0.5cm!90:(se1)$);
		\draw[->,blue!80!black,opacity=0.6, thick] (i) -- (ii);  		
  		\draw[-,blue!80!black,opacity=0.6, shorten >= -0.3cm, shorten <= -0.3cm] (se1) -- (ne1);
		\coordinate (i) at ($(se1)!0.5!(ne1)$);
		\coordinate (ii) at ($(i)!0.5cm!90:(ne1)$);
		\draw[->,blue!80!black,opacity=0.6, thick] (i) -- (ii);

  		\path[name path=p1] (se) -- (pp);
  		
		\draw[fill=green!90,opacity=0.3,name path=f1] (intersection of ne1--n1 and se--p) -- (intersection of se--pp and n1--nw1) -- (intersection of se--sw and sw1--nw1) -- (se) -- cycle;

		\node[circle, fill=green!60!black, minimum size=3pt,inner sep=0pt, outer sep=0pt] at (intersection of ne1--n1 and se--p) {};
		\node[circle, fill=green!60!black, minimum size=3pt,inner sep=0pt, outer sep=0pt] at (se) {};
		\node[circle, fill=green!60!black, minimum size=3pt,inner sep=0pt, outer sep=0pt] at (intersection of se--sw and sw1--nw1) {};
		\node[circle, fill=green!60!black, minimum size=3pt,inner sep=0pt, outer sep=0pt] at (intersection of se--pp and n1--nw1) {};
	\end{tikzpicture}
};

  \node [scale=1.05,align=center, anchor=north] at (\xe, 0.5*\ys+0.2) {\scriptsize (b) First Intersection};

	\node (Constraints) at (2*\xe, \yn) {
	\begin{tikzpicture}[scale=0.6]
		
		\coordinate (e) at (00:2);
		\coordinate (ne) at (60:2);
		\coordinate (nw) at (120:2);
		\coordinate (w) at (180:2);
		\coordinate (sw) at (240:2);
		\coordinate (se) at (300:2);
		\coordinate (p) at ($(ne)!0.5!(e)$);
		\coordinate (pp) at ($(nw)!0.5!(w)$);
		
		\coordinate (ne1) at ($(30:2)+(1.2,-2)$);
		\coordinate (n1) at ($(90:2)+(1.2,-2)$);
		\coordinate (nw1) at ($(150:2)+(1.2,-2)$);
		\coordinate (sw1) at ($(210:2)+(1.2,-2)$);
		\coordinate (s1) at ($(270:2)+(1.2,-2)$);
		\coordinate (se1) at ($(330:2)+(1.2,-2)$);
		\coordinate (p1) at ($(ne1)!0.5!(n1)$);
		\coordinate (pp1) at ($(nw1)!0.7!(sw1)$);			
		
		\path[name path=f1] (sw1) -- (nw1) -- (n1) -- (ne1) -- (p1) --cycle;
		\draw[opacity=0.5] (ne1) -- (n1) -- (nw1) -- (sw1) -- (s1) -- (se1) --cycle;

		\draw[fill=red!90,opacity=0.3,name path=f1] (intersection of nw1--pp1 and se--sw) -- (intersection of n1--pp1 and sw--se) -- (intersection of nw1--p1 and se--e) --  (intersection of n1--ne1 and se--e) -- (n1) -- (nw1) -- cycle;

		\draw[opacity=0.5] (e) -- (se) -- (sw) -- (w) -- (nw) -- (ne) --cycle;
		\draw[-,blue!80!black,opacity=0.6, shorten >= -0.3cm, shorten <= -0.3cm, name path=i3] (ne) -- (e);
		\coordinate (i) at ($(ne)!0.5!(e)$);
		\coordinate (ii) at ($(i)!0.5cm!270:(e)$);
		\draw[->,blue!80!black,opacity=0.6, thick] (i) -- (ii);
		\draw[-,blue!80!black,opacity=0.6, shorten >= -0.3cm, shorten <= -0.3cm, name path=i1] (e) -- (se);
		\coordinate (i) at ($(e)!0.5!(se)$);
		\coordinate (ii) at ($(i)!0.5cm!270:(se)$);
		\draw[->,blue!80!black,opacity=0.6, thick] (i) -- (ii);  		
		\draw[-,blue!80!black,opacity=0.6, shorten >= -0.3cm, shorten <= -0.3cm, name path=i2] (se) -- (sw);
		\coordinate (i) at ($(se)!0.5!(sw)$);
		\coordinate (ii) at ($(i)!0.5cm!270:(sw)$);
		\draw[->,blue!80!black,opacity=0.6, thick] (i) -- (ii);  		
		\draw[-,blue!80!black,opacity=0.6, shorten >= -0.3cm, shorten <= -0.3cm] (sw) -- (w);
		\coordinate (i) at ($(sw)!0.5!(w)$);
		\coordinate (ii) at ($(i)!0.5cm!270:(w)$);
		\draw[->,blue!80!black,opacity=0.6, thick] (i) -- (ii);  		
		\draw[-,blue!80!black,opacity=0.6, shorten >= -0.3cm, shorten <= -0.3cm] (w) -- (nw);
		\coordinate (i) at ($(w)!0.5!(nw)$);
		\coordinate (ii) at ($(i)!0.5cm!270:(nw)$);
		\draw[->,blue!80!black,opacity=0.6, thick] (i) -- (ii);  		
		\draw[-,blue!80!black,opacity=0.6, shorten >= -0.3cm, shorten <= -0.3cm] (nw) -- (ne);
		\coordinate (i) at ($(nw)!0.5!(ne)$);
		\coordinate (ii) at ($(i)!0.5cm!270:(ne)$);
		\draw[->,blue!80!black,opacity=0.6, thick] (i) -- (ii);  		
		
		\path[name path=p1] (se) -- (pp);

		\node[circle, fill=red!60!black, minimum size=3pt,inner sep=0pt, outer sep=0pt] at (n1) {};
		\node[circle, fill=red!60!black, minimum size=3pt,inner sep=0pt, outer sep=0pt] at (nw1) {};
		\node[circle, fill=red!60!black, minimum size=3pt,inner sep=0pt, outer sep=0pt] at (intersection of p1--n1 and se--e) {};
		\node[circle, fill=red!60!black, minimum size=3pt,inner sep=0pt, outer sep=0pt] at (intersection of se--sw and n1--pp1) {};
		\node[circle, fill=red!60!black, minimum size=3pt,inner sep=0pt, outer sep=0pt] at (intersection of se--sw and pp1--nw1) {};
		\node[circle, fill=red!60!black, minimum size=3pt,inner sep=0pt, outer sep=0pt] at (intersection of se--e and p1--nw1) {};

	\end{tikzpicture}
};

  \node [scale=1.05,align=center, anchor=north] at ({2*\xe},0.5*\ys+0.2) {\scriptsize (c) Second Intersection};
  
	\node (Constraints) at (3*\xe, \yn) {
	\begin{tikzpicture}[scale=0.6]
		
		\coordinate (e) at (00:2);
		\coordinate (ne) at (60:2);
		\coordinate (nw) at (120:2);
		\coordinate (w) at (180:2);
		\coordinate (sw) at (240:2);
		\coordinate (se) at (300:2);
		\coordinate (p) at ($(ne)!0.5!(e)$);
		\coordinate (pp) at ($(nw)!0.5!(w)$);
		
		\coordinate (ne1) at ($(30:2)+(1.2,-2)$);
		\coordinate (n1) at ($(90:2)+(1.2,-2)$);
		\coordinate (nw1) at ($(150:2)+(1.2,-2)$);
		\coordinate (sw1) at ($(210:2)+(1.2,-2)$);
		\coordinate (s1) at ($(270:2)+(1.2,-2)$);
		\coordinate (se1) at ($(330:2)+(1.2,-2)$);
		\coordinate (p1) at ($(ne1)!0.5!(n1)$);
		\coordinate (pp1) at ($(nw1)!0.7!(sw1)$);

		\draw[opacity=0.5] (ne1) -- (n1) -- (nw1) -- (sw1) -- (s1) -- (se1) --cycle;
		\draw[opacity=0.5] (e) -- (se) -- (sw) -- (w) -- (nw) -- (ne) --cycle;
		
		\draw[fill=blue!80!black,opacity=0.6,opacity=0.3,name path=f1] (intersection of nw1--sw1 and se--sw) -- (se) -- (intersection of n1--ne1 and se--e) -- (n1) -- (nw1) -- cycle;
		
		\node[circle, fill=red!60!black, minimum size=3pt,inner sep=0pt, outer sep=0pt] at (n1) {};
		\node[circle, fill=red!60!black, minimum size=3pt,inner sep=0pt, outer sep=0pt] at (nw1) {};
		\node[circle, fill=red!60!black, minimum size=3pt,inner sep=0pt, outer sep=0pt] at (intersection of p1--n1 and se--e) {};
		\node[circle, fill=red!60!black, minimum size=3pt,inner sep=0pt, outer sep=0pt] at (intersection of se--sw and n1--pp1) {};
		\node[circle, fill=red!60!black, minimum size=3pt,inner sep=0pt, outer sep=0pt] at (intersection of se--sw and pp1--nw1) {};
		\node[circle, fill=red!60!black, minimum size=3pt,inner sep=0pt, outer sep=0pt] at (intersection of se--e and p1--nw1) {};
		
		\node[circle, fill=green!60!black, minimum size=3pt,inner sep=0pt, outer sep=0pt] at (intersection of ne1--n1 and se--p) {};
		\node[circle, fill=green!60!black, minimum size=3pt,inner sep=0pt, outer sep=0pt] at (se) {};
		\node[circle, fill=green!60!black, minimum size=3pt,inner sep=0pt, outer sep=0pt, opacity=0.6] at (intersection of se--sw and sw1--nw1) {};
		\node[circle, fill=green!60!black, minimum size=3pt,inner sep=0pt, outer sep=0pt] at (intersection of se--pp and n1--nw1) {};

	\end{tikzpicture}
};

	\node [scale=1.05,align=center, anchor=north] at (3*\xe, 0.5*\ys+0.2) {\scriptsize (d) Combining Intersections};
  	
  	\node (Constraints) at (4*\xe, \yn) {
  	\begin{tikzpicture}[scale=0.6]
  		
  		\coordinate (e) at (00:2);
  		\coordinate (ne) at (60:2);
  		\coordinate (nw) at (120:2);
  		\coordinate (w) at (180:2);
  		\coordinate (sw) at (240:2);
  		\coordinate (se) at (300:2);
  		\coordinate (p) at ($(ne)!0.5!(e)$);
  		\coordinate (pp) at ($(nw)!0.5!(w)$);
  		
  		\coordinate (ne1) at ($(30:2)+(1.2,-2)$);
  		\coordinate (n1) at ($(90:2)+(1.2,-2)$);
  		\coordinate (nw1) at ($(150:2)+(1.2,-2)$);
  		\coordinate (sw1) at ($(210:2)+(1.2,-2)$);
  		\coordinate (s1) at ($(270:2)+(1.2,-2)$);
  		\coordinate (se1) at ($(330:2)+(1.2,-2)$);
  		\coordinate (p1) at ($(ne1)!0.5!(se1)$);

  		\path[opacity=0.5] (ne1) -- (n1) -- (nw1) -- (sw1) -- (s1) -- (se1) --cycle;
  		\path[opacity=0.5] (e) -- (se) -- (sw) -- (w) -- (nw) -- (ne) --cycle;
  		
  		\node[circle, fill=blue!60!black, minimum size=3pt,inner sep=0pt, outer sep=0pt] at (intersection of ne1--n1 and se--e) {};
  		\node[circle, fill=blue!60!black, minimum size=3pt,inner sep=0pt, outer sep=0pt] at (n1) {};
  		\node[circle, fill=blue!60!black, minimum size=3pt,inner sep=0pt, outer sep=0pt] at (nw1) {};
  		\node[circle, fill=blue!60!black, minimum size=3pt,inner sep=0pt, outer sep=0pt] at (intersection of se--sw and sw1--nw1) {};
  		\node[circle, fill=blue!60!black, minimum size=3pt,inner sep=0pt, outer sep=0pt] at (se) {};

  		\draw[fill=blue!80!black,opacity=0.6,opacity=0.3,name path=f1] (intersection of nw1--sw1 and se--sw) -- (se) -- (intersection of n1--ne1 and se--e) -- (n1) -- (nw1) -- cycle;
 		
  	\end{tikzpicture}
  };

	\node [scale=1.05,align=center, anchor=north] at (4*\xe, 0.5*\ys+0.2) {\scriptsize (e) Enforcing A-Irredundancy};
\end{tikzpicture}

%% file: figures/sigmoid.tex
\begin{tikzpicture}
	\draw[->] (-3, 0) -- (3.0, 0) node[right] {$x$};
	\draw[->] (0, -0.4) -- (0, 2.0) node[above] {$y$};
	\draw[scale=1.5, domain=-4:4, smooth, variable=\x, blue] plot ({\x/2}, {exp(\x)/(1+exp(\x)});

	\def\a{-3}
	\def\b{3}
	\def\c{0.5}
	\coordinate (a) at ({1.5*\a/2},{1.5*exp(\a)/(1+exp(\a))});
	\coordinate (b) at ({1.5*\b/2},{1.5*exp(\b)/(1+exp(\b))});
	\coordinate (c) at ({1.5*\c/2},{1.5*exp(\c)/(1+exp(\c))});
	\coordinate (at) at ($(c)+1.5*({(\a-\c)/2},{(\a-\c)*exp(-\a)/(1+exp(-\a)^2)})$);
	\coordinate (ab) at ($(a)+1.5*({(\c-\a)/2},{(\c-\a)*exp(-\a)/(1+exp(-\a)^2)})$);
	
	\coordinate (bt) at ($(b)+1.5*({(\c-\b)/2},{(\c-\b)*exp(-\b)/(1+exp(-\b)^2)})$);
	
	\node[circle, fill=black, minimum size=3pt,inner sep=0pt, outer sep=0pt] at (a) {};
	\node[circle, fill=black, minimum size=3pt,inner sep=0pt, outer sep=0pt] at (b) {};
	\node[circle, fill=black, minimum size=3pt,inner sep=0pt, outer sep=0pt] at (c) {};
	
	\draw[fill=blue!90,opacity=0.25] (a) -- (ab) -- (c) --(at) -- cycle;
	\draw[fill=blue!90,opacity=0.25] (bt) -- (c) -- (b) -- cycle;
	\draw[-] ({1.5*\a/2},-0.4) -- ({1.5*\a/2},1.8);
	\draw[-] ({1.5*\b/2},-0.4) -- ({1.5*\b/2},1.8);
	\draw[-] ({1.5*\c/2},-0.4) -- ({1.5*\c/2},1.8);
	
	\node[anchor=north west,align=center] at ({1.5*\a/2},-0.00) {\scriptsize $l_x$};
	\node[anchor=north west,align=center] at ({1.5*\b/2},-0.00) {\scriptsize $u_x$};	
	\node[anchor=north west,align=center] at ({1.5*\c/2},-0.00) {\scriptsize $c$};
	
	\node[anchor=south west,align=center] at ($(b)+(0.3,-0.7)$) {$y=\frac{e^x}{1+e^x}$};
	
	\node[anchor=north,align=center,font=\scriptsize] at ({1.5*(\c+\a)/4},2.00) {neither convex\\ nor concave};
	\node[anchor=north,align=center] at ({1.5*(\c+\b)/4},1.97) {\scriptsize concave};
\end{tikzpicture}

%% file: figures/max_pool.tex
\begin{tikzpicture}
	\draw[->] (-2.5, 0) -- (2.5, 0) node[right] {$x_1$};
	\draw[->] (0, -1.5) -- (0, 1.2) node[above] {$x_2$};
	\draw[scale=1.0, domain=-1.0:0.8, smooth, variable=\x, blue] plot ({\x}, {\x});

	\def\a{-1.8}
	\def\c{1.6}
	\def\b{-1.0}
	\def\d{0.8}
	\coordinate (ab) at ({\a},{\b});
	\coordinate (ad) at ({\a},{\d});
	\coordinate (cb) at ({\c},{\b});
	\coordinate (cd) at ({\c},{\d});
	\coordinate (bb) at ({\b},{\b});
	\coordinate (dd) at ({\d},{\d});

	\draw[fill=blue!90,opacity=0.25] (ab) -- (bb) -- (dd) --(ad) -- cycle;
	\draw[fill=green!90,opacity=0.25] (cb) -- (cd) -- (dd) -- (bb) -- cycle;
	\draw[-, shorten >= -0.3cm, shorten <= -0.3cm] (ab) -- (cb);
	\draw[-, shorten >= -0.3cm, shorten <= -0.3cm] (cb) -- (cd);
	\draw[-, shorten >= -0.3cm, shorten <= -0.3cm] (cd) -- (ad);
	\draw[-, shorten >= -0.3cm, shorten <= -0.3cm] (ad) -- (ab);
	\node[anchor=south east,align=center] at (\a+0.05,-0.05) {\scriptsize $l_{x_1}$};
	\node[anchor=south west,align=center] at (\c+0.05,-0.05) {\scriptsize $u_{x_1}$};
	\node[anchor=north west,align=center] at (-0.05,\b-0.05) {\scriptsize $l_{x_2}$};
	\node[anchor=south west,align=center] at (-0.05,\d-0.05) {\scriptsize $u_{x_2}$};

	\node[anchor=south east,align=center] at ($(cb)+(-0.0,+0.2)$) {\scriptsize$\bc{D}^1\!\colon x_1 \geq x_2$};
	\node[anchor=north west,align=center] at ($(ad)+(0.1,-0.2)$) {\scriptsize$\bc{D}^2\!\colon x_2 \geq x_1$};

	\node[anchor=south west,align=center] at ($(cd)+(+0.2,+0.)$) {$\bc{B}^1\!\colon \: y=x_1$};
	\node[anchor=south east,align=center] at ($(ab)+(-0.1,0.1)$) {$\bc{B}^2\!\colon \: y=x_2$};
\end{tikzpicture}

%% file: experimental.tex
\section{Experimental Evaluation}\label{sec:experiments}
\renewcommand{\arraystretch}{1.2}
\begin{wraptable}[16]{r}{0.55\textwidth}
	\centering
	\vspace{-10.5mm}
	\caption{Neural network architectures used in experiments.}
	\vspace{-3mm}
	\footnotesize
	\scalebox{0.83}{
		\begin{threeparttable}
			\begin{tabular}{@{}l ll rrl@{}}
				\toprule
				Dataset  & Model & Type & Neurons & Layers & Activation \\
				\midrule
				MNIST
				& $5 \times 100$\tnote{\ref{fot:nets}} & FC   & 510     & 5 & ReLU           \\
				& $6 \times 100$ & FC   & 600     & 6 & Tanh/Sigm      \\
				& $8 \times 100$\tnote{\ref{fot:nets}} & FC   & 810     & 8 & ReLU           \\
				& $9 \times 100$ & FC   & 900     & 9 & Tanh/Sigm      \\
				& $5 \times 200$\tnote{\ref{fot:nets}} & FC   & 1\,010  & 5 & ReLU           \\
				& $6 \times 200$ & FC   & 1\,200  & 6 & Tanh/Sigm      \\
				& $8 \times 200$\tnote{\ref{fot:nets}} & FC   & 1\,610  & 8 & ReLU           \\
				& \convsmall     & Conv & 3\,604  & 3 & Relu/Tanh/Sigm \\
				& \convbig       & Conv & 48\,064 & 6 & ReLU           \\
				\midrule
				CIFAR10
				& \convsmall     & Conv & 4\,852  & 3 & ReLU \\
				& \CNNAmix   	 & Conv & 6\,244  & 3 & ReLU \\
				& \CNNBadv     	 & Conv & 16\,634  & 3 & ReLU \\
				& \convbig       & Conv & 62\,464  & 6 & ReLU \\
				& \resnet        & Residual & 107\,496  & 10 & ReLU \\
				\midrule
				Self-Driving
				& \NVIDIA     & Conv & 107\,032  & 8 & ReLU + Tanh \\
				\bottomrule
			\end{tabular}
		\end{threeparttable}
	}
	\label{Ta:networks}
\end{wraptable}
In this section, we evaluate the effectiveness of \tool and show that it significantly improves over state-of-the-art verifiers on a range of challenging benchmarks yielding up to 14\%, 30\% and 34\% precision gains on \mbox{ReLU-,} \mbox{Sigmoid-,} and Tanh-based networks, respectively.
Further, we show that \tool can scale to real-world problems, obtaining tight bounds in an autonomous driving steering-angle-prediction task. Finally, we demonstrate the effectiveness and benefits of computing relaxations with \SBLM and \PDDM compared to directly using the exact convex hull.

\subsection{Experimental Setup}

The neural network certification benchmarks for fully connected networks were run on a 20 core 2.20GHz Intel Xeon Silver 4114 CPU with 100 GB of main memory and those for convolutional networks on a 16 Core 3.6GHz Intel i9-9900K with 64GB of main memory and an NVIDIA RTX 2080Ti. We use Gurobi 9.0 for solving MILP and LP problems \cite{gurobi2018}.

\subsection{Benchmarks}

We evaluate \tool on a wide range of networks based on ReLU, Tanh, and Sigmoid activations:%
\begin{itemize}
	\item The set of fully-connected and convolutional ReLU networks\footnote{\label{fot:nets}The networks referred to as $6\times \cdot\,00$ and $9\times \cdot\,00$ in previous work only include $5$ and $8$ hidden layers, respectively, and have therefore been renamed.} from \cite{singh2019beyond} trained using DiffAI \cite{mirman2018diffai}, PGD \cite{madry2017towards}, Wong \cite{wong2018scaling}, and natural training (see results on MNIST and CIFAR10 in \Tableref{Ta:Relu}).
	\item The published set of CIFAR10 convolutional networks from \cite{dathathri2020enabling}, trained using either just PGD or a mix of standard and PGD training (see results on CIFAR10 in \Tableref{Ta:Relu_SDP}).
	\item The set of fully-connected and convolutional Tanh and Sigmoid networks from \cite{singh2019beyond} trained using natural training (see results on MNIST in \Tableref{Ta:SCurve}).
	\item The NVIDIA self-driving car network architecture \NVIDIA \cite{bojarski2016end} trained on a steering angle prediction task using the Udacity self-driving car dataset \cite{udacity2016selfdriving} with 31\,834 train and 1\,974 test samples\footnote{The labels of the original test set are not available (anymore), so we used videos 1, 2, 5, and 6 as train and video 4 (instead of 3) as test dataset.}, an input resolution of $3 \times 66 \times 200$, and PGD \cite{madry2017towards} training (see results in \Tableref{Ta:Relu_Dave}).
\end{itemize}
While we evaluate performance for the widely considered and challenging $\ell_\infty$ perturbations\footnote{That is, $y := c(\bm{x})_i = c(\bm{x}'), \forall \bm{x}' \in \bb{B}^\infty_\epsilon := \{\bm{x} \in \bc{X} \mid ||\bm{x}-\bm{x}'||_\infty \leq \epsilon\} \Leftrightarrow \min_{\bm{x}' \in \bb{B}^\infty_\epsilon} \bm{h}(\bm{x}')_y - \bm{h}(\bm{x}')_i > 0, \forall i \neq y$}, \tool can also be applied to other specifications including individual fairness \cite{ruoss2020learning}, global safety properties \cite{katz2017reluplex}, acoustic  \cite{ryou2020fast}, geometric \cite{balunovic2019certifying}, and spatial \cite{ruoss2020efficient} based perturbations.

For classification tasks and ReLU networks, we compare \tool with a range of state-of-the-art incomplete verifiers %
notably also the ReLU-specialized \kPoly \cite{singh2019beyond}, \optcv \cite{tjandraatmadja2020convex}, and additionally the highly optimized and fully GPU-based \BCrown \cite{wang2021beta} (in incomplete mode).
For classification using Tanh and Sigmoid activations, fewer verifiers are available and thus we compare with the state-of-the-art incomplete verifier \DeepPoly \cite{singh2019abstract}.
Few verification methods consider the regression setting and to the best of our knowledge, we are the first to analyze the full-size \NVIDIA network.  \Neurify \cite{wang2018neurify} analyses a heavily scaled-down version in a binary classification setting, but in complete mode it does not scale to the much larger networks analysed here. In incomplete mode, it uses the same bounds as \DeepZono \cite{singh2018fast} and is less precise than \GPUpoly \cite{mller2021neural} to which we compare. \BCrown does not support regression tasks and while an extension might be possible, it is non-trivial. It is also unclear if the approach scales to networks of this size.

\begin{table*}[t]
	\caption{Number of verified adversarial regions of the first 1\,000 samples and runtime for \tool, \optcv\cite{tjandraatmadja2020convex}, and \kPoly\cite{singh2019beyond}. Natural (NOR), adversarial (PGD \cite{madry2017towards}), or provable (DiffAI \cite{mirman2018diffai}, Wong \cite{wong2018scaling}) training was used.}
	\centering
	{  \footnotesize
		\begin{adjustbox}{max width=\textwidth}
			\small
			\begin{threeparttable}
				\begin{tabular}{@{}llllrr@{}rrrrrrr@{}}
					\toprule
					Dataset & Model & Training & Accuracy & $\epsilon$ & $n_s$ & \multicolumn{2}{c}{\kPoly}& \multicolumn{2}{c}{\optcv$^\dagger$}& \multicolumn{2}{c}{\tool (ours)}&  \# Upper Bound \\
					\cmidrule(lr){7-8}   \cmidrule(lr){9-10}  \cmidrule(lr){11-12}
					&               &    &       & & & \hspace{5pt} \# Ver & Time & \# Ver & Time & \# Ver & Time &\\
					
					\midrule
					MNIST
					& $5 \times 100$ & NOR & 960   & 0.026 & 100  &  441 & 307 &429 & 137  & \textbf{510} & 159  & 842 \\
					& $8 \times 100$ & NOR & 947   & 0.026 & 100  &  369 & 171 &384 & 759  & \textbf{428} & 301 & 820 \\
					& $5 \times 200$ & NOR & 972   & 0.015 & 50   &  574 & 187 &601 & 403  & \textbf{690} & 224 & 901 \\
					& $8 \times 200$ & NOR & 950   & 0.015 & 50   &  506 & 464 &528 & 3451 & \textbf{612} & 395 & 911 \\
					& \convsmall     & NOR & 980   & 0.120 & 100  &  347 & 477 &436 & 55   & \textbf{640} & 51  & 733 \\
					& \convbig       & DiffAI & 929 & 0.300 & 100 &  736 & 40 &771 & 102   & \textbf{775} & 5.5  & 790 \\
					\midrule
					CIFAR10
					& \convsmall     & PGD & 630  & 2/255 & 100 &  399 & 86  & 398 & 105 & \textbf{458} & 16  & 481 \\
					& \convbig       & PGD & 631  & 2/255 & 100 &  459 & 346 & $\text{n/a}^\dagger$   & $\text{n/a}^\dagger$   & \textbf{482} & 128  & 550 \\
					& \resnet		 & Wong & 290 & 8/255 & 50 & 245 & 91  & $\text{n/a}^\dagger$  & $\text{n/a}^\dagger$   & \textbf{248} & 1.9  & 248 \\
					\bottomrule
				\end{tabular}
				\begin{tablenotes}
					\item[]{$^\dagger$The \optcv \cite{tjandraatmadja2020convex} code has not been released; we report their runtimes and results where available.}
				\end{tablenotes}
			\end{threeparttable}
		\end{adjustbox}
	}
	\label{Ta:Relu}
\end{table*}

\begin{figure*}[t]
	\centering
	\begin{subfigure}[b]{0.31\textwidth}
		\centering
		\includegraphics[width=\textwidth]{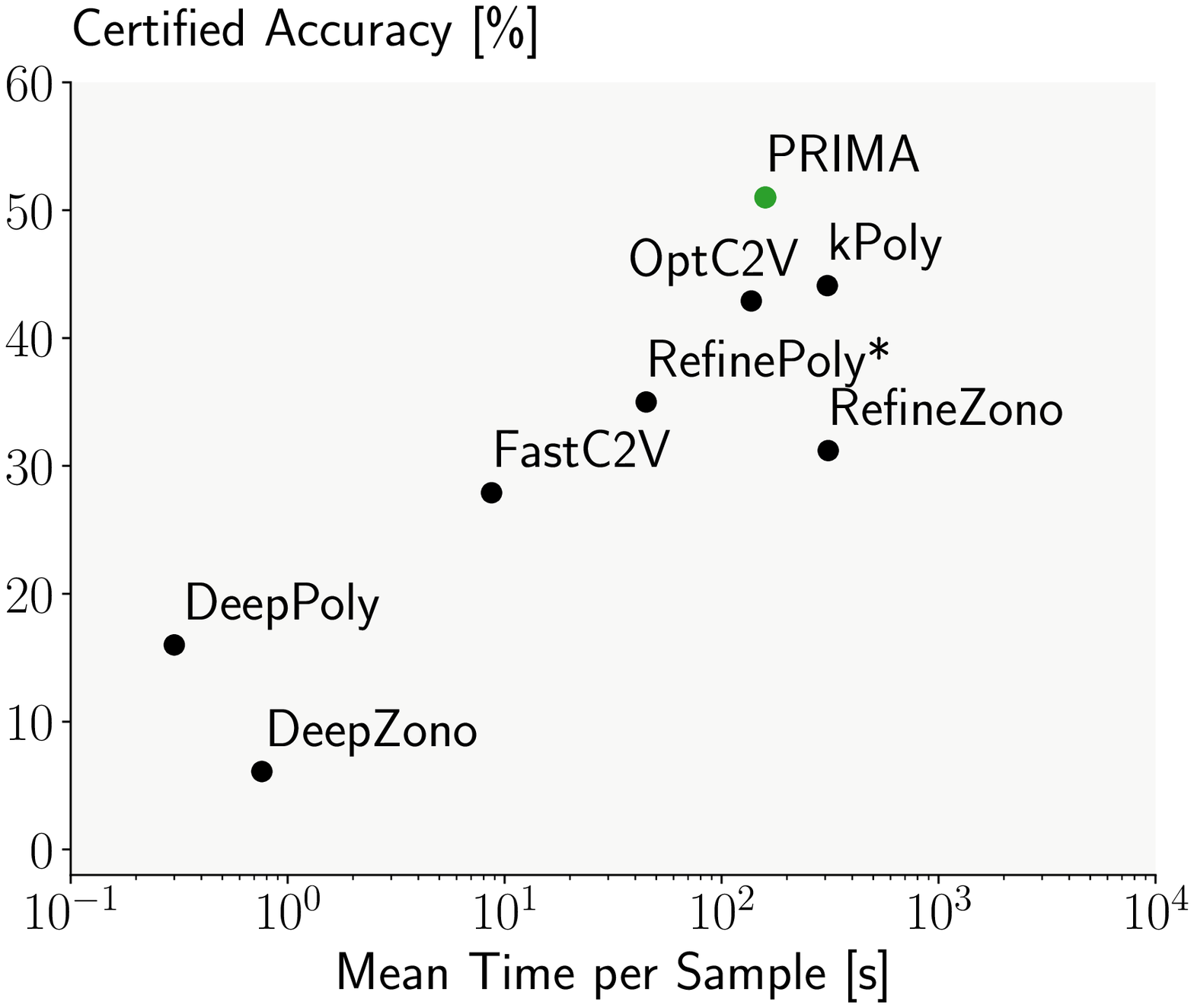}
		\vspace{-5.0mm}
		\subcaption{MNIST $5 \times 100$, $\epsilon =0.026$}
	\end{subfigure}
	\hfill
	\begin{subfigure}[b]{0.31\textwidth}
		\centering
		\includegraphics[width=\textwidth]{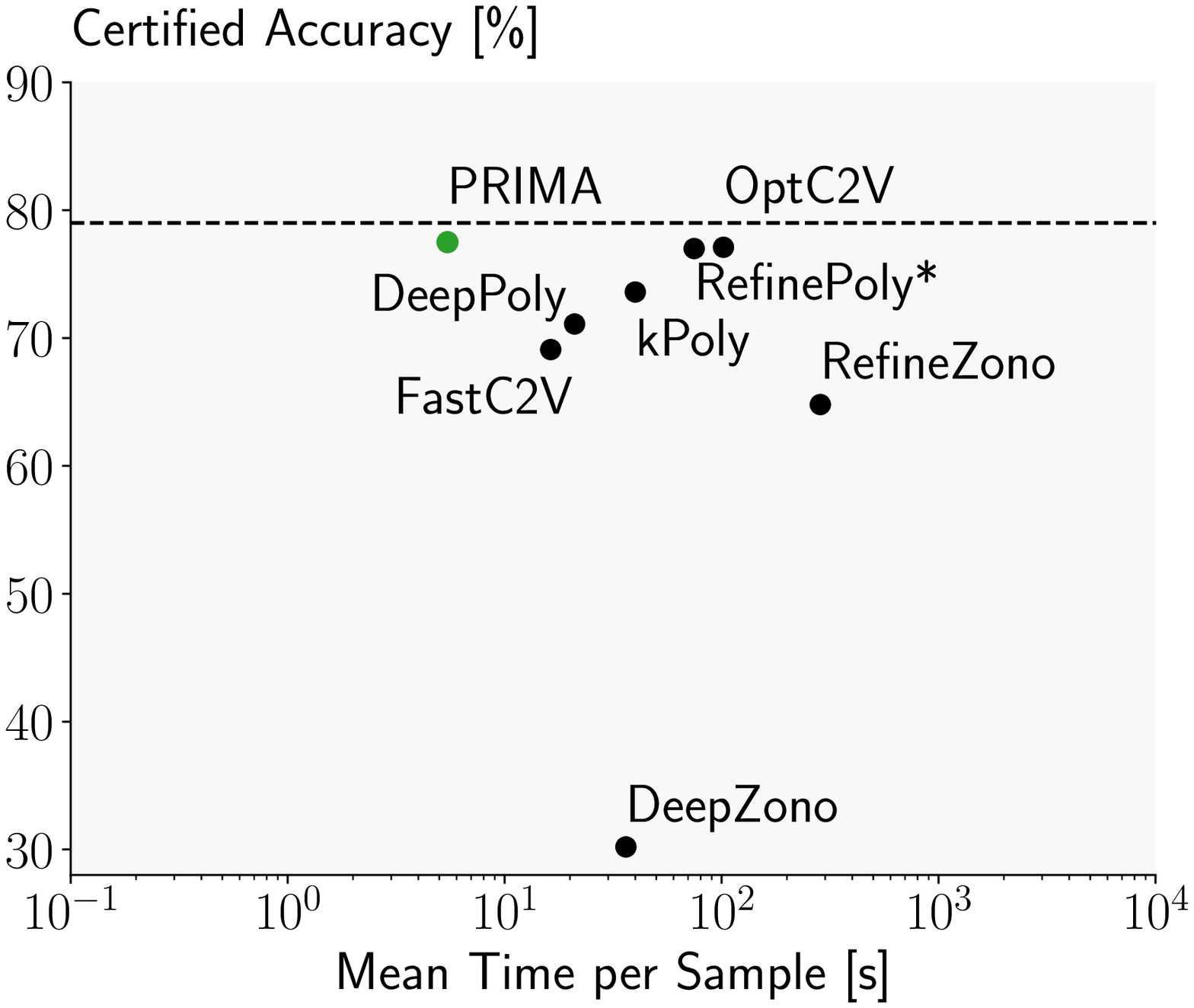}
		\vspace{-5.0mm}
		\subcaption{MNIST \convbig, $\epsilon =0.3$}
		\label{fig:method_comp_cb}
	\end{subfigure}
	\hfill
	\begin{subfigure}[b]{0.31\textwidth}
		\centering
		\includegraphics[width=\textwidth]{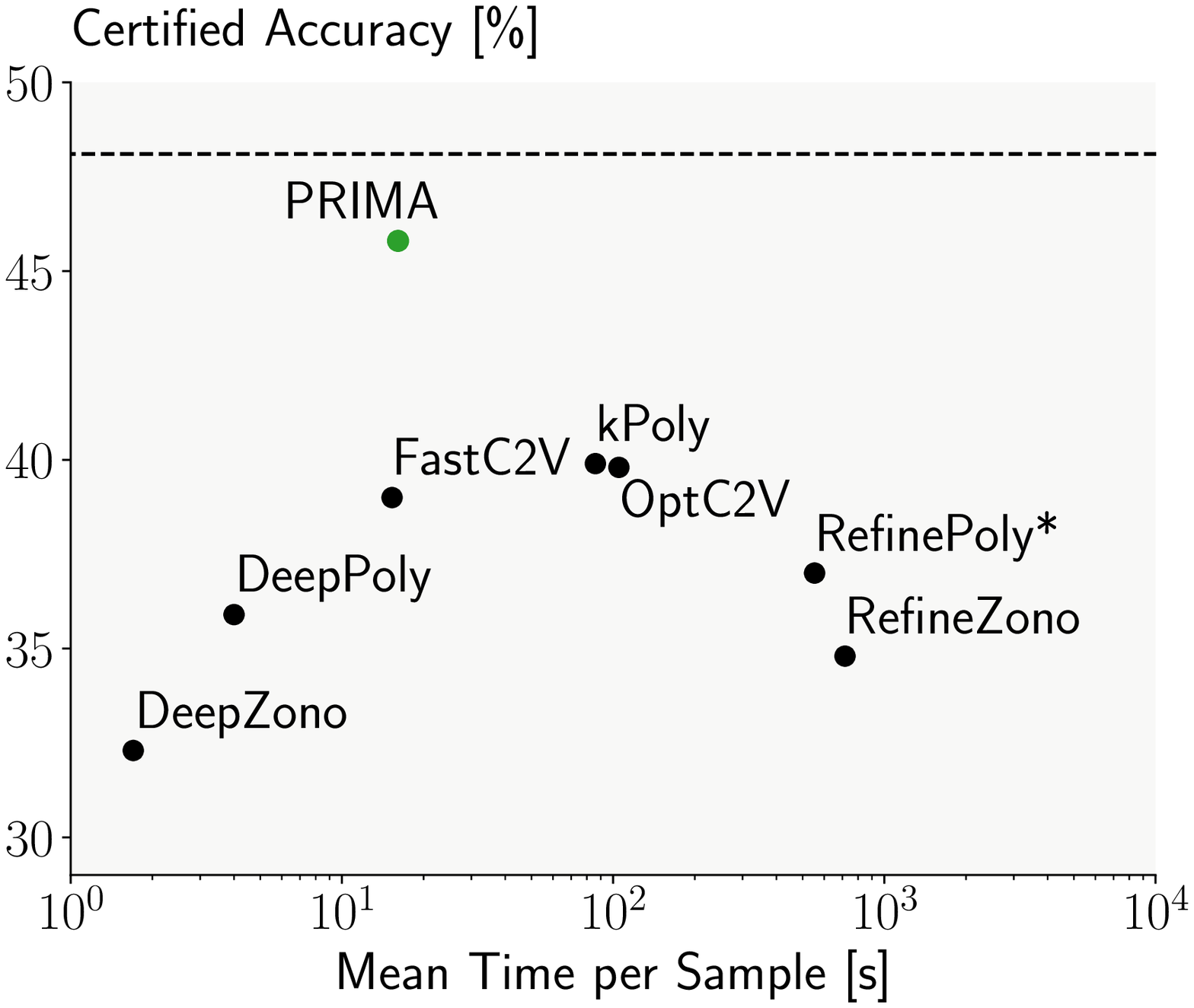}
		\vspace{-5.0mm}
		\subcaption{CIFAR10 \convsmall, $\epsilon = 2/255$}
	\end{subfigure}
	\vspace{-3mm}
	\caption{Comparison of the runtime/accuracy trade-off of \tool (ours), \optcv \cite{tjandraatmadja2020convex}, \fastcv \cite{tjandraatmadja2020convex}, \kPoly \cite{singh2019beyond}, \RefinePoly \cite{singh2019abstract}, \DeepPoly \cite{singh2019boosting} (equivalent bounds to \Crown \cite{zhang2018crown} and \CNNCert \cite{boopathy2019cnncert}), \RefineZono \cite{singh2019boosting} and \DeepZono \cite{singh2018fast} (equivalent bounds to \FastLin \cite{weng2018fastlin} and \Neurify \cite{wang2018neurify} in incomplete mode), evaluated on the first 1000 samples (100 for \RefinePoly) of the corresponding test sets. The tightest known upper bound to the certifiable accuracy is shown as dashed line. Higher and further left is better.}
	\label{fig:method_comp_small}
	\vspace{-4.5mm}
\end{figure*}

For our experiments, we use the setup outlined in \Secref{sec:framework} which is similar to \kPoly in \cite{singh2019beyond}. We use \DeepPoly or \GPUpoly (for convolutional networks) to determine the octahedral input bounds required to compute the multi-neuron constraints with \tool.
For fully-connected networks, we refine the neuron-wise bounds of unstable neurons using the MILP encoding from \citet{tjeng2017evaluating} for the second activation layer (the first layer bounds are already exact) and an LP encoding for the remaining layers.
We note that encoding more layers with MILP does not scale on these networks.
For convolutional networks, we encode some of the neurons in the last one or two layers using the MILP encoding from \cite{tjeng2017evaluating}. We note that the concurrent bound optimization in \BCrown corresponds to simultaneous bound-refinement on all neurons of all layers, which is orthogonal to our approach and a promising direction to be explored in future work (though intractable without a GPU-based LP solver).
We report as {\em Accuracy} the number of correctly classified samples out of the considered test set, as {\em \# Upper Bound} the number of properties that could not be falsified and hence form an upper bound to the number of certifiable properties, as {\em \# Ver} the number of verified regions, and as {\em Time} the average runtime per correctly classified sample in seconds.

\subsection{Image Classification with ReLU Activation}

We compare \tool against the state-of the art methods \kPoly and \optcv in \Tableref{Ta:Relu} and \BCrown in \Tableref{Ta:Relu_SDP}.
Computing multi-neuron constraints for groups of $k=4$ ReLU neurons becomes feasible with \SBLM and \PDDM reducing the time per group from several minutes, when directly computing exact convex hulls as in \kPoly, to less than $50$ milliseconds.
Nevertheless, we find empirically that the best strategy to leverage this speed-up is to evaluate a large variety of small groups. Unless reported differently, we consider overlapping groups of size $k=3$ with $n_s=100$.

\paragraph{Comparison with the state-of-the-art}
Figure~\ref{fig:method_comp_small} shows scatter plots comparing the runtime and precision of \tool with those of other state-of-the-art verifiers on the robustness certification of a normally trained $5\times 100$ MLP, a provably trained \convbig (MNIST) and an adversarially trained \convsmall (CIFAR10).
We note that adversarially and provably trained networks sacrifice accuracy for ease of certification, making normally trained networks more relevant and challenging. Here, fast, purely propagation-based, incomplete verifiers like \DeepPoly verify only about $16\%$ of the images. In contrast, \tool verifies $51\%$ in $< 160$ seconds per image.
The closest verifiers in terms of precision are \kPoly and \optcv, which verify $44\%$ and $43\%$ of samples and take around $310$ and $140$ seconds, respectively.
Based on these observations, we compare \tool with \kPoly and \optcv on the remaining benchmarks from \cite{singh2019beyond}.

\begin{wraptable}[11]{r}{0.55\textwidth}
	\vspace{-3.5mm}
	\caption{Number of verified adversarial regions of the 100 random samples from the CIFAR10 test set evaluated by \cite{wang2021beta}. \CNNAmix is trained using a combination of adversarial and natural training and \CNNBadv only adversarially. Both are taken from \cite{dathathri2020enabling}.}
	\vspace{-1.5mm}
	\centering
	{  \footnotesize
		\begin{adjustbox}{max width=0.98\linewidth}
			\small
			\begin{threeparttable}
				\begin{tabular}{@{}lllrrrrrrr@{}}
					\toprule
					Model & $\epsilon$ &  Acc &  \multicolumn{2}{c}{\BCrown} & \multicolumn{2}{c}{\tool (ours)}&  \# Bound  \\
					\cmidrule(rl){4-5} \cmidrule(lr){6-7}
					& & & \# Ver &Time \hspace{3pt}& \# Ver & Time \hspace{3pt}& \\
					\midrule
					\CNNAmix     & 2/255 & 100 & 43 & 209\hspace{3pt} & \textbf{57} & 53 \hspace{3pt} & 68 \\
					\CNNBadv     & 2/255 & 100 & \textbf{46} & 234\hspace{3pt} & 43 & 260\hspace{3pt}  & 81 \\
					
					\bottomrule
				\end{tabular}
			\end{threeparttable}
		\end{adjustbox}
	}
	\label{Ta:Relu_SDP}
\end{wraptable}
\paragraph{Comparison with \kPoly and \optcv}
For all normally trained networks, \tool is significantly more accurate than both \kPoly \cite{singh2019beyond} and \optcv \cite{tjandraatmadja2020convex}, verifying between $44$ and $201$ more regions than the better of the two while sometimes also being significantly faster. These results are summarized in~\Tableref{Ta:Relu}.
For the, comparatively easy to verify (as can be seen in \Figref{fig:method_comp_small}(b)), DiffAI trained \convbig MNIST network, we gain less precision verifying only $4$ more regions than \optcv.
However, the easier proofs come at the cost of reduced accuracy, making them less relevant for real-world applications.
For both PGD-trained CIFAR10 networks, \tool verifies between $23$ and $59$ more regions than \kPoly and \optcv while being around four times faster.
On the provably trained \resnet, \tool is $50$x faster than \kPoly and able to decide all properties. However, this network is so heavily regularized that even complete verification via a MILP encoding is tractable.
In summary, \tool is usually faster than \kPoly and \optcv, especially on larger networks, and is always more precise, sometimes substantially so.

\begin{wraptable}[22]{r}{0.52\textwidth}
	\vspace{-4.7mm}
	\caption{Evaluation of a range of parameters for grouping set size $n_s$, group size $k$, and overlap $s$, partial MILP refinement, and neuron-wise bound refinement for the first $100$ samples of the \mnist test set and the normally trained $5 \times 100$. Of the first $100$ samples, $99$ are classified correctly and for $9$ of those a counterexample is known.}
	\vspace{-2mm}
	\centering
	{\scalebox{0.79}{
			\small
			\renewcommand{\arraystretch}{1.1}
			\begin{threeparttable}
				\begin{tabular}{rccrrccrr}
					\toprule
					\multirow{2.2}{*}{$n_s$} & \multirow{2.2}{*}{$k$} & \multirow{2.2}{*}{$s$} & \multicolumn{2}{c}{Partial MILP} & \multicolumn{2}{c}{Refinement} & \multirow{2.2}{*}{\# Ver} & \multirow{2.2}{*}{Time [s]} \\
					\cmidrule(lr){4-5}   \cmidrule(lr){6-7}
					& & & \# layers & \# neurons & LP & MILP & & \\
					\midrule
					1   & 1 & - & - & -   & - & - & 21  & 2.56   \\
					10  & 3 & 1 & - & -   & - & - & 26  & 5.75   \\
					20  & 3 & 1 & - & -   & - & - & 28  & 6.52   \\
					20  & 3 & 2 & - & -   & - & - & 28  & 67.79  \\
					20  & 4 & 1 & - & -   & - & - & 28  & 54.05  \\
					100 & 3 & 1 & - & -   & - & - & 28  & 16.59  \\
					\cmidrule(lr){1-3}
					1   & 1 & - & 1 & 30  & - & - & 23  & 4.58   \\
					100 & 3 & 1 & 1 & 30  & - & - & 30  & 42.00  \\
					100 & 3 & 1 & 1 & 100 & - & - & 30  & 44.03  \\
					100 & 3 & 1 & 2 & 100 & - & - & 35  & 117.37 \\
					\cmidrule(lr){1-5}
					1   & 1 & - & - & -   & y & - & 27 & 24.15   \\
					100 & 3 & 1 & - & -   & y & - & 45 & 99.40   \\
					100 & 3 & 1 & - & -   & y & y & 54 & 115.24  \\
					100 & 3 & 1 & 2 & 100 & y & y & 60 & 189.21 \\
					\bottomrule
				\end{tabular}
			\end{threeparttable}
	}}
	\label{tab:param_study}
	\vspace{-5mm}
\end{wraptable}

\paragraph{Comparison with \BCrown}
\mbox{\BCrown} \cite{wang2021beta} is a
highly optimized, fully GPU-based complete BaB \cite{morrison2016branch}
solver, supporting only ReLU activations\footnote{Extensions to
piecewise-linear activations with more than $m=2$ linear regions would
significantly increase runtime ($\bc{O}(m^d)$ with split depth $d$), while
precision would be significantly lower for non-piecewise linear activations.}
and the classification setting. When comparing complete and incomplete
verifiers on accuracy, it is crucial to ensure that similar runtimes were
achieved, as complete verifiers can, given sufficient time, decide any
property. The GPU-based LP solver underlying \BCrown is an orthogonal
development to the \tool multi-neuron constraints. \tool currently uses a much slower CPU-based solver which is the main bottleneck for large networks as the runtime for computing multi-neuron constraints becomes small via our improved algorithms (see \Secref{sec:SBLM_PDDM_eval}). We consider combining the GPU-based solver from \BCrown with our multi-neuron approximations as an interesting item for future work.
Despite the discrepancy in LP-solver performance distorting the comparison, \tool is still significantly faster on \CNNAmix while also achieving notably higher precision. On the larger network \CNNBadv, where LP-solver performance is more dominant, \BCrown achieves slightly higher precision and smaller runtime.
Unfortunately, we could not run the public version of \BCrown without soundness issues on the networks from \cite{singh2019beyond} and consequently only compare on networks they provide.
The recent SDP-based (semidefinite programming) \SDPFO \cite{dathathri2020enabling} takes many hours per sample and is outperformed by \BCrown. Thus we do not compare to it directly.

\subsection{Parameter Study}

In \Tableref{tab:param_study}, we compare the effect of different parameter combinations on runtime and accuracy for the $5\times 100$ MLP, which allows also more expensive settings to be evaluated while still representing a challenging verification problem with $\epsilon = 0.026$.
Using the single-neuron triangle relaxation ($k=1$) only $21$ regions can be verified.
Adding our multi-neuron constraints with partition sizes of $n_s=10$ and $n_s=20$ increases this to $26$ and $28$ regions, respectively.
Neither considering a larger overlap ($s=2$), nor larger groups ($k=4$), nor larger partition sizes ($n_s=100$) can increase the number of verified regions, despite significantly increased the runtime.
While using triangle relaxations with a partial MILP encoding is relatively fast it also only increases the accuracy to $23$ regions. In contrast, combining a partial MILP encoding with multi-neuron constraints yields, depending on the exact setting, an almost $75\%$ increase to $35$ verified regions, although at the price of increased runtime.
Refining the neuron-wise bounds using a triangle relaxation and LP encoding only improves the number of verified regions to $27$, while additionally using multi-neuron constraints yields a significant jump to $45$. This further improves to $54$ when using MILP to refine the second layer bounds and $60$ when additionally encoding the last two layers with MILP.
The significant increase in precision when combining tight multi-neuron constraints computed via \SBLM and \PDDM with other methods demonstrates their utility and highlights the potential of our abstraction-refinement-based approach.

\subsection{Effect of Grouping Strategy}
\begin{wrapfigure}[15]{r}{0.37\textwidth}
	\centering
	\vspace{-8.5mm}
	\includegraphics[width=1.0\linewidth]{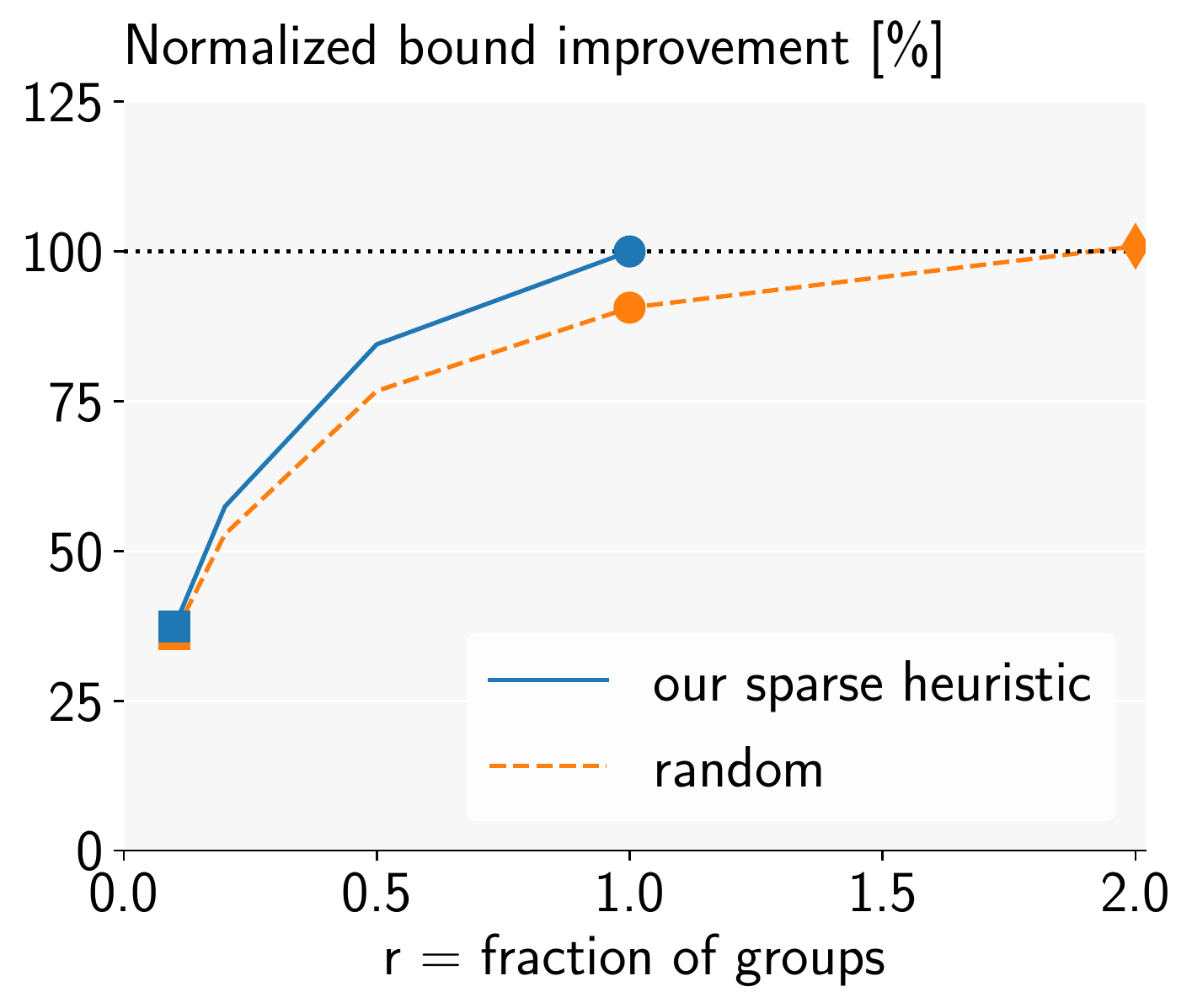}
	\vspace{-7mm}
	\caption{Normalized bound improvement over the fraction of groups used to compute multi-neuron constraints, $r$. Our method is the blue circle, whose gain is normalized to $100$\%.}
	\label{fig:grouping_strategy}
\end{wrapfigure}
We evaluate the sensitivity of \tool to the chosen neuron groupings, by comparing the performance\footnote{Concretely, we compare the obtained improvement of $\underline{h_{y,i}}$, the lower bound to the optimization objective $\min_{\bm{x}' \in \bb{B}^\infty_\epsilon} \bm{h}(\bm{x}')_y - \bm{h}(\bm{x}')_i$, over the triangle relaxation ($\Delta$) normalized using our standard sparse heuristic ($\tool$): $({\underline{h_{y,i}} - \underline{h_{y,i}^\Delta}})/({\underline{h_{y,i}^{\tool}} - \underline{h_{y,i}^\Delta}})$} of random groups with those generated by our sparse grouping heuristic in \Figref{fig:grouping_strategy} for the first $100$ test images of \cifar and the \convsmall network.
Concretely, we first generate a deterministic sparse grouping with our heuristic for a group size of $k=3$, a partition size of $n_s=100$, and a maximum overlap of $s=1$. Then we (randomly) reduce this grouping to a fraction $r$ (x-axis in \Figref{fig:grouping_strategy}) of the original number of groups. The random groupings are generated to have the same size (number of groups) by repeatedly drawing $k$ indices uniformly at random and rejecting duplicates.

We observe that considering fewer groups from our heuristic (blue in \Figref{fig:grouping_strategy}) reduces the bound improvement notably, e.g., to $37$\% at $r=0.1$ (blue square). Choosing random groups (orange in \Figref{fig:grouping_strategy}) is consistently worse (vertical gap in \Figref{fig:grouping_strategy}); by around $10$\% at $r=1.0$ (circles) closing to $3.4$\% at $r=0.1$ (squares). 
While our heuristic generates groups with small overlap to evenly cover all neurons, random sampling can lead to some groups with large overlap, while potentially not covering some neurons at all, leading to worse performance. Considering fewer groups makes overlaps between groups less likely, making the groupings resulting from the two sampling strategies more similar and explaining the shrinking performance gap. 
To obtain the same precision with random groups as with our heuristic, about twice as many ($r=2.0$, diamond) groups are needed (horizontal gap in \Figref{fig:grouping_strategy}). We repeated these experiments several times with different random seeds and obtained consistent results.

Overall, we conclude that while our heuristic consistently outperforms random groups, \tool is relatively insensitive to the exact groupings, as long as sufficiently many are used.

\subsection{Image Classification with Tanh and Sigmoid Activations}

\begin{wraptable}[16]{r}{0.52\textwidth}
	\vspace{-4.5mm}
	\caption{Number of verified adversarial regions and runtime in seconds of \tool vs. \DeepPoly for Tanh/Sigmoid on 100 images from the MNIST dataset.}
	\vspace{-2mm}
	\centering
	{  \footnotesize
		\scalebox{0.9}{
			\begin{tabular}{@{}lllrr@{}rrrr@{}}
				\toprule
				Act. & Model & Acc. & $\epsilon$ && \multicolumn{2}{c}{\DeepPoly} & \multicolumn{2}{c}{\tool} \\
				\cmidrule(lr){5-7}     \cmidrule(l){8-9}
				& 				 &    &       & & Ver. & Time & Ver. & Time \\
				\midrule
				Tanh
				& $6 \times 100$ & 97 & 0.006 & & 38 & 0.3 & \textbf{61} & 72.5 \\
				& $9 \times 100$ & 98 & 0.006 & & 18 & 0.4 & \textbf{52} & 186.0 \\
				& $6 \times 200$ & 98 & 0.002 & & 39 & 0.6 & \textbf{68} & 170.0 \\
				& \convsmall 		  & 99 & 0.005 & & 16 & 0.4 & \textbf{30} & 27.8 \\
				\midrule
				Sigm
				& $6 \times 100$ & 99 & 0.015 & & 30 & 0.3 & \textbf{53} & 96.9 \\
				& $9 \times 100$ & 99 & 0.015 & & 38 & 0.5 & \textbf{56} & 336.4 \\
				& $6 \times 200$ & 99 & 0.012 & & 43 & 1.0 & \textbf{73} & 267.0 \\
				& \convsmall 		   & 99 & 0.014 & & 30 & 0.5 & \textbf{51} & 47.0 \\
				\bottomrule
			\end{tabular}
		}
	}
	\label{Ta:SCurve}
\end{wraptable}
While using the exact convex hull algorithm for ReLU relaxations is merely slow, it becomes infeasible for non-piecewise-linear activations such as Tanh and Sigmoid. Computing the constraints for a single group of $k=3$ neurons can take minutes using direct exact convex hull computation, whereas \SBLM using \PDDM takes only $10$ milliseconds.
This dramatic speed-up is a result of \SBLM's decompositional approach of solving the problem in lower dimensions (see~\Secref{sec:FastPoly}), significantly reducing its complexity. Note that both methods compute only approximations of the optimal group-wise convex relaxation for these cases, as the underlying interval-wise bounds are not exact.

We evaluate our method on normally trained, fully-connected and convolutional networks for the \mnist dataset. We choose an $\epsilon$ for the $B^\infty_\epsilon$ region such that the state-of-the-art verifier for Tanh and Sigmoid activations, \DeepPoly, verifies less than 50\% of the regions. We remark that \DeepPoly is based on the same principles and has similar precision as other state-of-the-art verifiers for these activations such as \CNNCert \cite{boopathy2019cnncert} and \Crown \cite{zhang2018crown}.

We use overlapping groups with $n_s=10$ and again refine neuron-wise lower- and upper-bounds for fully-connected networks.
We verify between $14\%$ and $34\%$ more regions than the current state-of-the-art, in some cases doubling the number of verified samples, while maintaining a reasonable runtime comparable to that for ReLU networks (see \Tableref{Ta:SCurve}).

\subsection{Autonomous Driving}
\begin{wraptable}[11]{r}{0.54\textwidth}
	\vspace{-5.5mm}
	\caption{Standard (std.), empirically maximal (emp.) and certifiably maximal (cert.) mean absolute steering angle error (MAE) (smaller is better) for \tool vs. \GPUpoly evaluated on every 20$^{th}$ sample and mean evaluation time.}
	\vspace{-3.5mm}
	\centering
	{  \footnotesize
		\begin{adjustbox}{max width=\linewidth}
			\small
			\begin{tabular}{@{}llrrrrr@{}}
				\toprule
				$\epsilon$ & Method & \makecell{std. \\MAE} & \makecell{emp. \\MAE} & \makecell{cert. \\MAE} & \makecell{cert. \\Width} & Time [s] \\
				\midrule
				$1/255$
				& \GPUpoly & 7.37° & 9.41° & 10.35° & 5.75° &  1.55 \\
				& \tool    & 7.37° & 9.41° & \textbf{10.17°} & \textbf{5.30°} & 154.2 \\
				\midrule
				$2/255$
				& \GPUpoly & 7.37° & 11.46° & 18.35° & 19.63° &  2.41 \\				
				& \tool    & 7.37° & 11.46° & \textbf{17.05°} & \textbf{17.03°} & 239.5 \\ %
				\bottomrule
			\end{tabular}
		\end{adjustbox}
	}
	\label{Ta:Relu_Dave}
\end{wraptable}
We evaluate \tool in the setting of autonomous driving, deriving upper and lower bounds to the predicted steering angle under an $\ell_\infty$ threat-model in a regression setting. We thereby demonstrate scalability to large networks ($>100$k neurons and over $27$ million connections) and inputs ($3\times66\times200$) of real-world relevance. We report the certified maximum absolute steering angle error and the width of reachable steering angles. We use PGD \cite{madry2017towards} to compute empirical bounds (emp). We use the CNN architecture proposed by \citet{bojarski2016end} and adversarial training \cite{madry2017towards} on the Udacity autonomous driving dataset \cite{udacity2016selfdriving} to obtain the network evaluated here.
\begin{wrapfigure}[14]{r}{0.50\textwidth}
	\centering
	\vspace{-4.5mm}
	\begin{subfigure}[b]{1.0\linewidth}
		\centering
		\vspace{-0.7cm}
		\includegraphics[width=1.0\textwidth]{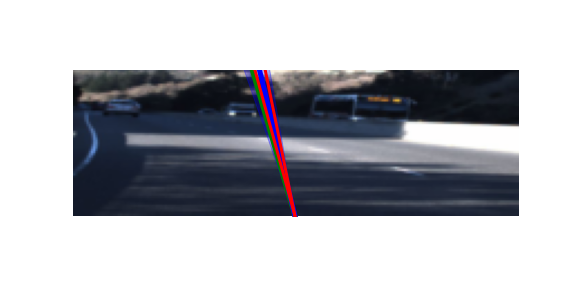}
		\vspace{-1.3cm}
	\end{subfigure}
	\begin{subfigure}[b]{1.0\linewidth}
		\centering
		\vspace{-0.7cm}
		\includegraphics[width=1.0\textwidth]{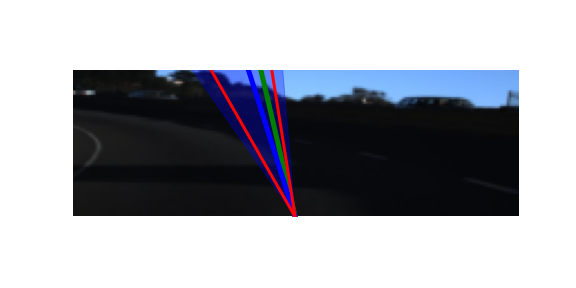}
		\vspace{-1.5cm}
	\end{subfigure}
	\vspace{-4mm}
	\caption{Samples from the self-driving car dataset. The target steering angle is illustrated in green, the predicted one in blue. The empirical bounds for $\eps = 2/255$ are shown in red and the certified range is shaded blue.}
	\label{fig:DAVE}
\end{wrapfigure}

When the permissible perturbation size is small and the standard error of the model is larger than the perturbation effect, cheaper methods such as \GPUpoly already yield good results. However, for larger perturbations, \tool reduces the gap between empirical and certified error around $20$\% (see \Tableref{Ta:Relu_Dave}). In \Figref{fig:DAVE}, we show two representative samples, where the certified steering angle range for $\eps = 2/255$ is shaded blue, the empirical bounds on the steering angle are shown in red, the target in green and the prediction on the unperturbed sample in blue.
Qualitatively, we find that while the network often still performs well on unperturbed samples with poor lighting or contrast (see lower example in \Figref{fig:DAVE}) the sensitivity to perturbations and consequently the width of the reachable steering angle range is much larger than for samples in better conditions (see upper example in \Figref{fig:DAVE}).

\begin{figure*}[t]
	\centering
	\vspace{2mm}
	\begin{subfigure}[t]{0.31\textwidth}
		\centering
		\includegraphics[width=1.05\textwidth]{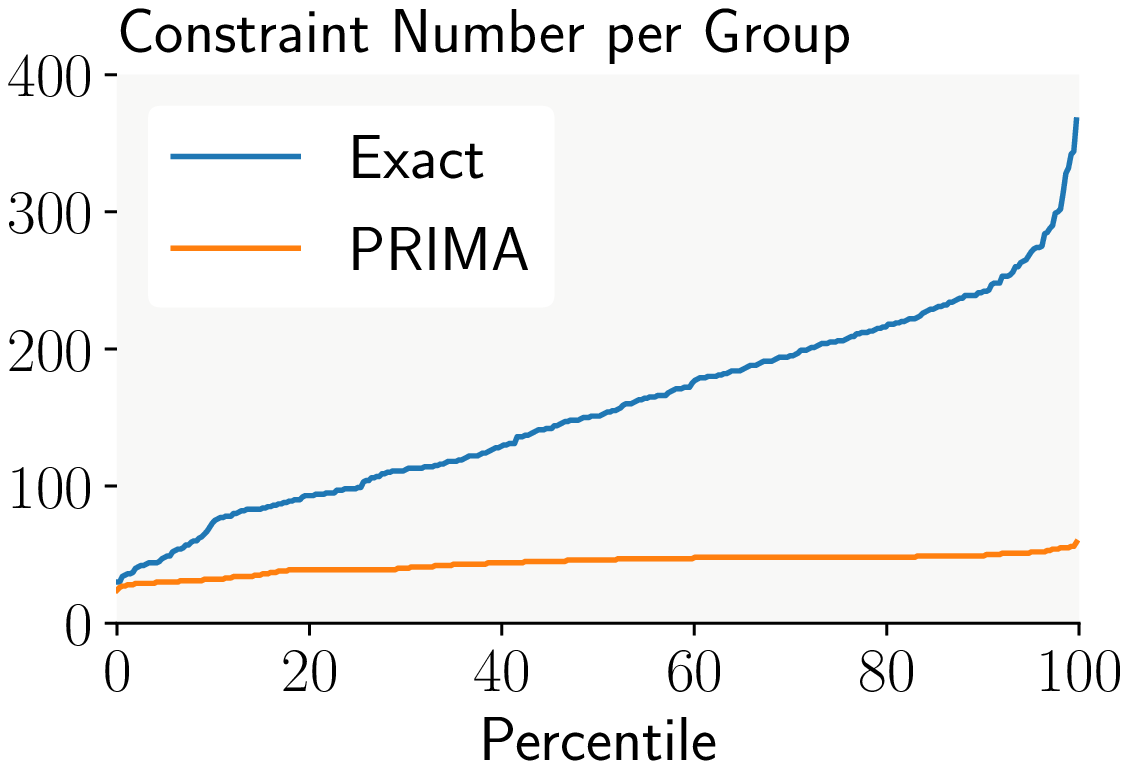}
		\subcaption{Number of constraints for individual $k$-neuron abstractions.}
		\label{fig:nconstraints}
	\end{subfigure}
	\hfill
	\begin{subfigure}[t]{0.31\textwidth}
		\centering
		\includegraphics[width=1.05\textwidth]{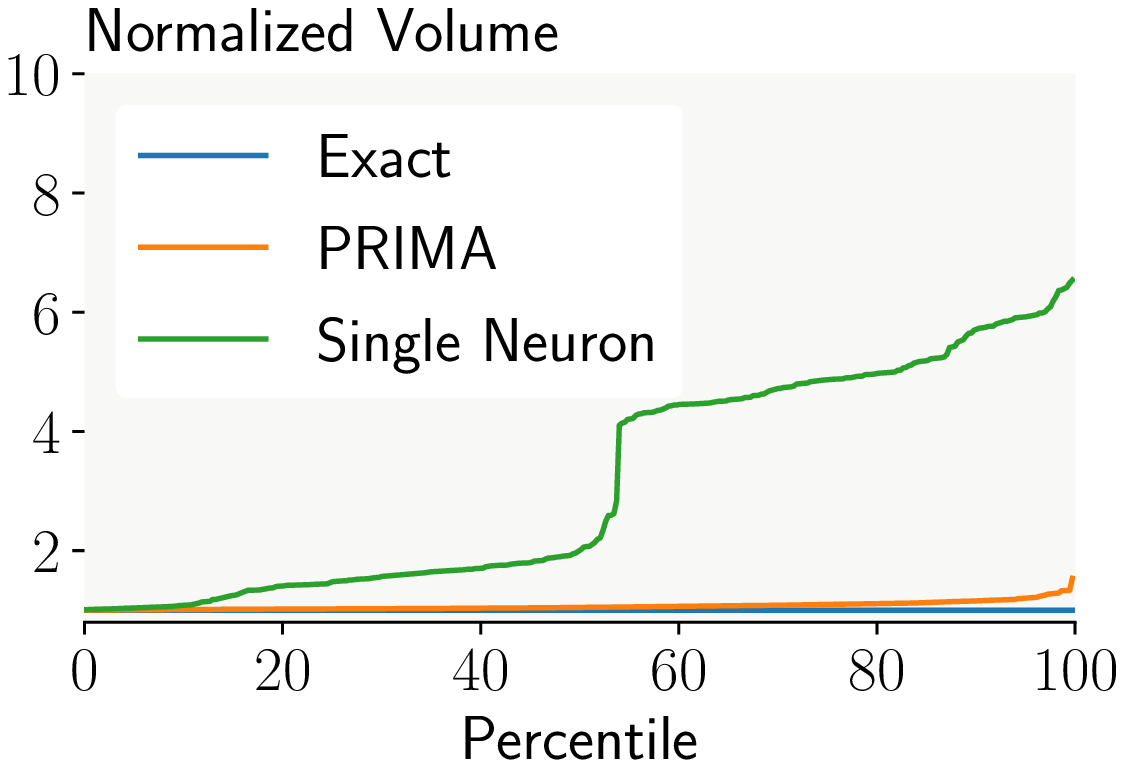}
		\subcaption{Volume of \tool and single-neuron constraint polytopes compared to the exact convex hull.}
		\label{fig:volume}
	\end{subfigure}
	\hfill
	\begin{subfigure}[t]{0.31\textwidth}
		\centering
		\includegraphics[width=1.05\textwidth]{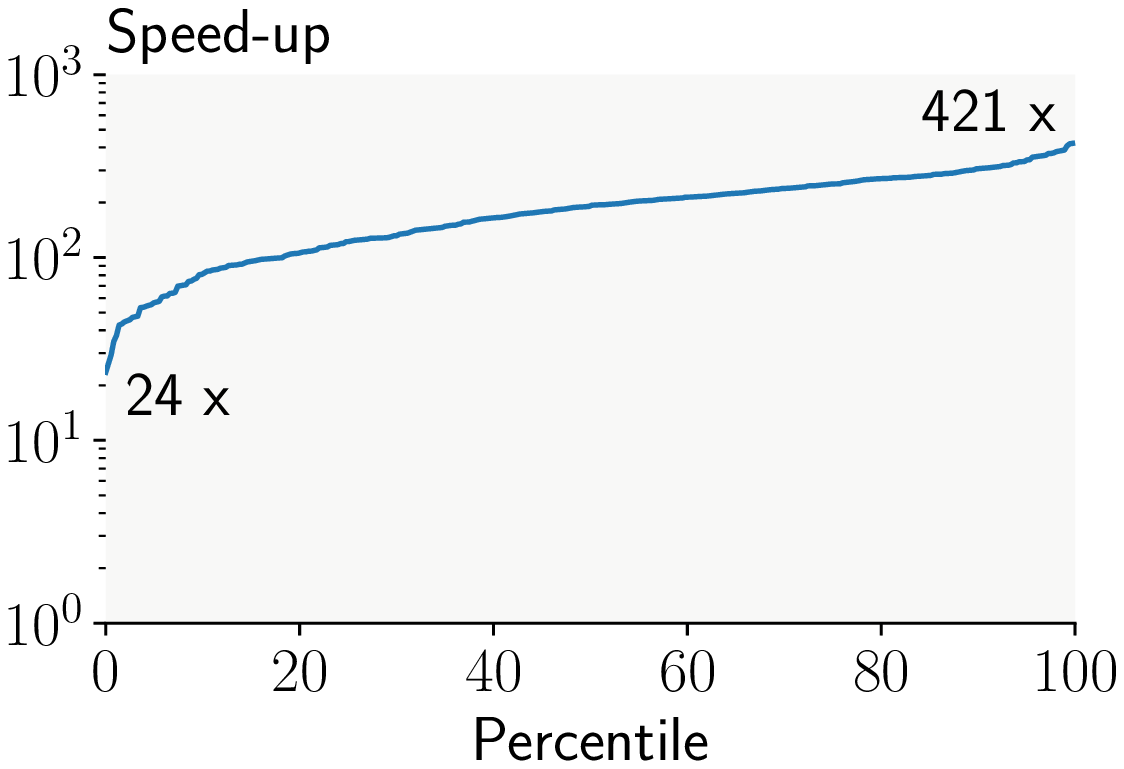}
		\subcaption{Speedup of constraint computation using \SBLM and \PDD compared to an exact convex hull.}
		\label{fig:speedup}
	\end{subfigure}
	\vspace{-3mm}
	\caption{Case study: Analysis of the distribution of the number of discovered constraints, abstraction volume, and runtime over all ($\approx 360$) individual $3$-neuron groups processed during the verification of a single MNIST image on the $5\times 100$ ReLU network.}
	\label{fig:casestudy}
\end{figure*}

\begin{figure*}[t]
	\centering
	\begin{subfigure}[t]{0.32\textwidth}
		\centering
		\includegraphics[width=\textwidth]{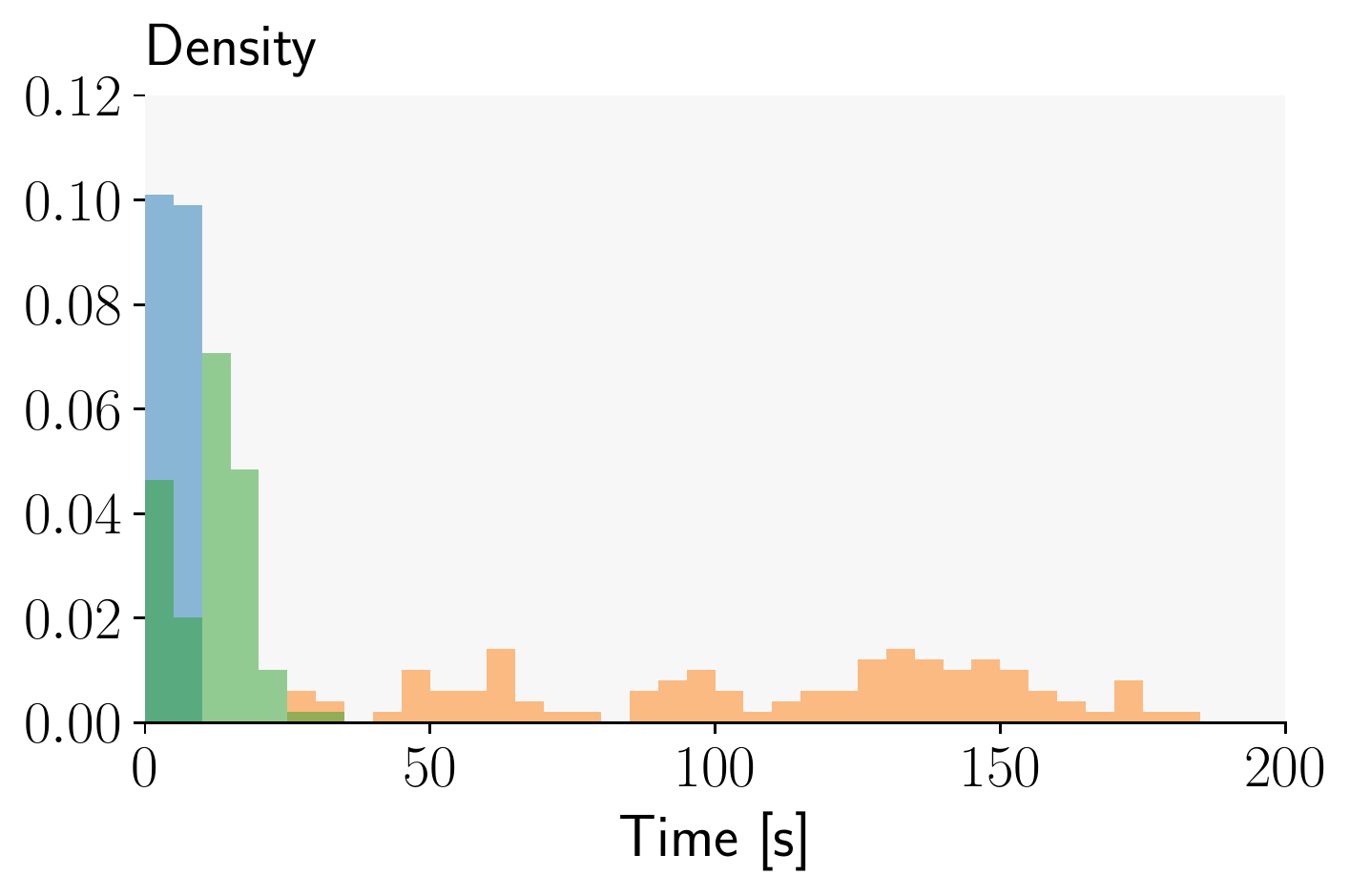}
		\vspace{-6mm}
		\subcaption{\scriptsize Exact -- MNIST $5 \times 100$}
	\end{subfigure}
	\hfill
	\begin{subfigure}[t]{0.32\textwidth}
		\centering
		\includegraphics[width=\textwidth]{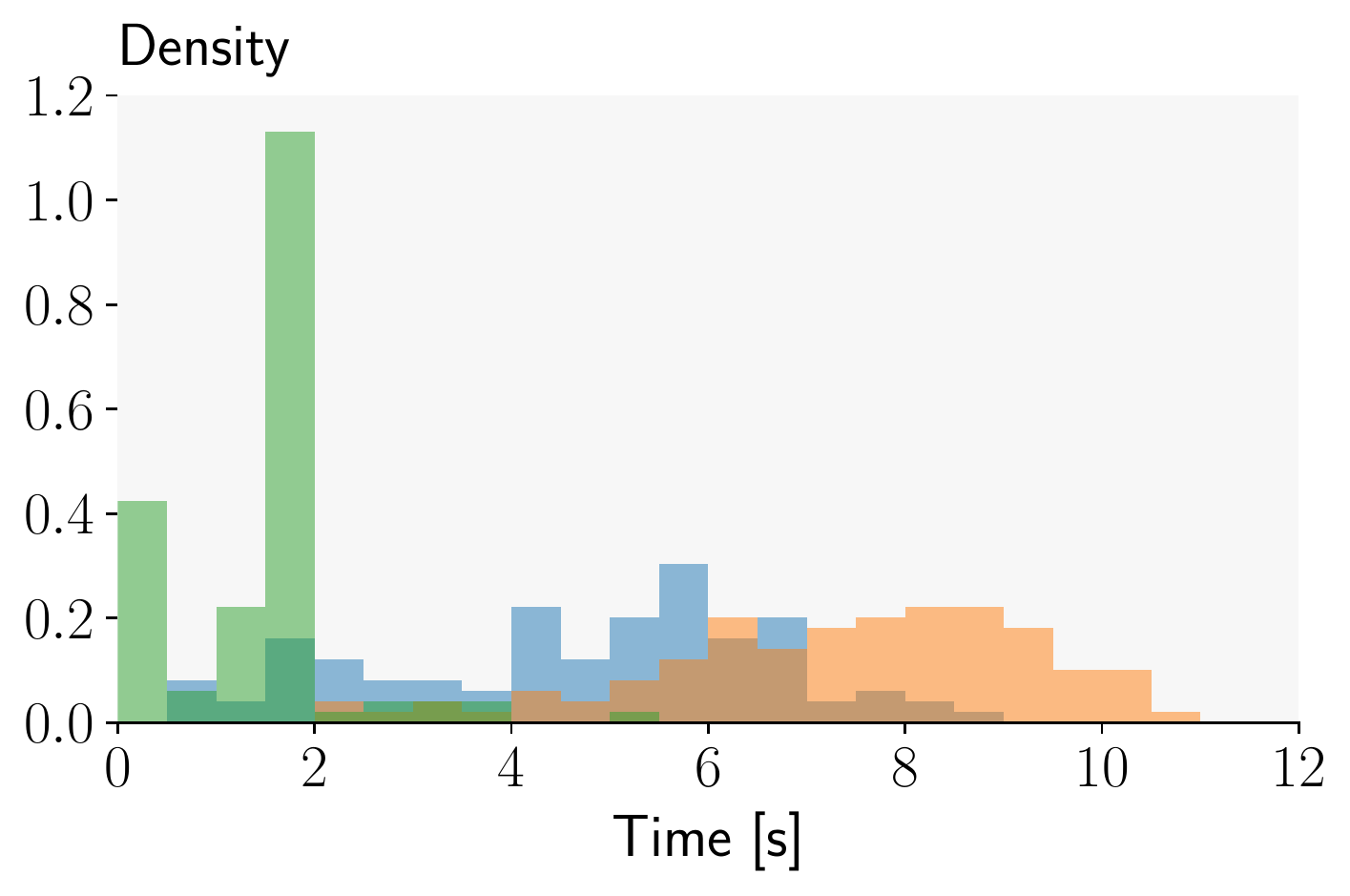}
		\vspace{-6mm}
		\subcaption{\scriptsize SBLM + PDDM -- MNIST $5 \times 100$}
	\end{subfigure}
	\hfill
	\begin{subfigure}[t]{0.32\textwidth}
		\centering
		\includegraphics[width=\textwidth]{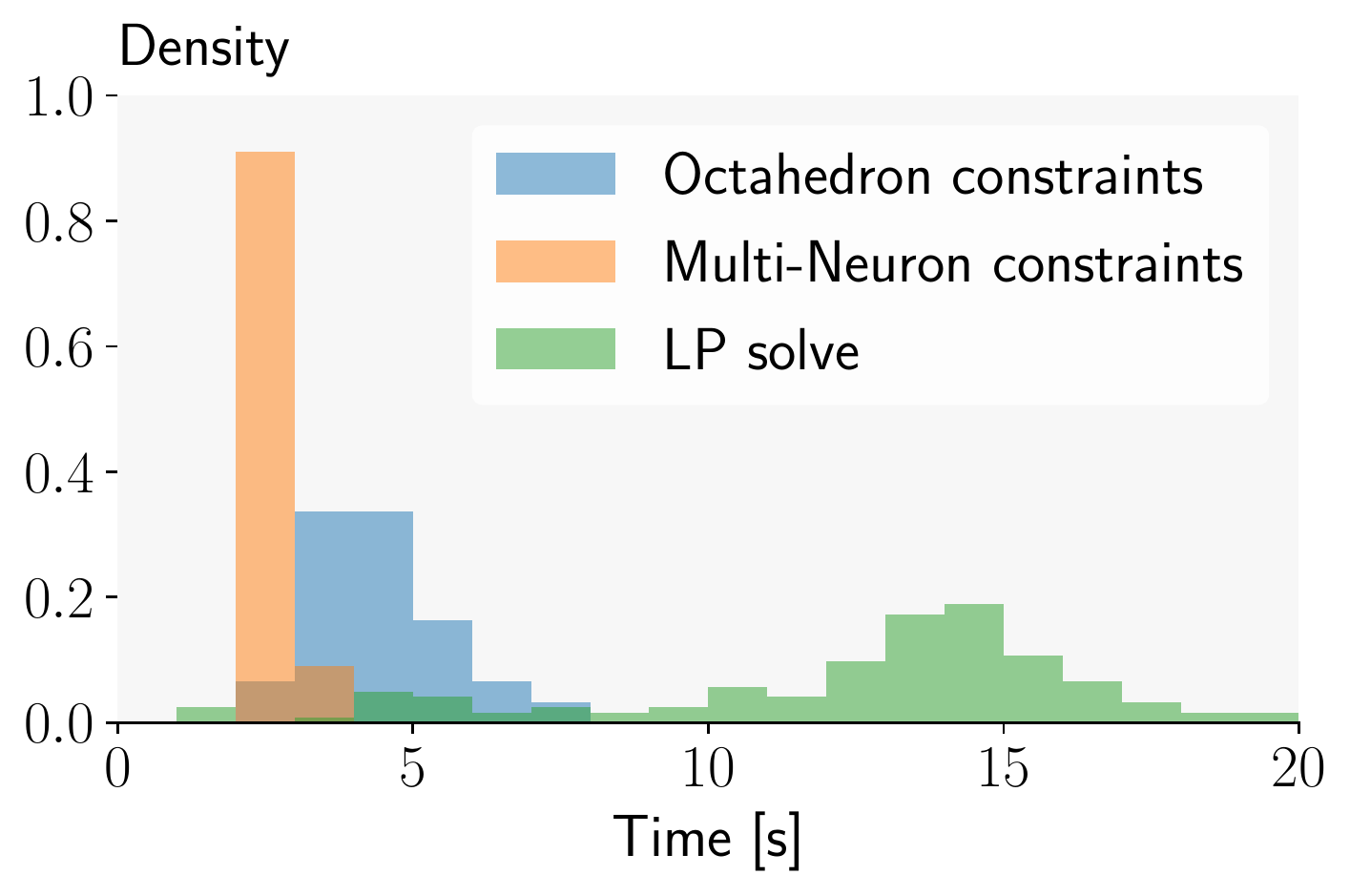}
		\vspace{-6mm}
		\subcaption{\scriptsize SBLM + PDDM -- CIFAR10 \convsmall}
	\end{subfigure}
	\vspace{-2mm}
	\caption{Comparison of the runtime contribution of the octahedral input constraint computation, multi-neuron constraint computation and LP solve.}%
	\label{fig:runtime_comp}
	\vspace{-1mm}
\end{figure*}

\subsection{Effectiveness of \SBLM and \PDDM for Convex Hull Computations} \label{sec:SBLM_PDDM_eval}
Computing approximations with \SBLM using \PDDM has two main advantages compared to the direct convex hull approach: It is significantly faster and produces fewer constraints, making the resulting LP easier to solve, while barely losing any precision.

For example, verifying the $5\times 100$ network with \tool and comparing abstractions for groups of $k=3$ computed with \SBLM and \PDDM or naively and neuron-wise triangle relaxation (Figure \ref{fig:casestudy}), we observe the following:
Using \SBLM and \PDDM we reduce the mean number of constraints computed per neuron-group by over $70\%$ from $156$ to $44$ significantly reducing the number of constraints in the resulting LP, as many hundred such neuron groups are considered.
The mean volume of the constraint polytopes defined by these constraints in the $6$-dimensional input-output space of the individual neuron groups, meanwhile, is only around $5\%$ larger. Single neuron constraints, in contrast, yield $4$-times larger volumes. Additionally, computing the approximate constraints is about $200$ times faster than the exact convex hull.

Not only are \tool constraints faster to generate and allow the verification of the same properties, but a runtime analysis for the first $100$ samples (illustrated in~\Figref{fig:runtime}) shows that they also speed up the final LP solve $8$-fold compared to the naive approach, as significantly fewer constraints have to be considered.
This effect is also observed in the time-intensive neuron-wise bound-refinement where \tool constraints reduce the runtime by $70$\% while allowing $3$ additional regions to be verified. This can be explained by the fewer but more diverse \tool constraints also speeding up the final LP solve in the refinement step reducing the number of timeouts and allowing tighter neuron-wise bounds to be computed.
Using neuron-refinement with \tool is in fact still quicker than the naive approach without any refinement, while almost verifying twice as many samples.
\SBLM combined with exact convex hulls computations already yields a small speed-up of around $20\%$, but the synergy with \PDDM is key to unlock its full potential.

\begin{figure}
	\centering
	\vspace{-2.5mm}
	\includegraphics[width=0.68\linewidth]{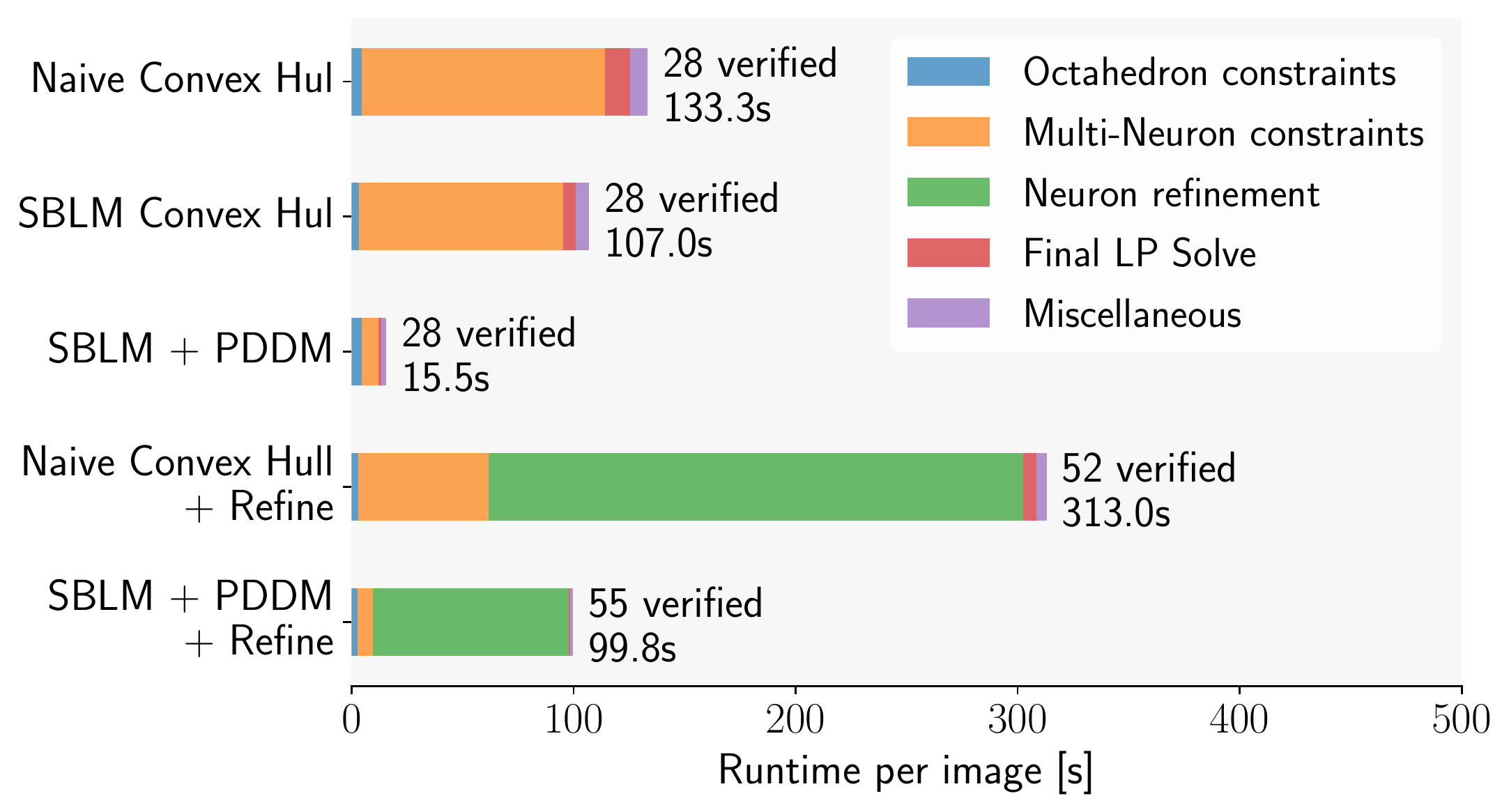}
	\vspace{-2mm}
	\caption{Runtime comparison of using \SBLM vs. exact convex hull for computing relaxations in \tool. Evaluated on 100 images and the MNIST $5 \times 100$ ReLU network.}
	\label{fig:runtime}
	\vspace{-2.5mm}
\end{figure}
An analysis of the runtime contributions of the octahedral input constraint computation, the multi-neuron constraint computation and the final LP solve (illustrated in \Figref{fig:runtime_comp}), shows the following:
Using the naive approach, the multi-neuron constraint computation clearly dominates the runtime, while only contributing around $50$\% when using \SBLM and \PDDM.
For larger networks, the input constraint computation and LP-solve become more expensive, reducing the multi-neuron constraints computation runtime contribution further and further, e.g., $7\%$ for the \cifar \convsmall, and shifting the performance bottleneck to the LP-solver, especially when neuron-wise bound-refinement or partial MILP encodings are used.

%% file: related.tex
\section{Related Work}

The importance of certifying the robustness of neural networks to input perturbations has
created a surge of research activity in recent years. 
The approaches with deterministic guarantees can be divided into exact and incomplete methods. Incomplete methods are much faster and more scalable than exact ones, but they can be imprecise, i.e., they may fail to certify a property even if it holds.

Complete methods are mostly based on satisfiability modulo theory (SMT) \cite{katz2019marabou,ehlers2017formal,katz2017reluplex,huang2017safety} or the branch-and-bound approach \cite{anderson2020strong,botoeva2020efficient,bunel2020branch,tjeng2017evaluating,Lu2020Neural,xu2020fast,wang2021beta,depalma2021scaling}, often implemented using mixed integer linear programming (MILP). These methods offer exactness guarantees but are based on solving NP-hard optimization problems, which can make them intractable even for small networks.
Incomplete methods can be divided into bound propagation approaches \cite{gowal2019scalable,mirman2018diffai,singh2018fast, weng2018fastlin, 	singh2019abstract,zhang2018crown,mller2021neural} and those that generate polynomially-solvable optimization problems \cite{bunel2020efficient, lyu2019fastened, singh2019beyond, raghunathan2018semidefinite,	xiang2018output,tjandraatmadja2020convex,dathathri2020enabling} such as linear programming (LP) or semidefinite programming (SDP) optimization problems. 
Compared to deterministic certification methods, randomized smoothing \cite{lecuyer2018dp,cohen2019smoothing,salman2019provably} is a defence method providing only probabilistic guarantees and incurring significant runtime costs at inference time, with the generalization to arbitrary safety properties still being an open problem.

A new avenue towards more precision are methods \cite{singh2019beyond, tjandraatmadja2020convex, depalma2021scaling} breaking the so-called convex barrier \cite{salman2019convex} by considering activation functions jointly.
However, their scalability is limited by the need to solve NP-hard convex
hull problems. There are many approaches for solving the convex hull problem
for polyhedra exactly \cite{joswig2003beneath, edelsbrunner2012algorithms, fukuda1995double, motzkin1953double, barber1993quickhull, dantzig1998linear, avis1991basis, avis1992pivoting},
in contrast to few approximate methods which either sacrifice soundness
\cite{bentley1982approximation,khosravani2013simple,zhong2014finding,sartipizadeh2016computing}
or still exhibit exponential complexity \cite{xu1998approximate}, prohibiting
their use in neural network verification.

Our work follows the line of convex barrier-breaking methods, generalizing the concept to arbitrary bounded, multivariate activations. In contrast to prior work, we decompose the underlying convex hull problem into lower-dimensional spaces and solve it approximately using a novel relaxed Double Description, irredundancy formulation, and a new ray-shooting-based algorithm to add multiple constraints jointly. The resulting speed-ups make \tool tractable for non-piecewise-linear activations, a first for convex barrier-breaking methods.

%% file: conclusion.tex
\section{Conclusion}

We presented \tool, a general framework that substantially advances the state-of-the-art in neural network verification by providing efficient multi-neuron abstractions for arbitrary, bounded, multivariate non-linear activation functions. Our key idea is to compute tighter overall abstractions by considering many overlapping neuron groups thereby capturing more inter-neuron dependencies. To enable this, we decompose the bottleneck convex hull computation into lower-dimensional spaces and solve it approximately. Our extensive experimental evaluation shows that our algorithmic advances shift the bottleneck to the LP-solver while significantly improving both precision and scalability over prior work.

%% file: appendix.tex
\section{Architectures}\label{app:architectures}
In this section, we provide an overview over all the architectures evaluated in \Secref{sec:experiments}.
The fully connected networks with the naming format \texttt{AxB} have \texttt{A} hidden layers with \texttt{B} neurons each and additionally an input and output layer.
The architecture of the convolutional networks for MNIST are detailed in \Tableref{tab:architectures}.

\begin{table}[htb]
	\centering
	\small
	\caption{Network architectures of the convolutional networks for \cifar, \mnist, and the steering angle prediction task. All layers listed below are followed by an activation layer. The output layer is omitted. The output of the \NVIDIA network is passed through a tnah function. `\textsc{Conv} c h$\times$w/s/p' corresponds to a 2D convolution with c output channels, an h$\times$w kernel size, a stride of s in both dimensions, and an all around zero padding of p.}
	\scalebox{0.85}{
	\begin{tabular}{ccccc}
		\toprule
		\convsmall & \convbig & \CNNA & \CNNB & \NVIDIA \\
		\midrule
		\textsc{Conv} 16 4$\times$4/2/0 & \textsc{Conv} 32 3$\times$3/1/1 & \textsc{Conv} 16 4$\times$4/2/1  & \textsc{Conv} 32 5$\times$5/2/0  & \textsc{Conv} 24 5$\times$5/2/0 \\
		\textsc{Conv} 32 4$\times$4/2/0 & \textsc{Conv} 32 4$\times$4/2/1 & \textsc{Conv} 32 4$\times$4/2/1  & \textsc{Conv} 128 4$\times$4/2/1 & \textsc{Conv} 36 5$\times$5/2/0 \\
		\textsc{FC} 100                 & \textsc{Conv} 64 3$\times$3/1/1 & \textsc{FC} 100                  & \textsc{FC} 250                  & \textsc{Conv} 48 5$\times$5/2/0 \\
		                                & \textsc{Conv} 64 4$\times$4/2/1 &                                  &                                  & \textsc{Conv} 64 3$\times$3/1/0 \\ 
		                                & \textsc{FC} 512                 &                                  &                                  & \textsc{Conv} 64 3$\times$3/1/0 \\
		                                & \textsc{FC} 512                 &                                  &                                  & \textsc{FC} 100 \\
  		                                &                                 &                                  &                                  & \textsc{FC} 50 \\
  		                                &                                 &                                  &                                  & \textsc{FC} 10 \\
		\bottomrule
	\end{tabular}
	}
	\label{tab:architectures}
\end{table}

\begin{table}[htb]
	\centering
	\small
	\caption{Network architecture of the \resnet. All layers listed below are followed by a ReLU activation layer, except if they are followed by a \textsc{ResAdd} layer. The output layer is omitted. `\textsc{Conv} c h$\times$w/s/p' corresponds to a 2D convolution with c output channels, an h$\times$w kernel size, a stride of s in both dimensions, and an all around zero padding of p.}
	\scalebox{0.85}{
		\begin{tabular}{cc}
			\toprule
			\multicolumn{2}{c}{\resnet}\\
			\midrule
			\multicolumn{2}{c}{\textsc{Conv} 16 3$\times$3/1/1} \\ 
			\multirow{1}{*}{\textsc{Conv} 16 1$\times$1/1/0} & \textsc{Conv} 16 3$\times$3/1/1   \\
			& \textsc{Conv} 16 3$\times$3/1/1 \\
			\multicolumn{2}{c}{\textsc{ResAdd}} \\ 
			\multirow{1}{*}{\textsc{Conv} 16 1$\times$1/1/0} & \textsc{Conv} 16 3$\times$3/1/1   \\
			& \textsc{Conv} 16 3$\times$3/1/1 \\
			\multicolumn{2}{c}{\textsc{ResAdd}} \\ 
			\textsc{Conv} 32 2$\times$2/2/0 & \textsc{Conv} 32 3$\times$4/2/1 \\
			&\textsc{Conv} 32 3$\times$3/1/1 \\
			\multicolumn{2}{c}{\textsc{ResAdd}} \\ 
 			\textsc{Conv} 64 2$\times$2/2/0 & \textsc{Conv} 64 3$\times$4/2/1 \\
			&\textsc{Conv} 64 3$\times$3/1/1 \\
			\multicolumn{2}{c}{\textsc{ResAdd}} \\ 
			\multicolumn{2}{c}{\textsc{FC} 1000} \\
			\bottomrule
		\end{tabular}
	}
	\label{tab:resnet_architecture}
\end{table}